\documentclass{article} % For LaTeX2e
\usepackage{iclr2021_conference,times}

\usepackage{amsmath,amsfonts,bm}

\def\eqref#1{equation~\ref{#1}}
\def\1{\bm{1}}

\def\vs{{\bm{s}}}

\DeclareMathAlphabet{\mathsfit}{\encodingdefault}{\sfdefault}{m}{sl}
\SetMathAlphabet{\mathsfit}{bold}{\encodingdefault}{\sfdefault}{bx}{n}

\newcommand{\R}{\mathbb{R}}

\newcommand{\KL}{D_{\mathrm{KL}}}

\usepackage{hyperref}
\usepackage{url}

\usepackage[inline]{enumitem}
\usepackage[utf8]{inputenc} % allow utf-8 input
\usepackage{hyperref}       % hyperlinks
\usepackage{url}            % simple URL typesetting
\usepackage{booktabs}       % professional-quality tables
\usepackage{amsfonts}       % blackboard math symbols
\usepackage{amsmath}
\usepackage{graphicx,wrapfig,lipsum}
\usepackage{cleveref}
\usepackage{booktabs}
\usepackage{amsthm}
\usepackage{caption}
\usepackage{subcaption}
\usepackage{mathrsfs}  
\usepackage{amssymb}
\usepackage{wrapfig}
\usepackage{tikz}
\usepackage{tikz-layers}
\usepackage{todonotes}
\usepackage{bm}
\usepackage[inline]{enumitem}
\usepackage{mathtools}

\usepackage{rotating}
\usepackage{xcolor}
\usepackage{algorithm}
\usepackage{algorithmic}
\usepackage{siunitx}
\usepackage{wrapfig}
\usepackage{multirow}
\usepackage{bibunits}

\tikzset{%outer sep=5pt, inner sep=1.0pt,
    rect_sharp/.style={rectangle, draw=black, minimum width=30, minimum height=20, outer sep=2.5pt, inner sep=2.0pt},
    rect/.style={rectangle, rounded corners=2.5pt, draw=black, minimum width=30, minimum height=20, outer sep=2.5pt, inner sep=2.0pt},
	rect2/.style={rectangle, rounded corners=2.5pt, draw=black, minimum width=40, minimum height=20, outer sep=2.5pt, inner sep=2.0pt},
	every picture/.style={line width=0.75pt},
    ncbar angle/.initial=90,
    ncbar/.style={
        to path=(\tikztostart)
        -- ($(\tikztostart)!#1!\pgfkeysvalueof{/tikz/ncbar angle}:(\tikztotarget)$)
        -- ($(\tikztotarget)!($(\tikztostart)!#1!\pgfkeysvalueof{/tikz/ncbar angle}:(\tikztotarget)$)!\pgfkeysvalueof{/tikz/ncbar angle}:(\tikztostart)$)
        -- (\tikztotarget)
    },
    ncbar/.default=0.5cm,
}

\tikzset{square left brace/.style={ncbar=0.15cm}}
\tikzset{square right brace/.style={ncbar=-0.2cm}}

\tikzset{
    seq/.style={rectangle, text centered, minimum height=0.5cm, minimum width=3.5cm, rotate=90, outer sep=2.5pt, inner sep=0pt},
    wideseq/.style={rectangle, text centered, minimum height=0.75cm, minimum width=3.5cm, rotate=90, outer sep=2.5pt, inner sep=0pt},
    seq2/.style={rectangle, text centered, minimum height=0.5cm, minimum width=3.5cm, outer sep=1.0pt, inner sep=0pt},
    wideseq2/.style={rectangle, text centered, minimum height=0.75cm, minimum width=3.5cm, outer sep=1.0pt, inner sep=0pt}
}

\tikzset{
    sqr/.style={rectangle, text centered, minimum height=1.0cm, minimum width=1.0cm, outer sep=0.25cm, inner sep=0pt},
}

\tikzset{
    circ/.style={circle, draw=black, minimum size=25, outer sep=2.5pt, inner sep=0.0pt},
    circ2/.style={circle, minimum width=0.4cm, minimum height=0.4cm, outer sep=2.5pt, inner sep=0pt}
}

\tikzset{
    cross/.style={cross out, draw=black, minimum size=12.5pt, inner sep=0pt, outer sep=0pt, rotate=45},
    cross/.default={1pt}
}

\definecolor{grey}{HTML}{F5F5F5}
\definecolor{darkgrey}{HTML}{888888}
\definecolor{orange}{HTML}{FFF2CC}
\definecolor{brown}{HTML}{D6B656}
\definecolor{lightyellow}{HTML}{fff0d2}
\definecolor{yellow}{HTML}{d0ab49}
\definecolor{lightblue}{HTML}{e6f1f8}
\definecolor{blue}{HTML}{69a7d0}
\definecolor{lightgreen}{HTML}{e9f5e9}
\definecolor{green}{HTML}{7cb959}
\definecolor{lightpurple}{HTML}{f4ebf9}
\definecolor{purple}{HTML}{b793c8}
\definecolor{darkerpurple}{HTML}{9673A6}

\usetikzlibrary{calc,positioning,decorations.pathreplacing,shapes.misc,arrows.meta,3d}

\pgfrealjobname{camera_ready}

\newcommand{\cN}{\mathcal{N}}

\newcommand{\bW}{\mathbf{W}}
\newcommand{\bx}{\mathbf{x}}
\newcommand{\bX}{\mathbf{X}}
\newcommand{\bz}{\mathbf{z}}
\newcommand{\bv}{\mathbf{v}}
\newcommand{\m}{\mathbf{m}}
\newcommand{\bh}{\mathbf{h}}
\newcommand{\bZ}{\mathbf{Z}}
\newcommand{\by}{\mathbf{y}}
\newcommand{\NN}{\mathbb{N}}
\newcommand{\EE}{\mathbb{E}}

\newcommand{\cX}{\mathcal{X}}
\newcommand{\cL}{\mathcal{L}}
\newcommand{\cH}{\mathcal{H}}
\newcommand{\cY}{\mathcal{Y}}

\renewcommand\given[1][]{\:#1\vert\:}

\newcommand{\T}[1]{\mathrm{T}({#1})}
\newcommand{\Tra}[2]{\mathrm{T}^{{#1}}({#2})}
 
\newcommand{\bb}{\mathbf{b}}
\newcommand{\bt}{\mathbf{t}}
\newcommand{\bs}{\mathbf{s}}
\newcommand{\Seq}[1]{\operatorname{Seq}(#1)}
\newcommand{\bphi}{\bm\Phi}

\newcommand{\FCN}[1]{\ensuremath{\text{FCN}_{#1}}}

\newcommand{\LStwoTwidth}[2]{\ensuremath{\text{LS2T}_{#1}^{#2}}}

\newcommand{\FCNLStwoTwidth}[3]{\ensuremath{\text{FCN}_{#1}\text{-}\text{LS2T}_{#2}^{#3}}}

\DeclareMathOperator{\spn}{span}

\theoremstyle{plain}
\newtheorem{thm}{Theorem}[section]
\newtheorem{proposition}[thm]{Proposition}
\newtheorem{lemma}[thm]{Lemma}

\theoremstyle{definition}
\newtheorem{definition}[thm]{Definition}
\newtheorem{remark}[thm]{Remark}
\newtheorem{example}[thm]{Example}

\title{Seq2Tens: An Efficient Representation of Sequences by Low-Rank Tensor Projections}

\author{Csaba Toth$ ^\star $ \And
  Patric Bonnier$ ^\star $ \And 
  Harald Oberhauser$ ^\star $ \And \\[-12pt]
  $ ^\star $Mathematical Institute, University of Oxford \\
  \texttt{\{toth, bonnier, oberhauser\}@\hspace{0.1pt}maths.ox.ac.uk}
}

\iclrfinalcopy % Uncomment for camera-ready version, but NOT for submission.
\begin{document}

\maketitle

\begin{abstract}
  Sequential data such as time series, video, or text can be challenging to analyse as the ordered structure gives rise to complex dependencies. 
  At the heart of this is non-commutativity, in the sense that reordering the elements of a sequence can completely change its meaning.  
  We use a classical mathematical object -- the free algebra -- to capture this non-commutativity. 
  To address the innate computational complexity of this algebra, we use compositions of low-rank tensor projections.
  This yields modular and scalable building blocks that give state-of-the-art performance on standard benchmarks such as multivariate time series classification, mortality prediction and generative models for video. Code and benchmarks are publically available at \url{https://github.com/tgcsaba/seq2tens}.
\end{abstract}
\section{Introduction}
A central task of learning is to find representations of the underlying data that efficiently and faithfully capture their structure.
In the case of sequential data, one data point consists of a sequence of objects.
This is a rich and non-homogeneous class of data and includes classical uni- or multi-variate time series (sequences of scalars or vectors), video (sequences of images), and text (sequences of letters). 
%this includes classical uni- or multi-variate time series ($\cX=\R$ or $\cX=\R^d$); video ($\cX=\{\text{pixel}({i,j=1,\ldots,d}):\text{pixel} \in \{\text{red,blue,green} \}\}$); or text ($\cX=\{A,\ldots,Z,a,\ldots,z\}$).
Particular challenges of sequential data are that each sequence entry can itself be a highly structured object and that data sets typically include sequences of different length which makes naive vectorization troublesome.
% Even with a good representation for every entry, e.g.~a feature map for an image, we would still need to construct a method that turns this representation of static objects into a faithful representation of the sequence.
% Naive approaches such as applying the feature map entry-wise and flattening the resulting sequence of features into a long vector run into trouble even for the arguably simple case of uni-variate time series: the underlying data often consists of sequences of different length and this vectorization results in vectors in different dimensions.
% Classical approaches such as dynamic time warping address this to a certain extent and work well in the uni-variate case but typically do not work well for multi-variate time series; further, like kernels for sequence this only provide a distance between sequences rather than features of a sequence which makes it troublesome to combine with deep learning pipelines.  

\paragraph{Contribution.}
Our main result is a generic method that takes a \emph{static feature map} for a class of objects (e.g.~a feature map for vectors, images, or letters) as input and turns this into a feature map for sequences of arbitrary length of such objects (e.g.~a feature map for time series, video, or text). 
We call this feature map for sequences Seq2Tens for reasons that will become clear; among its attractive properties are that it 
\begin{enumerate*}[label=(\roman*)]
\item provides a structured, parsimonious description of sequences; generalizing classical methods for strings, 
\item comes with theoretical guarantees such as universality,
\item can be turned into modular and flexible neural network (NN) layers for sequence data. 
\end{enumerate*}
The key ingredient to our approach is to embed the feature space of the static feature map into a larger linear space that forms an algebra (a vector space equipped with a multiplication).
The product in this algebra is then used to ``stitch together'' the static features of the individual sequence entries in a structured way.
The construction that allows to do all this is classical in mathematics, and known as the \emph{free algebra} (over the static feature space). 
\paragraph{Outline.}
Section \ref{sec:2} formalizes the main ideas of Seq2Tens and introduces the free algebra $\T{V}$ over a space $V$ as well as the associated product, the so-called \emph{convolution tensor product}. 
Section \ref{sec:3} shows how low rank (LR) constructions combined with sequence-to-sequence transforms allows one to efficiently use this rich algebraic structure. 
Section \ref{sec:4} applies the results of Sections~\ref{sec:2} and \ref{sec:3} to build modular and scalable NN layers for sequential data.
Section \ref{sec:5} demonstrates the flexibility and modularity of this approach on both discriminative and generative benchmarks.
Section \ref{sec: summary} makes connections with previous work and summarizes this article.   
In the appendices we provide mathematical background, extensions, and detailed proofs for our theoretical results.
% Appendix \ref{app:tensor} contains some background on tensor algebras. Appendix \ref{app:algebra} studies the algebraic properties of $ \Phi $ and shows that one may use it to construct universal representations of sequences as tensors. Appendix \ref{app:iterated} shows that one may compute higher order features of $ \Phi $ by iterated compositions. Appendix \ref{app:kmers} shows that the tensor representation we propose indeed generalizes k-mers to vector valued data. Finally, Appendix \ref{app:details} contains implementation details for our experiments.
% By putting these ideas together we get efficiently computable universal non/commutative polynomial maps that can efficiently encode sequential data. We use this map to increase the accuracy of time series classification, and to efficiently disentangle generated sequences of data. 

%\begin{enumerate}
%	\item Representing a sequence of events by non-commutative polynomials.
%	\item The universality of tensors.
%	\item The efficiency of iterated compositions of low-rank objects.
%\end{enumerate}

\section{Capturing order by non-commutative multiplication} \label{sec:2}
We denote the set of sequences of elements in a set $\cX$ by 
\begin{align}
  \Seq{\cX}=\{\bx=(\bx_i)_{i=1,\ldots,L}: \bx_i \in \cX,\, L \ge 1\} %= \bigcup_{T\ge 1} (\R^d)^T. 
\end{align}
where $ L\geq 1 $ is some arbitrary length.
Even if $\cX$ itself is a linear space, e.g.~$\cX=\R$, $\Seq{\cX}$ is never a linear space since there is no natural addition of two sequences of different length.
\paragraph{Seq2Tens in a nutshell.}
% This is typically overcome by embedding a subset of $ \Seq{V} $ into a suitably large vector space,
Given any vector space $ V $ we may construct the so-called free algebra $ \T{V} $ over $ V $.
We describe the space $\T{V}$ in detail below, but as for now the only thing that is important is that $\T{V}$ is also a vector space that includes $V$, and that it carries a non-commutative product, which is, in a precise sense, ``the most general product'' on $ V $.

The main idea of Seq2Tens is that any ``static feature map'' for elements in $\cX$
\[
  \phi: \cX \to V
\]
% Below we construct a feature map $\Phi: \Seq{V}\rightarrow \T{V}$ that represents a sequence as an element of the linear space $\T{V}$.
% A function $f(\bx)$ on sequences can then be described as a linear functional of $\Phi(\bx)$, that is $f(\bx)\approx \langle \ell, \Phi(\bx) \rangle$ for some $\ell \in \T{\vs}$. 
% For a background on tensors we refer to Appendix \ref{app:tensor}.
% Given a sequence $\bx=(\bx_1,\ldots,\bx_L) \in \Seq{\R^d}$
can be used to construct a new feature map $\Phi:\Seq{\cX} \rightarrow \T{V}$ for sequences in $\cX$ by using the algebraic structure of $\T{V}$: 
the non-commutative product on $ \T{V} $ makes it possible to ``stitch together'' the individual features $\phi(\bx_1),\ldots,\phi(\bx_L) \in V \subset \T{V}$ of the sequence $\bx$ in the larger space $\T{V}$ by multiplication in $\T{V}$.
With this we may define the feature map $\Phi(\bx)$ for a sequences $\bx=(\bx_1,\ldots,\bx_L) \in \Seq{\cX}$ as follows 
\begin{enumerate}[label=(\roman*)]
\item\label{itm:lift map}
  lift the map $\phi:\cX \to V$ to a map $\varphi: \cX \to \T{V}$,
	\item \label{itm:lift}
	map $\Seq{\cX } \to \Seq{\T{V}}$ by $(\bx_1,\ldots,\bx_L) \mapsto (\varphi(\bx_1),\ldots,\varphi(\bx_L))$, 
	\item \label{itm:mult}
    map $\Seq{\T{V}} \to \T{V}$ by multiplication $(\varphi(\bx_1),\ldots,\varphi(\bx_L)) \mapsto \varphi(\bx_1)\cdots \varphi(\bx_L)$. 
%	use the non-commutative product \eqref{eq:ncp} to form an element of $\T{\R^d}$, \[\Phi(\bx)=\varphi(\bx_1)\cdots \varphi(\bx_L) \in \T{\R^d}.\] 
\end{enumerate}
In a more concise form, we define $\Phi$ as
\begin{align}
 \Phi: \Seq{\cX} \to \T{V},\quad \Phi(\bx)=\prod_{i=1}^L \varphi(\bx_i) \label{eq:mult_varphi}
\end{align}
where $\prod$ denotes multiplication in $\T{V}$.
We refer to the resulting map $\Phi$ as the Seq2Tens map, which stands short for \textit{\textbf{Seq}uences-\textbf{2}-\textbf{Tens}ors}.
Why is this construction a good idea?
First note, that step \ref{itm:lift map} is always possible since $V\subset \T{V}$ and we discuss the simplest such lift before Theorem~\ref{thm:univ} as well as other choices in Appendix~\ref{app:algebra}. 
Further, if $\phi$, respectively~$\varphi$, provides a faithful representation of objects in $\cX$, then there is no loss of information in step~\ref{itm:lift}.
Finally, since step \ref{itm:mult} uses ``the most general product'' to multiply $\varphi(\bx_1)\cdots \varphi(\bx_L)$ one expects that $\Phi(\bx) \in \T{V}$ faithfully represents the sequence $\bx$ as an element of $\T{V}$.
Indeed in Theorem~\ref{thm:univ} below we show an even stronger statement, namely that if the static feature map $\phi:\cX \to V$ contains enough non-linearities so that non-linear functions from $\cX $ to $\R$ can be approximated as \emph{linear functions} of the static feature map $\phi$, then the above construction extends this property to functions of sequences.
Put differently, \emph{if $\phi$ is a universal feature map for $\cX$, then $\Phi$ is a universal feature map for $\Seq{\cX}$};
that is, any non-linear function $f(\bx)$ of a sequence $\bx$ can be approximated as a linear functional of $\Phi(\bx)$,  $f(\bx) \approx \langle \ell, \Phi(\bx) \rangle$.  
We also emphasize that the domain of $\Phi$ is the space $\Seq{\cX}$ of sequences of \emph{arbitrary} (finite) length.
The remainder of this Section gives more details about steps~\ref{itm:lift map},\ref{itm:lift},\ref{itm:mult} for the construction of $\Phi$. 
% The remainder of Section~\ref{app:algebra} make this statement mathematically rigorous, in the the sense that Theorem~\ref{prop:univ} shows that $\Phi$ is a universal feature map for sequences in $\R^d$ whenever $\varphi$ is a universal feature map for vectors in $\R^d$.
\paragraph{The free algebra $\T{V}$ over a vector space $V$.}
Let $V$ be a vector space.
We denote by $\T{V}$ the set of sequences of tensors indexed by their degree $m$, 
\begin{align}
  \T{V} :=% \prod_{m\geq 0} (\R^d)^{\otimes m}=
  \{\bt=(\bt_m)_{m\ge 0}\,\vert\, \bt_m \in V^{\otimes m}\}%= (1, \R^d, \R^d\otimes \R^d, \R^d\otimes \R^d\otimes \R^d, \ldots ).
\end{align}
where by convention $V^{\otimes 0}=\R$.
For example, if $V=\R^d$ and $\bt=(\bt_m)_{m\ge 0}$ is some element of $ \T{\R^d} $, then its degree $m=1$ component is a $ d $-dimensional vector $\bt_1$, its degree $m=2$ component is a $d\times d$ matrix $\bt_2$, and its degree $m=3$ component is a degree $ 3 $ tensor $\bt_3$.
By defining addition and scalar multiplication as
\begin{align}
\bs+\bt:=(\bs_m+\bt_m)_{m\geq 0}, \quad c\cdot \bt=(c \bt_m)_{m \ge 0}
\end{align}
the set $\T{V}$ becomes a linear space.
By identifying $v \in V$ as the element $(0,v,0,\ldots,0) \in \T{V}$ we see that $V$ is a linear subspace of $\T{V}$. 
Moreover, while $V$ is only a linear space, $\T{V}$ carries a product that turns $\T{V}$ into an algebra.
%Since we assume the sequence $\ell=(\ell_m)_{m=0}^M$ is finite, the above sum is finite so convergence is not an issue. \todo{remove?}
%By Riesz representation we identify $L(\R^d,\R)$ with $\R^d$ and thus $\ell \in \T{\R^d} $ by setting $\ell_m=0^{\otimes m}$ for $m \ge M$.
% \begin{example} \label{ex:represent}\todo{maybe move to appendix?}
% The non-commutative polynomial $ P(x_1,x_2,x_3) = 3 + x_1 + 2x_1x_3 + 5x_2^2$ corresponds to the linear functional on $ \T{\R^3}$, 
% \begin{equation}
% \bx \mapsto \langle \ell, (\bx^{\otimes m})_{m \ge 0}   \rangle \text{ where }
% \ell: = (\ell_0,\ell_1,\ell_2):= \Big(3,
% \begin{pmatrix}
% 1\\
% 0 \\
% 0
% \end{pmatrix},
% \begin{pmatrix}
% 0&0&2\\
% 0&5&0\\
% 0&0&0
% \end{pmatrix}
% \Big). 
% \end{equation}
% In particular, as idea 2 shows $k$-mers can be encoded on degree $k$-tensors.
% \end{example}
%\paragraph{Feature maps for sequences.}
%We now produce a feature map $\Phi$ that represent a sequence $\bx$ in $\vs$ as an element $\Phi(\bx)$ of the linear space $\T{\vs}$.  
This product is the so-called \emph{tensor convolution product}, and is defined for $\bs,\bt \in \T{V}$ as
\begin{align} \label{eq:tensorprod}
\bs \cdot \bt := \big( \sum_{i=0}^m \bs_i\otimes \bt_{m-i} \big)_{m\geq 0} = \big( 1, \bs_1 + \bt_1, \bs_2 + \bs_1\otimes \bt_1 + \bt_2,\ldots \big) \in \T{V}
\end{align}
where $\otimes$ denotes the usual outer tensor product; e.g.~for vectors $u=(u_i),v=(v_i) \in \R^d$ the outer tensor product $u \otimes v$ is the $d\times d$ matrix $(u_iv_j)_{i,j=1,\ldots,d}$. 
We emphasize that like the outer tensor product $\otimes$, the tensor convolution product $\cdot$ is non-commutative, i.e.~$\bs \cdot \bt \neq \bt \cdot \bs$.
In a mathematically precise sense, \emph{$\T{V}$ is the most general algebra that contains $V$; it is a ``free construction''}. 
Since $\T{V}$ is realized as series of tensors of increasing degree, the \emph{free algebra} $\T{V}$ is also known as the \emph{tensor algebra} in the literature.
Appendix~\ref{app:tensor} contains background on tensors and further examples. 
% We now use this non-commutativity to capture the ordered structure in a sequence.
% \begin{definition} \label{def:phi}
%   Define
%   \begin{align}\label{eq:defphi}
%    \Phi: \Seq{\vs} \rightarrow \T{\vs},\quad \bx \mapsto \prod_{i=1}^L (1+ \bx_i) %\EE_{\pi_t}[ \bx_{t_1} \cdots \bx_{t_m}]
%   \end{align}
%   where $1=(1,0,0^{\otimes 2}, 0^{\otimes 3},\ldots)\in \T{\vs}$ and $\bx_i$ is identified as $(0,\bx_i,0^{\otimes 2}, 0^{\otimes 3},\ldots) \in \T{\vs}$.
% \end{definition}
% It is possible to replace the terms $(1+\bx_i)$ in the definition of $\Phi$ with other embeddings of $\bx_i$ into $\T{V}$; we discuss this in detail in Appendix \ref{app:algebra} but below we will use \eqref{eq:defphi} as it is the simplest and it also generalizes $k$-mers.
% The essential part of Definition~\ref{def:phi} is that the product $\prod_{i}$ is the non-commutative convolution tensor product as defined in \eqref{eq:tensorprod} .
% This methodology of representing sequences by elements of non-commutative tensor algebras is what we call \emph{Seq2Tens}.
% Another way to think about $\Phi$ is to note that $\Phi_m(\bx) $ is proportional to $ \EE[\bx_{i_1}\otimes \cdots \otimes \bx_{i_m}]$ where the expectation is taken by sampling $m$ points uniformly at random without replacement from $\{1,\ldots,L\}$ and putting them into increasing order $i_1<\cdots <i_m$. In Appendix \ref{app:kmers} we show that $\Phi$ is a strict generalization of the string features.  
\paragraph{Lifting static feature maps.}
Step~\ref{itm:lift map} in the construction of $\Phi$ requires turning a given feature map $\phi:\cX \to V$ into a map $\varphi: \cX \to \T{V}$.
%Since $ \T{V} $ is infinite dimensional there is a lot of freedom in choosing the map $ \varphi $.
Throughout the rest of this article we use the lift  
\begin{align} \label{eq:oneplus}
  \varphi(\bx) = (1, \phi(\bx), 0, 0 \ldots ) \in \T{V}.
\end{align}
We discuss other choices in Appendix~\ref{app:algebra}, but attractive properties of the lift \ref{eq:oneplus} are that
\begin{enumerate*}[label=(\alph*)]
\item
  the evaluation of $\Phi$ against low rank tensors becomes a simple recursive formula (Proposition~\ref{prop:rank one}, 
\item 
  it is a generalization of sequence sub-pattern matching as used in string kernels (Appendix~\ref{sec:different length},  
\item 
  despite its simplicity it performs exceedingly well in practice (Section \ref{sec:4}).
\end{enumerate*}

\paragraph{Extending to sequences of arbitrary length.} Steps \ref{itm:lift map} and \ref{itm:lift} in the construction specify how the map $\Phi: \cX \rightarrow \T{V}$ behaves on sequences of length-$1$, that is, single observations. Step \ref{itm:mult} amounts to the requirement that for any two sequences $\bx = (\bx_1, \dots, \bx_K), \by=(\by_1, \dots, \by_L) \in \Seq{V}$, their concatenation defined as $\bz = (\bx_1, \dots, \bx_K, \by_1, \dots, \by_L) \in \Seq{V}$ can be understood in the feature space as (non-commutative) multiplication of their corresponding features
\begin{align}
	\Phi(\bz) = \Phi(\bx) \cdot \Phi(\by). \label{eq:conc_product}
\end{align}
%The benefit of using \eqref{eq:conc_product} is twofold: firstly, since $\T{V}$ is closed under multiplication, it guarantees that we stay in the same feature space for sequences of any length. Secondly, the image of $\Phi$ becomes a subalgebra of $\T{V}$ itself, a property used in showing universality of $\Phi$ in the sense below.
In other words, we inductively extend the lift $\varphi$ to sequences of arbitrary length by starting from sequences consisting of a single observation, which is given in \eqref{eq:mult_varphi}.
Repeatedly applying the definition of the tensor convolution product in \eqref{eq:tensorprod} leads to the following explicit formula

\begin{align}\label{eq:phi formula}
	\Phi_m(\bx) = \sum_{1 \leq i_1 < \dots < i_m \leq L} \bx_{i_1} \otimes \cdots \otimes \bx_{i_m} \in V^{\otimes m}, \hspace{5pt} \Phi(\bx) = (\Phi_m(\bx))_{m \geq 0},
\end{align}
where $\bx=(\bx_1, \dots, \bx_L) \in \Seq{V}$ and the summation is over non-contiguous subsequences of $\bx$. 

\paragraph{Some intuition: generalized pattern matching.}
Our derivation of the feature map $\Phi(\bx)=(1,\Phi_1(\bx), \Phi_2(\bx),\ldots) \in \T{V}$ was guided by general algebraic principles, but \eqref{eq:phi formula} provides an intuitive interpretation.
It shows that for each $m\ge 1$, the entry $\Phi_m(\bx) \in V^{\otimes m}$ constructs a summary of a long sequence $\bx=(\bx_1,\ldots,\bx_L) \in \Seq{V}$ based on subsequences $(\bx_{i_1},\ldots,\bx_{i_m})$ of $\bx$ of length-$m$.
It does this by taking the usual outer tensor product $\bx_{i_1} \otimes \cdots \otimes \bx_{i_m} \in V^{\otimes m}$ and summing over all possible subsequences.
This is completely analogous to how string kernels provide a structured description of text by looking at non-contiguous substrings of length-$m$ (indeed, Appendix \ref{sec:different length} makes this rigorous).
However, the main difference is that the above construction works for arbitrary sequences and not just sequences of discrete letters.
Readers with less mathematical background might simply take this as motivation and regard \eqref{eq:phi formula} as definition. 
However, the algebraic background allows to prove that $\Phi$ is universal, see Theorem \ref{thm:univ} below.
%The algebraic background allows us to prove not only that no information is lost but that $\Phi$ is universal, see Theorem \ref{thm:univ} below.

\paragraph{Universality.}
A function $ \phi : \cX \to V $ is said to be \emph{universal for }$\cX$ if all continuous functions on $ \cX $ can be approximated as linear functions on the image of $ \phi $.
One of the most powerful features of neural nets is their universality~\citep{hornik1991approximation}.
A very attractive property of $\Phi$ is that it preserves universality: if $\phi:\cX \to V$ is universal for $\cX$, then $\Phi: \Seq{X} \to \T{V}$ is universal for $\Seq{\cX}$. 
To make this precise, note that $V^{\otimes m}$ is a linear space and therefore any $\ell= (\ell_0,\ell_1,\ldots,\ell_M,0,0,\ldots) \in \T{V}$ consisting of $M$ tensors $\ell_m\in V^{\otimes m}$, yields a linear functional on $\T{V}$; e.g.~if $V=\R^d$ and we identify $\ell_m$ in coordinates as $\ell_m=(\ell_m^{i_1,\ldots,i_m})_{i_1,\ldots,i_m \in \{1,\ldots,d\}}$ then
\begin{align}\label{eq: coordinates functional}
  \langle \ell, \bt \rangle:= \sum_{m=0}^M \langle \ell_m, \bt_m \rangle = \sum_{m = 0}^M \sum_{i_1,\ldots,i_m \in \{1,\ldots,d\}} \ell^{i_1,\ldots,i_m}_m \bt_m^{i_1,\ldots,i_m}.
\end{align}
Thus linear functionals of the feature map $\Phi$, are real-valued functions of sequences.
Theorem \ref{thm:univ} below shows that any continuous function $f:\Seq{\cX}\to \R$ can by arbitrary well approximated by a $\ell \in \T{V}$, $f(\bx) \approx \langle \ell, \Phi(\bx) \rangle$.
\begin{thm} \label{thm:univ}
	Let $ \phi : \cX \to V $  be a universal map with a lift that satisfies some mild constraints, then the following map is universal:
	\begin{align}
\Phi:	\mathrm{Seq}(\cX) \to \T{V}, \quad \bx \mapsto \Phi(\bx).
	\end{align}
\end{thm}
A detailed proof and the precise statement of Theorem \ref{thm:univ} is given in Appendix \ref{app:algebra}.

\section{Approximation by low-rank linear functionals} \label{sec:3}
\paragraph{The combinatorial explosion of tensor coordinates and what to do about it.}
The universality of $\Phi$ suggests the following approach to represent a function $f:\Seq{\cX} \to \R$ of sequences: First compute $\Phi(\bx)$ and then optimize over $\ell$ (and possibly also the hyperparameters of $\phi$) such that $f(\bx) \approx \langle \ell, \Phi(\bx) \rangle=\sum_{m=0}^M \langle \ell_m, \Phi_m(\bx) \rangle$.   
Unfortunately, tensors suffer from a combinatorial explosion in complexity in the sense that even just storing $ \Phi_m(\bx) \in  V^{\otimes m} \subset \T{V} $ requires $O(\operatorname{dim}(V)^{m})$ real numbers.
% Unfortunately, this quickly becomes computationally infeasible: if $v \in V$ is a vector in a $d$-dimensional space $V$, then $d^m$ real numbers required to describe and store $v^{\otimes m} \in V^{\otimes m}$. 
Below we resolve this computational bottleneck as follows: in Proposition \ref{prop:rank one} we show that for a special class of low-rank elements $\ell \in \T{V}$, the functional $\bx \mapsto \langle \ell, \Phi(\bx) \rangle$ can be efficiently computed in both time and memory.
This is somewhat analogous to a kernel trick since it shows that $\langle \ell, \Phi(\bx) \rangle$ can be cheaply computed without explicitly computing the feature map $\Phi(\bx)$. 
However, Theorem~\ref{thm:univ} guarantees universality under no restriction on $\ell$, thus restriction to rank-$1$ functionals limits the class of functions $f(\bx)$ that can be approximated.
%However, the linear functional $\ell$ might need to be of very high degree which makes it costly, unstable, and prone to overfitting.
% We now leverage idea 3, namely that compositions of simple objects can give efficient and robust approximations to complex objects.
% Theorem \ref{thm:compose} makes this precise in our context of sequences and tensor algebra functionals; the proof requires delicate concepts from non-commutative algebra and is given in Appendix \ref{app:iterated} where we also give more examples. 
Nevertheless, by iterating these ``low-rank functional'' constructions in the form of sequence-to-sequence transformations this can be ameliorated.
We give the details below but to gain intuition, we invite the reader to think of this iteration analogous to stacking layers in a neural network: each layer is a relatively simple non-linearity (e.g.~a sigmoid composed with an affine function) but by composing such layers, complicated functions can be efficiently approximated.
% In this analogy, one rank-$1$ functional $\langle \ell, \Phi(\bx)\range$ corresponds to a single layer which represents a relative simple non-linear function of sequences but by composing such rank-$1$ functionals, very complex sequence functions can be efficiently approximated. 

\paragraph{Rank-1 functionals are computationally cheap.} 
Degree $m=2$ tensors are matrices and low-rank (LR) approximations of matrices are widely used in practice \citep{Udell19} to address the quadratic complexity. 
The definition below generalizes the rank of matrices (tensors of degree $m=2$) to tensors of any degree $m$. 
\begin{definition} \label{def:rank}
	The \emph{rank} (also called \emph{CP rank} \citep{carroll1970analysis}) of a degree-$m$ tensor $ \bt_m \in V^{\otimes m} $ is the smallest number $ r \geq 0 $ such that one may write
	\begin{align}
	\bt_m = \sum_{i=0}^r \bv_i^1 \otimes \cdots \otimes \bv_i^m, \quad \bv_i^1, \ldots, \bv_i^m \in V.
	\end{align}
  We say that $\bt = (\bt_m)_{m \ge0} \in \T{V}$ has rank-$1$ (and degree-$M$) if each $\bt_m \in V^{\otimes m}$ is a rank-$1$ tensor and $\bt_i=0$ for $i>M$.
\end{definition}
\begin{remark}
	For $\bx = (\bx_1, \dots, \bx_L) \in \Seq{V}$, the rank $r_m \in \NN$ of $\Phi_m(\bx)$ satisfies $r_m \leq {L \choose m}$, while the rank and degree $r, d \in \NN$ of $\Phi(\bx)$ satisfy $r \leq {L \choose K}$ for $K = \left\lfloor \frac{L}{2} \right\rfloor$ and $d \leq L$.
\end{remark}
%Definition \ref{def:rank} is also known as the \textit{outer-product rank}.% and should not be confused with the \textit{modal} or \textit{multilinear rank}; see e.g. \cite{de2008tensor, chen2009tensor}.
A direct calculation shows that if $\ell$ is of rank-$1$, then $\langle \ell, \Phi(\bx) \rangle$ can be computed very efficiently by inner product evaluations in $V$.
\begin{proposition}\label{prop:rank one}
  Let $\ell=(\ell_m)_{m \ge 0} \in \T{V}$ be of rank-$1$ and degree-$M$.
  If $\phi $ is lifted to $\varphi$ as in \eqref{eq:oneplus}, then
  \begin{align}
    \langle \ell, \Phi(\bx)\rangle &=
    % \langle \ell, \prod_{t=1}^T (1+\bx_t) \rangle = \langle \ell, 1+\underbrace{\sum_{i_1<i_2}\bx_{i_1}\otimes \bx_{i_2}}_{\in V^{\otimes 2}\subset \T{V}} + \underbrace{\sum_{i_1<i_2<_3} \bx_{i_1}\otimes\bx_{i_2}\otimes \bx_{i_3}}_{\in V^{\otimes 3} \subset \T{V}} +\cdots\rangle\\
    \sum_{m=0}^M \sum_{1 \le i_1 < \cdots < i_m \leq L} \prod_{k=1}^m \langle \bv_k^m, \phi(\bx_{i_k})\rangle \label{eq:sum}
  \end{align}
where $\ell_m=\bv_1^m \otimes \dots \otimes \bv_m^m \in V^{\otimes m}$, $\bv_i^m \in V$ and $m=0,\ldots,M$.
  % The feature map $\Phi : \Seq{\cX} \to \T{V}$, $\bx \mapsto (\Phi_m(\bx))_{m\geq 0}$ can be computed as
  % \begin{align} \label{eq:phi_m}
  %   \Phi_m(\bx) = \sum_{1 \le i_1 < \cdots < i_m \le L} \varphi(\bx_{i_1}) \otimes \cdots \otimes \varphi(\bx_{i_m}) \in V^{\otimes m}.
  % \end{align}
\end{proposition}
Note that the inner sum is taken over all non-contiguous subsequences of $\bx$ of length-$m$, analogously to $m$-mers of strings and we make this connection precise in Appendix~\ref{sec:different length}; the proof of Proposition \ref{prop:rank one} is given in Appendix~\ref{app:quasi}. While \eqref{eq:sum} looks expensive, by casting it into a recursive formulation over time, it can be computed in $O(M^2 \cdot L \cdot d)$ time and $O(M^2 \cdot (L + c))$ memory, where $d$ is the inner product evaluation time on $V$, while $c$ is the memory footprint of a $v \in V$. This can further be reduced to $O(M \cdot L \cdot d)$ time and $O(M \cdot (L + c))$ memory by an efficient parametrization of the rank-$1$ element $\ell \in \T{V}$. We give further details in Appendices~\ref{app:recursive}, \ref{app:algs}, \ref{app:complexity}.
%However, instead of just counting occurrences of letters, we now represent a subsequence as an element of $V^{\otimes m}$.
% one may compute $\langle \ell, \Phi(\bx)\rangle$ very efficiently without computing all of $\Phi(\bx)=\left( \Phi_m(\bx) \right)_{m \ge 0 } \in \T{V}$ since a direct calculation (see Appendix \ref{app:quasi} and \ref{app:recursive}) shows that
% \begin{align}
%   \langle \ell, \Phi(\bx)\rangle &=
%                                    %\langle \ell, \prod_{t=1}^T (1+\bx_t) \rangle = \langle \ell, 1+\underbrace{\sum_{i_1<i_2}\bx_{i_1}\otimes \bx_{i_2}}_{\in V^{\otimes 2}\subset \T{V}} + \underbrace{\sum_{i_1<i_2<_3} \bx_{i_1}\otimes\bx_{i_2}\otimes \bx_{i_3}}_{\in V^{\otimes 3} \subset \T{V}} +\cdots\rangle\\
% \sum_{1 \le i_1 < \cdots < i_m \leq L} \prod_{k=1}^m \langle \bz_k, \bx_{i_k} \rangle. \label{eq:sum}
% \end{align}
%It is possible to compute $ \langle \Phi(\bx), z \rangle $ using Equation \eqref{eq:sum} and let $ m $ be very high to capture higher order information about $ \bx $.
%We call $ \bx \mapsto \langle \ell, \Phi_m(\bx) \rangle$ a rank-$r$ linear function if $\ell$ is a rank-$r$ tensor.  

\paragraph{Low-rank Seq2Tens maps.}
The composition of a linear map $\cL: \T{V} \rightarrow \R^{N}$ with $\Phi$ can be computed cheaply in parallel using \eqref{eq:sum} when $\cL$ is specified through a collection of $N \in \NN$ rank-$1$ elements $\ell^1,\ldots,\ell^N \in \T{V}$ such that 
\begin{align}
	\tilde\Phi_{\tilde\theta}(\bx_1, \dots, \bx_L) := \cL\circ \Phi(\bx_1,\ldots,\bx_L) = (\langle \ell^j, \Phi(\bx_1,\ldots,\bx_L))_{j=1}^N \in \R^{N}.
\end{align}
We call the resulting map $\tilde\Phi_{\tilde\theta}: \Seq{\cX} \rightarrow \R^N$ a \textit{\textbf{L}ow-rank \textbf{S}eq\textbf{2T}ens} map of width-$N$ and order-$M$, where $M \in \NN$ is the maximal degree of $\ell^1, \dots, \ell^N$ such that $\ell^j_i = 0$ for $i > M$. The LS2T map is parametrized by \begin{enumerate*}[label=(\arabic*)] \item the component vectors $\bv_{j,m}^k \in V$ of the rank-$1$ elements $\ell^j_m = \bv_{j,m}^1 \otimes \cdots \otimes \bv_{j,m}^{m}$, \item by any parameters $\theta$ that the static feature map $\phi_\theta: \cX \rightarrow V$ may depend on. We jointly denote these parameters by $\tilde\theta = (\theta, \ell^1, \dots, \ell^N)$ \end{enumerate*}. In addition, by the subsequent composition of $\tilde\Phi_{\tilde \theta}$ with a linear functional $\R^{N} \rightarrow \R$, we get the following function subspace as hypothesis class for the LS2T
\begin{align}
	\tilde\cH = \big\{\langle \sum_{j=1}^N \alpha_j \ell^j, \Phi(\bx_1, \dots, \bx_L) \rangle \,\vert\, \alpha_j \in \R \big\} \subsetneq \cH = \big\{\langle \ell, \Phi(\bx_1, \dots, \bx_L) \,\vert\, \ell \in \T{V}\big\}	
\end{align}
Hence, we acquire an intuitive explanation of the (hyper)parameters: the width of the LS2T, $N \in \NN$ specifies the maximal rank of the low-rank linear functionals of $\Phi$ that the LS2T can represent, while the span of the rank-$1$ elements, $\spn(\ell^1, \dots, \ell^N)$ determine an $N$-dimensional subspace of the dual space of $\T{V}$ consisting of at most rank-$N$ functionals.

Recall now that without rank restrictions on the linear functionals of Seq2Tens features, Theorem~\ref{thm:univ} would guarantee that any real-valued function $f:\Seq{\cX}\to \R$ could be approximated by $f(\bx) \approx \langle \ell, \Phi(\bx_1, \dots, \bx_L) \rangle$. As pointed out before, the restriction of the hypothesis class to low-rank linear functionals of $\Phi(\bx_1, \dots, \bx_L$) would limit the class of functions of sequences that can be approximated. To ameliorate this, we use LS2T transforms in a sequence-to-sequence fashion that allows us to stack such low-rank functionals, significantly recovering expressiveness.
 
\paragraph{Sequence-to-sequence transforms.}
% in feature space
% \begin{align} \label{eq:stream_sigs}
% \Seq{V} &\rightarrow \Seq{\T{V}},\quad(\bx_1,\bx_2,\ldots, \bx_L) \mapsto \Big(\Phi(\bx_1),\Phi(\bx_1, \bx_2), \ldots, \Phi(\bx_1,\ldots,\bx_L) \Big).
% \end{align}
% As pointed out above, this is very expensive computationally. 
%Motivated by the empirical successes of stacked RNNs \cite{graves2013speech, graves2013hybrid, Sutskever2014Seq2Seq}, we take build sequence-to-sequence transformations.
We can use LS2T to build sequence-to-sequence transformations in the following way: fix the static map $\phi_{\theta}: \cX \to V$ parametrized by $\theta$ and rank-$1$ elements such that $\tilde\theta = (\theta, \ell^1, \dots, \ell^N)$ and apply the resulting LS2T map $\tilde\Phi_{\tilde \theta}$ over expanding windows of $\bx$:
\begin{align} \label{eq:seq2seq_lls2t}
	\Seq{\cX} \to \Seq{\R^N},\quad \bx \mapsto \big(\tilde\Phi_{\tilde\theta}(\bx_1), \tilde\Phi_{\tilde\theta}(\bx_1, \bx_2), \dots, \tilde\Phi_{\tilde\theta}(\bx_1, \dots, \bx_L)\big).
\end{align}
Note that the cost of computing the expanding window sequence-to-sequence transform in \eqref{eq:seq2seq_lls2t} is no more expensive than computing $\tilde\Phi_{\tilde\theta}(\bx_1, \dots, \bx_L)$ itself due to the recursive nature of our algorithms, for further details see Appendices~\ref{app:recursive}, \ref{app:algs}, \ref{app:complexity}.

% Recall now that without the restricting coordinates of $\tilde\Phi_{\tilde\theta}$ to rank-$1$ functionals, Theorem~\ref{thm:univ} would guarantee that any multivariate function $f:\Seq{\cX}\to \R^{N}$ can be approximated by simply evaluating the resulting sequence in $\R^{N}$ at the endpoint $t=L$, i.e.~$f(\bx)\approx \tilde\Phi_{\tilde\theta}(\bx_1,\ldots,\bx_L)$.
% The restriction to rank-$1$ functionals limits the class of functions but by stacking such sequence-to-sequence transformations we recover an expressive model that is computationally cheap since each transformation is cheap. 
% Note that the sequence transformation \eqref{eq:lrSeq2Tens_causal} is parametrized by the parameters $\theta_1$ of the ``static feature map'' $\varphi_{\theta_1}$ and the $n$ rank-$1$ tensors $(\ell_1,\ldots,\ell_{n_1})$.
%( \langle \ell_{1}, \Phi(\bx_1,\ldots,\bx_i) \rangle, \ldots, \langle \ell_{n},\Phi(\bx_1,\ldots,\bx_i) \rangle)_{i=1}^L.
% We refer to $\theta = (\ell_j)_{j=1}^n$ as the set of weights of the NLST layer.
%In analogy with NNs, $n$ is called the \emph{width} of the layer $\tilde \Phi_\theta$ and the maximal degree of $\ell_1, \ldots, \ell_n $ is its \emph{order}.
\paragraph{Deep sequence-to-sequence transforms.}
Inspired by the empirical successes of deep RNNs \citep{graves2013speech, graves2013hybrid, Sutskever2014Seq2Seq}, we iterate the transformation \ref{eq:seq2seq_lls2t} $D$-times:
\begin{align}\label{eq:stacked seq2seq}
  \Seq{\cX} \rightarrow \Seq{\R^{N_1}} \rightarrow \Seq{\R^{N_2}} \rightarrow \cdots \rightarrow \Seq{\R^{N_D}}. %\xrightarrow{\Phi^{1}} \R 
\end{align}
Each of these mappings $\Seq{\R^{N_i}} \rightarrow \Seq{\R^{N_{i+1}}}$ is parametrized by the parameters $\tilde\theta_i$ of a static feature map $\phi_{\theta_i}$ and a linear map $\cL_i$ specified by $N_i$ rank-$1$ elements of $\T{V}$; these parameters are collectively denoted by $\tilde\theta_i= (\theta_i, \ell_i^1, \dots, \ell_i^{N_i})$. 
Evaluating the final sequence in $\Seq{\R^{N_D}}$ at the last observation-time $t=L$, we get the deep LS2T map with depth-$D$
\begin{align}\label{eq:SeqFeatures}
  \tilde\Phi_{\tilde \theta_1,\ldots,\tilde \theta_{D}}:\Seq{\cX} \rightarrow \R^{n_D}.
\end{align}
% which, by Theorem \ref{thm:compose} can cheaply approximate high degree functionals on $ \Phi $. 
% We refer to $\varphi^{\theta_1,\ldots, \theta_{D}}$ as the $D$-times stacked \emph{NLST} layer. 
Making precise how the stacking of such low-rank sequence-to-sequence transformations approximates general functions requires more tools from algebra, and we provide a rigorous quantitative statement in Appendix \ref{app:iterated}. 
Here, we just appeal to the analogy made with adding depth in neural networks mentioned earlier and empirically validate this in our experiments in Section~\ref{sec:4}. 

\section{Building neural networks with LS2T layers} \label{sec:4}
%By parsing the original sequence and applying $\Phi$ ``online'' and repeating this $D$ times, we have created $D-1$ sequence-to-sequence transforms to turn $\bx \in \Seq{\vs}$ into a sequence in the (very rich) linear space $\vs_D$, that is an element of $\Seq{\vs_D}$ see~\eqref{eq:seq2seq}.
% Computing the map $ \Phi $ in its entirety is usually infeasible as it suffers tremendously from the curse of dimensionality.
% In fact, even if one only computes the tensor series up to degree $M$, one would need $O(d^{M})$ coordinates to store all tensors, which is too expensive for most applications.
% Thankfully Equation \eqref{eq:sum} shows that low-rank functionals can be very quickly computed.
The Seq2Tens map $\Phi$ built from a static feature map $\phi$ is universal if $\phi$ is universal, Theorem~\ref{thm:univ}.
NNs form a flexible class of universal feature maps with strong empirical success for data in $\cX=\R^d$, and thus make a natural choice for $\phi$.
Combined with standard deep learning constructions, the framework of Sections~\ref{sec:2} and \ref{sec:3} can build modular and expressive layers for sequence learning.
%We then proceed to use the mathematical objects we introduced in Section~\ref{sec:2} to generalize well-known methodologies to construct scalable and modular layers to build neural networks for sequences. %that capture their structure efficiently by the non-commutative structure of the tensor algebra.
\paragraph{Neural LS2T layers.}
The simplest choice among many is to use as static feature map $\phi:\cX=\R^d \to \R^h$ a feedforward network with depth-$P$, $\phi = \phi_{P} \circ \cdots \circ \phi_1$ where $\phi_{j}(\bx) = \sigma(\mathbf{W}_j \bx + \mathbf{b}_j)$ for $\mathbf{W}_j \in \R^{h \times d}$, $\mathbf{b}_j \in \R^{h}$.
We can then lift this to a map $\varphi: \R^d \to \T{\R^h}$ as prescribed in~\eqref{eq:oneplus}.
Hence, the resulting LS2T layer $\bx \mapsto (\tilde\Phi_{\tilde\theta} (\bx_1,\ldots,\bx_i))_{i=1,\ldots,L} $ is a sequence-to-sequence transform $\Seq{\R^d} \rightarrow \Seq{\R^h}$ that is parametrized by $\tilde \theta=(\bW_1, \bb_1, \dots, \bW_P, \bb_P, \ell_1^1, \dots, \ell^{N_1}_1)$.
% By stacking such sequence transforms $D$-times and evaluating at the endpoint as in~\eqref{eq:stacked seq2seq}, we get a map $\tilde\Phi_{\tilde \theta_1,\ldots, \tilde \theta_{D}}:\Seq{\R^d} \to \R^N_D$ that is differentiable with respect to the collections of parameters $(\tilde \theta_j)_{j \in \{1,\ldots,D\}}$, all of which can then be chosen by e.g.~gradient based optimization of some loss function. 
% This can also be combined with standard preprocessing techniques such as adding lags andconvolutions (see Appendix \ref{app:variations}). 

\paragraph{Bidirectional LS2T layers.} The transformation in \eqref{eq:seq2seq_lls2t} is completely causal in the sense that each step of the output sequence depends only on past information.
For generative models, it can behove us to make the output depend on both past and future information, see~\citet{graves2013hybrid, baldi1999exploiting,li2018disentangled}.
Similarly to bidirectional RNNs and LSTMs \citep{schuster1997bidirectional, Graves2005BiLSTM}, we may achieve this by defining a bidirectional layer, 
\begin{align}
    \tilde\Phi^{\operatorname{b}}_{(\tilde\theta_1, \tilde\theta_2)}(\mathbf{x}): \Seq{\R^d} \rightarrow \Seq{\R^{N + N^{\prime}}}, \quad \mathbf{x} \mapsto (\tilde\Phi_{\tilde\theta_1}(\bx_1, \ldots, \bx_i), \tilde\Phi_{\tilde\theta_2}(\bx_i, \ldots, \bx_L))_{i=1}^L.
\end{align}
The sequential nature is kept intact by making the distinction between what classifies as past (the first $N$ coordinates) and future (the last $N^\prime$ coordinates) information. This amounts to having a form of precognition in the model, and has been applied in e.g. dynamics generation \citep{li2018disentangled}, machine translation \citep{sundermeyer2014translation}, and speech processing \citep{graves2013hybrid}.

\paragraph{Convolutions and LS2T.} We motivate to replace the time-distributed feedforward layers proposed in the paragraph above by temporal convolutions (CNN) instead. Although theory only requires the preprocessing layer of the LS2T to be a static feature map, we find that it is beneficial to capture some of the sequential information in the preprocessing layer as well, e.g.~using CNNs or RNNs. From a mathematical point of view, CNNs are a straightforward extension since they can be interpreted as time-distributed feedforward layers applied to the input sequence augmented with a $p \in \NN$ number of its lags for CNN kernel size $p$ (see Appendix \ref{app:variations} for further discussion). 

In the following, we precede our deep LS2T blocks by one or more CNN layers. Intuitively, CNNs and LS2Ts are similar in that both transformations operate on subsequences of their input sequence. The main difference between the two lies in that \emph{CNNs operate on contiguous subsequences}, and therefore, capture local, short-range nonlinear interactions between timesteps; \emph{while LS2Ts (\eqref{eq:sum}) use all non-contiguous subsequences}, and hence, learn global, long-range interactions in time. This observation motivates that the inductive biases of the two types of layers (local/global time-interactions) are highly complementary in nature, and we suggest that the improvement in the experiments on the models containing vanilla CNN blocks are due to this complementarity.

\section{Experiments}\label{sec:5}
We demonstrate the modularity and flexibility of the above LS2T and its variants by applying it to 
\begin{enumerate*}[label=(\roman*)] 
\item  multivariate time series classification,
\item mortality prediction in healthcare, 
\item  generative modelling of sequential data. 
\end{enumerate*}
In all cases, we take a strong baseline model (FCN and GP-VAE, as detailed below) and upgrade it with LS2T layers. As Thm.~\ref{thm:univ} requires the Seq2Tens layers to be preceded by at least a static feature map, we expect these layers to perform best as an add-on on top of other models, which however can be quite simple, such as a CNN. The additional computation time is negligible (in fact, for FCN it allows to reduce the number of parameters significantly, while retaining performance), but it can yield substantial improvements.
This is remarkable, since the original models are already state-of-the-art on well-established (frequentist and Bayesian) benchmarks. 
%See Appendix \ref{app:details} for more details on the implementation.
\subsection{Multivariate time series classification} \label{subseq:4_tsc}

As the first task, we consider multivariate time series classification (TSC) on an archive of benchmark datasets collected by \citet{baydogan2015multivarate}. Numerous previous publications report results on this archive, which makes it possible to compare against several well-performing competitor methods from the TSC community. These baselines are detailed in Appendix \ref{app:tsc}. This archive was also considered in a recent popular survey paper on DL for TSC \citep{Fawaz2019DLforTSC}, from where we borrow the two best performing models as DL baselines: FCN and ResNet. 
% We provide further details on the experiment, baselines and datasets in Appendix \ref{app:tsc}.
The FCN is a fully convolutional network which stacks 3 convolutional layers of kernel sizes $(8, 5, 3)$ and filters $(128, 256, 128)$ followed by a global average pooling (GAP) layer, hence employing global parameter sharing. We refer to this model as \FCN{128}. The ResNet is a residual network stacking 3 FCN blocks of various widths with skip-connections in between \citep{he2016deep} and a final GAP layer.

The FCN is an interesting model to upgrade with LS2T layers, since the LS2T also employs parameter sharing across the sequence length, and as noted previously, convolutions are only able to learn local interactions in time, that in particular makes them ill-suited to picking up on long-range autocorrelations, which is exactly where the LS2T can provide improvements. 
% As a baseline we include several methods for which accuracies are available on this collection: DTW\textsubscript{i} \cite{Sakoe1978DTW}, ARKernel \cite{Cuturi2011AR}, SMTS \cite{baydogan2015learning}, LPS \cite{Baydogan2015TimeSR}, gRSF \cite{karlsson2016generalized}, mvARF \cite{tuncel2018autoregressive}, MUSE \cite{Schfer2017MUSE}, MLSTMFCM\footnote{For MLSTMFCN, the results are the same as those reported in \cite{Schfer2017MUSE}}. \cite{Karim2019LSTMFCN}. We also include two deep learning baselines from \cite{Wang2017TSCfromScratch}, FCN and ResNet, since they were found to be the best performing deep learning methods for TSC in \cite{Fawaz2019DLforTSC}. 
% \paragraph{LS2T\textsuperscript{3} and FCN+LS2T\textsuperscript{3}.}
As our models, we consider three simple architectures: \begin{enumerate*}[label=(\roman*)] \item \LStwoTwidth{64}{3} stacks $3$ LS2T layers of order-$2$ and width-$64$; \item \FCNLStwoTwidth{64}{64}{3} precedes the \LStwoTwidth{64}{3} block by an \FCN{64} block; a downsized version of \FCN{128}; \item \FCNLStwoTwidth{128}{64}{3} uses the full \FCN{128} and follows it by a \LStwoTwidth{64}{3} block as before. \end{enumerate*} Also, both FCN-LS2T models employ skip-connections from the input to the LS2T block and from the FCN to the classification layer, allowing for the LS2T to directly see the input, and for the FCN to directly affect the final prediction. These hyperparameters were only subject to hand-tuning on a subset of the datasets, and the values we considered were $H, N \in \{32, 64, 128\}$, $M \in \{2, 3, 4\}$ and $D \in \{1, 2, 3\}$, where $H, N \in \NN$ is the FCN and LS2T width, resp., while $M \in \NN$ is the LS2T order and $D \in \NN$ is the LS2T depth. We also employ techniques such as time-embeddings \citep{Liu2018coordConv}, sequence differencing and batch normalization, see Appendix \ref{app:variations}; Appendix \ref{app:tsc} for further details on the experiment and Figure \ref{fig:ls2t_architectures} in thereof for a visualization of the architectures.

\begin{table}
	\caption{Posterior probabilities given by a Bayesian signed-rank test comparison of the proposed methods against the baselines. $\{>\}$, $\{<\}$, $\{=\}$ refer to the respective events that the row method is better, the column method is better, or that they are equivalent.}
	\label{table:bayes_signed_rank_probs}
	\vspace{-10pt}
	\begin{center}
	\begin{small}
	\begin{sc}
	\resizebox{\textwidth}{!}{
		\begin{tabular}{lccccccccc}
			\toprule
			\multirow{2}{*}{Model} & \multicolumn{3}{c}{\LStwoTwidth{64}{3}} & \multicolumn{3}{c}{\FCNLStwoTwidth{64}{64}{3}} & \multicolumn{3}{c}{\FCNLStwoTwidth{128}{64}{3}} \\
			\cmidrule{2-10}
			& $p(>)$ & $p(=)$ & $p(<)$ & $p(>)$ & $p(=)$ & $p(<)$ & $p(>)$ & $p(=)$ & $p(<)$ \\
			\midrule 
			SMTS \citep{baydogan2015learning} & $0.180$ & $0.000$ & $\mathbf{0.820}$ & $0.010$ & $0.000$ & $\mathbf{0.990}$ & $0.008$ & $0.000$ & $\mathbf{0.992}$ \\
			LPS \citep{Baydogan2015TimeSR} & $0.191$ & $0.002$ & $\mathbf{0.807}$ & $0.012$ & $0.001$ & $\mathbf{0.987}$ & $0.006$ & $0.001$ & $\mathbf{0.993}$ \\
			mvARF \citep{tuncel2018autoregressive} & $0.011$ & $0.140$ & $\mathbf{0.849}$ & $0.000$ & $0.126$ & $\mathbf{0.874}$ & $0.000$ & $0.088$ & $\mathbf{0.912}$ \\
			DTW \citep{Sakoe1978DTW} & $0.033$ & $0.000$ & $\mathbf{0.967}$ & $0.001$ & $0.000$ & $\mathbf{0.999}$ & $0.000$ & $0.000$ & $\mathbf{1.000}$ \\
			ARKernel \citep{Cuturi2011AR} & $0.100$ & $0.097$ & $\mathbf{0.803}$ & $0.000$ & $0.021$ & $\mathbf{0.979}$ & $0.000$ & $0.015$ & $\mathbf{0.985}$ \\
			gRSF \citep{karlsson2016generalized} & $0.481$ & $0.011$ & $\mathbf{0.508}$ & $0.028$ & $0.013$ & $\mathbf{0.960}$ & $0.022$ & $0.013$ & $\mathbf{0.965}$ \\
			MUSE \citep{Schfer2017MUSE} & $0.405$ & $0.128$ & $\mathbf{0.467}$ & $0.001$ & $0.074$ & $\mathbf{0.925}$ & $0.001$ & $0.077$ & $\mathbf{0.922}$ \\
			MLSTMFCN \citep{Karim2019LSTMFCN} & $\mathbf{0.916}$ & $0.043$ & $0.041$ & $0.123$ & $0.071$ & $\mathbf{0.807}$ & $0.055$ & $0.110$ & $\mathbf{0.835}$ \\
			$\text{FCN}_{128}$ \citep{Wang2017TSCfromScratch} & $\mathbf{0.998}$ & $0.002$ & $0.000$ & $0.363$ & $0.186$ & $\mathbf{0.451}$ & $0.169$ & $0.011$ & $\mathbf{0.820}$ \\
			ResNet \citep{Wang2017TSCfromScratch} & $\mathbf{0.998}$ & $0.002$ & $0.001$ & $0.056$ & $0.240$ & $\mathbf{0.704}$ & $0.016$ & $0.048$ & $\mathbf{0.935}$ \\
			\LStwoTwidth{64}{3} & - & - & - & $0.000$ & $0.001$ & $\mathbf{0.999}$ & $0.000$ & $0.001$ & $\mathbf{0.999}$ \\
			\FCNLStwoTwidth{64}{64}{3} & $\mathbf{0.999}$ & $0.001$ & $0.000$ & - & - & - & $0.020$ & $0.387$ & $\mathbf{0.593}$ \\
			\bottomrule
		\end{tabular}}
	\end{sc}
	\end{small}
	\end{center}
\vspace{-20pt}
\end{table}

\paragraph{Results.}
We trained the models, \FCN{128}, ResNet, \LStwoTwidth{64}{3}, \FCNLStwoTwidth{64}{64}{3}, \FCNLStwoTwidth{128}{64}{3} on each of the $16$ datasets $5$ times while results for other methods were borrowed from the cited publications. In Appendix \ref{app:tsc}, Figure \ref{fig:box_and_cd} depicts the box-plot of distributions of accuracies and a CD diagram using the Nemenyi test \citep{nemenyi1963distribution}, while Table \ref{table:classification_accuracies} shows the full list of results.
Since mean-ranks based tests raise some paradoxical issues~\citep{Benavoli2016}, it is customary to conduct pairwise comparisons using frequentist \citep{demvsar2006statistical} or Bayesian \citep{benavoli2017time} hypothesis tests. We adopted the Bayesian signed-rank test from \citet{benavoli2014bayesian}, the posterior probabilities of which are displayed in Table \ref{table:bayes_signed_rank_probs}, while the Bayesian posteriors are visualized on Figure \ref{fig:baycomp1} in App.~\ref{app:tsc}. The results of the signed-rank test can be summarized as follows: \begin{enumerate*}[label=(\arabic*)] \item \LStwoTwidth{64}{3} already outperforms some classic TS classifiers with high probability ($p \geq 0.8$), but it is not competitive with other DL classifiers.  This observation is not surprising since even theory requires at least a static feature map to precede the LS2T. \item \FCNLStwoTwidth{64}{64}{3} outperforms almost all models with high probability ($p \geq 0.8$), except for ResNet (which is stil outperformed by $p \geq 0.7$), \FCN{128} and \FCNLStwoTwidth{128}{64}{3}. When compared with \FCN{128}, the test is unable to decide between the two, which upon inspection of the individual results in Table \ref{table:classification_accuracies} can be explained by that on some datasets the benefit of the added LS2T block is high enough that it outweighs the loss of flexibility incurred by reducing the width of the FCN - arguably these are the datasets where long-range autocorrelations are present in the input time series, and picking up on these improve the performance - however, on a few datasets the contrary is true. \item Lastly, \FCNLStwoTwidth{128}{64}{3}, \emph{outperforms all baseline methods with high probability ($p \geq 0.8$)}, and hence successfully improves on the \FCN{128} via its added ability to learn long-range time-interactions.
We remark that \emph{\FCNLStwoTwidth{64}{64}{3} has fewer parameters than \FCN{128} by more than 50\%}, hence we managed to compress the FCN to a fraction of its original size, while on average still slightly improving its performance, a nontrivial feat by its own accord.
\end{enumerate*}

% The results of the signed-rank test indicate that while LS2T\textsuperscript{3} is better with moderate probability ($p \geq 0.6)$ than 4 of the baselines, $\text{FCN-LS2T}^3$ performs better than all with high probability ($p \geq 0.75$), except for MLSTMFCN which is only slightly outperformed by $\text{FCN-LS2T}^3$. 
% Note however that \emph{$\text{FCN-LS2T}^3$ has fewer parameters than MLSTMFCN by more than $60\%$}, see Table \ref{table:params} for a parameter comparison. Also, as MLSTMFCN is the concatenation of an FCN with an LSTM \citep{Karim2019LSTMFCN}, the LSTM layer can bottleneck the computations for very long sequences as it scales linearly, while \emph{LS2T layers effectively scale sublinearly in sequence length}, Appendix \ref{app:complexity} for a computational comparison. In conclusion, we have verified that LS2T\textsuperscript{3} already performs well on some classification tasks, and by preceding it with an FCN, its performance is elevated to outperforming all baseline models. This observation is not surprising, as Thm. \ref{thm:univ} warrants, the full strength of the LS2T layer shines when it is preceded by at least some state-space nonlinearities. Hence, we expect LS2Ts to find their best use as a modular building block being part of larger models. We also suggest the improvement on the vanilla FCN composed of CNN blocks is due to the LS2T enhancing the model's ability to pick up long-range interactions in time as previously discussed at the end of Section \ref{sec:4}.

\subsection{Mortality prediction} \label{subseq:4_mortality}
We consider the  {\sc Physionet2012} challenge dataset \citep{goldberger2000components} for mortality prediction, which is a case of medical TSC as the task is to predict in-hospital mortality of patients after their admission to the ICU. This is a difficult ML task due to missingness in the data, low signal-to-noise ratio (SNR), and imbalanced class distributions with a prevalence ratio of around $14 \%$. We extend the experiments conducted in \citet{horn2020set}, which we also use as very strong baselines. Under the same experimental setting, we train two models: $\text{FCN-LS2T}$ as ours and the FCN as another baseline. For both models, we conduct a random search for all hyperparameters with 20 samples from a pre-specified search space, and the setting with best validation performance is used for model evaluation on the test set over 5 independent model trains, exactly the same way as it was done in \cite{horn2020set}. We preprocess the data using the same method as in \citet[eq.~(9)]{che2018recurrent} and additionally handle static features by tiling them along the time axis and adding them as extra coordinates. We additionally introduce in both models a \texttt{SpatialDropout1D} layer after all CNN and LS2T layers with the same tunable dropout rate to mitigate the low SNR of the dataset.

\begin{wraptable}{r}{0.55\textwidth}
	\vspace{-17.5pt}
	\caption{Comparison of FCN-LS2T and FCN on {\sc Physionet2012} with the results from \citet{horn2020set}.}
	\label{table:mort_pred}
	\vspace{-10pt}
	\begin{center}
	\begin{scriptsize}
	\begin{sc}
		\begin{tabular}{lcccc}
			\toprule
			Model & Accuracy & AUPRC & AUROC \\
			\midrule
			$\text{FCN-LS2T}$ & $\mathbf{84.1 \pm 1.6}$ & $\mathbf{53.9 \pm 0.5}$ & $85.6 \pm 0.5$ \\ 
			FCN & $80.7 \pm 1.7$ & $52.8 \pm 1.3$ & $85.6 \pm 0.2$ \\
			GRU-D & $80.0 \pm 2.9$ & $\mathit{53.7 \pm 0.9}$ & $\mathbf{86.3 \pm 0.3}$ \\
			GRU-Simple & $82.2 \pm 0.2$ & $42.2 \pm 0.6$ & $80.8 \pm 1.1$ \\
			IP-Nets & $79.4 \pm 0.3$ & $51.0 \pm 0.6$ & $\mathit{86.0 \pm 0.2}$ \\
			Phased-LSTM & $76.8 \pm 5.2$ & $38.7 \pm 1.5$ & $79.0 \pm 1.0$ \\
			Transformer & $\mathit{83.7 \pm 3.5}$ & $52.8 \pm 2.2$ & $\mathbf{86.3 \pm 0.8}$ \\
			Latent-ODE & $76.0 \pm 0.1$ & $50.7 \pm 1.7$ & $85.7 \pm 0.6$ \\
			SeFT-Attn. & $75.3 \pm 3.5$ & $52.4 \pm 1.1$ & $85.1 \pm 0.4$ \\
			\bottomrule
		\end{tabular}
	\end{sc}
	\end{scriptsize}
	\end{center}
\vspace{-15pt}
\end{wraptable}

\paragraph{Results.} Table \ref{table:mort_pred} compares the performance of $\text{FCN-LS2T}$ with that of FCN and the results from \cite{horn2020set} on 3 metrics: \begin{enumerate*}[label=(\arabic*)] \item {\sc accuracy}, \item area under the precision-recall curve ({\sc AUPRC}), \item area under the ROC curve ({\sc AUROC}). 
\end{enumerate*} We can observe that \emph{FCN-LS2T takes on average first place according to both {\sc Accuracy} and {\sc AUPRC}, outperforming FCN and all SOTA methods}, e.g.~{\sc Transformer} \citep{vaswani2017attention}, {\sc GRU-D} \cite{che2018recurrent}, {\sc SeFT} \citep{horn2020set}, and also being competitive in terms of {\sc AUROC}. This is very promising, and it suggests that LS2T layers might be particularly well-suited to complex and heterogenous datasets, such as medical time series, since the FCN-LS2T models significantly improved accuracy on {\sc ECG} as well, another medical dataset in the previous experiment.

% \vspace{-10pt}

\subsection{Generating sequential data} \label{subseq:4_gpvae}
Finally, we demonstrate on sequential data imputation for time series and video that LS2Ts do not only provide good representations of sequences in discriminative, but also generative models.
\paragraph{The GP-VAE model.}
In this experiment, we take as base model the recent GP-VAE \citep{fortuin2019gpvae}, that provides state-of-the-art results for probabilistic sequential data imputation. 
The GP-VAE is essentially based on the HI-VAE \citep{nazabal2018handling} for handling missing data in variational autoencoders (VAEs) \citep{kingma2013auto} adapted to the handling of time series data by the use of a Gaussian process (GP) prior \citep{williams2006gaussian} across time in the latent sequence space to capture temporal dynamics.
% The GP-VAE assumes that $\bx \in \Seq{\R^d} $ is generated as 
% \begin{align}
%     p_\theta(\bx_i \given \bz_i) = \cN(\bx_i \given g_\theta(\bz_i), \sigma^2 \mathbf{I}_d),
% \end{align}
% where $\bz \in \Seq{\R^{d^\prime}}$ denotes a latent process and $g_\theta: \R^{d^\prime} \rightarrow \R^d$ is the time-point-wise \textit{decoder network} parametrized by $\theta$.
% % while $\sigma^2 \in \R$ is the observation noise variance.
% The temporal interdependencies are modelled in the latent space by assigning independent GP priors to the coordinate processes of $\bz$, i.e.~denoting $\bz_i = (z_i^j)_{j=1}^d \in \R^{d^\prime}$, it is assumed that $z^j \sim \mathcal{GP}(m(\cdot), k(\cdot, \cdot))$, where $m$ and $k$ are the mean and covariance functions.
% As usual in VAEs, exact Bayesian inference is intractable and free-form variational approximations are inefficient. Hence they apply amortized variational inference \cite{Gershman2014AmortizedII, Rezende14Stochastic} to fit a Gaussian approximation to the posterior time-marginals of $\bz$, the means and covariances of which are represented by the so-called \textit{encoder network}.
% This allows one to efficiently fit a model to several examples jointly at training time, and make inference about unseen examples at testing time without any optimization overhead.
% Missing data is handled by imputing the missing values with $0$ before feeding them into the encoder network, while the ELBO loss is only computed across the observed features during training similarly to the HI-VAE solution \cite{nazabal2018handling}. 
Since the GP-VAE is a highly advanced model, its in-depth description is deferred to Appendix~\ref{app:imputation}.% and for the experiments Appendix \ref{app:imputation}.
% As is usual, the mean function is a-priori set to $m \equiv 0$, while for the covariance the authors use the Cauchy kernel 
% \begin{align}
%     k(\tau, \tau^\prime) = \tilde\sigma^2 \left(1 + \frac{(\tau - \tau^\prime)^2}{l^2}\right)^{-1}
% \end{align} 
We extend the experiments conducted in \citet{fortuin2019gpvae}, and we make one simple change to the GP-VAE architecture without changing any other hyperparameters or aspects: we introduce a single bidirectional LS2T layer (B-LS2T) into the encoder network that is used in the amortized representation of the means and covariances of the variational posterior.
The B-LS2T layer is preceded by a time-embedding and differencing block, and succeeded by channel flattening and layer normalization as depicted in Figure \ref{fig:gpvae_ls2t_encoder}. The idea behind this experiment is to see if we can improve the performance of a highly complicated model that is composed of many interacting submodels, by the naive introduction of LS2T layers.

			 %   Mean imputation & - & $0.168 \pm 0.000$ & $0.938 \pm 0.000$ & $0.013 \pm 0.000$ & $0.703 \pm 0.000$ \\
			 %   Forward imputation & - & $0.177 \pm 0.000$ & $0.935 \pm 0.000$ & $0.028 \pm 0.000$ & $0.710 \pm 0.000$ \\
    %             VAE \cite{kingma2013auto} & $0.599 \pm 0.002$ & $0.232 \pm 0.000$ & $0.922 \pm 0.000$ & $0.028 \pm 0.000$ & $0.677 \pm 0.002$ \\
    %             HI-VAE \cite{nazabal2018handling} & $0.372 \pm 0.008$ & $0.134 \pm 0.003$ & $\mathbf{0.962 \pm 0.00}1$ & $0.007 \pm 0.000$ & $0.686 \pm 0.010$ \\
    %             GP-VAE \cite{fortuin2019gpvae} & $0.350 \pm 0.007$  & $0.114 \pm 0.002$ & $\mathbf{0.960 \pm 0.002}$ & $\mathbf{0.002 \pm 0.000}$ & $0.730 \pm 0.006$ \\  
    %             GP-VAE (B-LS2T) & $\mathbf{0.251 \pm 0.008}$ & $\mathbf{0.092 \pm 0.003}$ & $\mathbf{0.962 \pm 0.001}$ & $\mathbf{0.002 \pm 0.000}$ & $\mathbf{0.743 \pm 0.007}$\\
    %             BRITS \cite{Cao2018brits} & - & - & - & - & $\mathbf{0.742 \pm 0.008}$\\

% \begin{minipage}[t]{\textwidth}
% \vspace{-20pt}
% \hspace{-20pt}

% \end{minipage}

\paragraph{Results.}
To make the comparison, we ceteris paribus re-ran all experiments the authors originally included in their paper \citep{fortuin2019gpvae}, which are imputation of Healing MNIST, Sprites, and Physionet 2012.
The results are in Table \ref{table:gp_vae_comparison}, which report the same metrics as used in \citet{fortuin2019gpvae}, i.e.~negative log-likelihood (NLL, lower is better), mean squared error (MSE, lower is better) on test sets, and downstream classification performance of a linear classifier (AUROC, higher is better). For all other models beside our GP-VAE (B-LS2T), the results were borrowed from \citet{fortuin2019gpvae}. We observe that simply adding the B-LS2T layer improved the result in almost all cases, except for Sprites, where the GP-VAE already achieved a very low MSE score. Additionally, when comparing GP-VAE to BRITS on Physionet, the authors argue that although the BRITS achieves a higher AUROC score, the GP-VAE should not be disregarded as it fits a generative model to the data that enjoys the usual Bayesian benefits of predicting distributions instead of point predictions. The results display that by simply adding our layer into the architecture, we managed to elevate the performance of GP-VAE to the same level while retaining these same benefits. We believe the reason for the improvement is a tighter amortization gap in the variational approximation \citep{Cremer2018inference} achieved by increasing the expressiveness of the encoder by the LS2T allowing it to pick up on long-range interactions in time. We provide further discussion in Appendix \ref{app:imputation}.

\begin{table}[t]
	\caption{Performance comparison of GP-VAE (B-LS2T) with the baseline methods}
	\label{table:gp_vae_comparison}
% 	\vskip 0.15in
	\vspace{-10pt}
	\begin{center}
	    \begin{small}
		\begin{sc}
		\makebox[\textwidth][c]{
	    \resizebox{\textwidth}{!}{
		\begin{tabular}{lccccc}
			\toprule
			\multirow{2}{*}{Method} & \multicolumn{3}{c}{HMNIST} & \multicolumn{1}{c}{Sprites} & \multicolumn{1}{c}{Physionet} \\
			\cmidrule{2-6}
			& NLL & MSE & AUROC & MSE & AUROC \\
			\midrule
			    Mean imputation & - & $0.168 \pm 0.000$ & $0.938 \pm 0.000$ & $0.013 \pm 0.000$ & $0.703 \pm 0.000$ \\
			    Forward imputation & - & $0.177 \pm 0.000$ & $0.935 \pm 0.000$ & $0.028 \pm 0.000$ & $0.710 \pm 0.000$ \\
                VAE & $0.599 \pm 0.002$ & $0.232 \pm 0.000$ & $0.922 \pm 0.000$ & $0.028 \pm 0.000$ & $0.677 \pm 0.002$ \\
                HI-VAE & $0.372 \pm 0.008$ & $0.134 \pm 0.003$ & $\mathbf{0.962 \pm 0.00}1$ & $0.007 \pm 0.000$ & $0.686 \pm 0.010$ \\
                GP-VAE & $0.350 \pm 0.007$  & $0.114 \pm 0.002$ & $\mathbf{0.960 \pm 0.002}$ & $\mathbf{0.002 \pm 0.000}$ & $0.730 \pm 0.006$ \\  
                GP-VAE (B-LS2T) & $\mathbf{0.251 \pm 0.008}$ & $\mathbf{0.092 \pm 0.003}$ & $\mathbf{0.962 \pm 0.001}$ & $\mathbf{0.002 \pm 0.000}$ & $\mathbf{0.743 \pm 0.007}$\\
                BRITS & - & - & - & - & $\mathbf{0.742 \pm 0.008}$\\
			\bottomrule
		\end{tabular}}}
	\end{sc}
	\end{small}
	\end{center}
	\vspace{-20pt}
\end{table}
	
\section{Related work and Summary}\label{sec: summary}	
%We used a classical non-commutative structure to construct a feature map for sequential data by composing algebraic non-linearities and state space non-linearities.
%Compositions of low rank functionals allowed us to build scalable ``tensorized'' versions of well-known neural network models.
%Finally, we would like to emphasize two points: firstly, although we provide a detailed theoretical analysis in the Appendix that applies advanced concepts of algebra, none of this is needed for practitioners who just want to apply the LS2T layers. 
%Secondly, our experiments show that simple applications of such layers in popular NN architectures yields improvements on competitive baselines.
%However, many other models can be modified by adding LS2T layers.
%Please try it out in your own favourite model!  
\paragraph{Related Work.}
The literature on tensor models in ML is vast. 
Related to our approach we mention pars-pro-toto Tensor Networks \citep{cichocki2016tensor}, that use classical LR decompositions, such as CP \citep{carroll1970analysis}, Tucker \citep{tucker1966some}, tensor trains \citep{oseledets2011tensor} and tensor rings \citep{zhao2019learning};
further, CNNs have been combined with LR tensor techniques \citep{cohen2016expressive,kossaifi2017tensor} and extended to RNNs \citep{khrulkov2019generalized}; Tensor Fusion Networks \citep{zadeh2017tensor} and its LR variants \citep{liu2018efficient,liang2019learning,hou2019deep}; tensor-based gait recognition \citep{tao2007general}.
Our main contribution to this literature is the use of the free algebra $\T{V}$ with its convolution product $\cdot$, instead of $V^{\otimes m}$ with the outer product $\otimes$ that is used in the above papers. 
While counter-intuitive to work in a larger space $\T{V}$, the additional algebra structure of $(\T{V},\cdot)$ is the main reason for the nice properties of $\Phi$ (\emph{universality, making sequences of arbitrary length comparable, convergence in the continuous time limit}; see Appendix~\ref{app:algebra}) which we believe are in turn the main reason for the \emph{strong benchmark performance}. 
Stacked LR sequence transforms allow to exploit this rich algebraic structure with little computational overhead.
Another related literature are path signatures in ML \citep{lyons2014rough, primer2016, graham2013sparse, Deepsig, Toth19}.
These arise as special case of Seq2Tens (Appendix~\ref{app:algebra}) and our main contribution to this literature is that Seq2Tens resolves a well-known computational bottleneck in this literature since it \emph{never needs to compute and store a signature}, instead it \emph{directly and efficiently learns the functional of the signature}.   

\paragraph{Summary.}
We used a classical non-commutative structure to construct a feature map for sequences of arbitrary length. %by translating temporal dependencies to algebraic relations.
By stacking sequence transforms we turned this into scalable and modular NN layers for sequence data.
The main novelty is the use of the free algebra $\T{V}$ constructed from the static feature space $V$.  
While free algebras are classical in mathematics, their use in ML seems novel and underexplored.
We would like to re-emphasize that $(\T{V}, \cdot)$ is not a mysterious abstract space: if you know the outer tensor product $\otimes$ then you can easily switch to the tensor convolution product $\cdot$ by taking sums of outer tensor products, as defined in~\eqref{eq:tensorprod}.  
% \begin{enumerate*}[label=(\roman*)] 
% 	\item Although our analysis in the appendices applies advanced ideas from algebra, this theoretical background is not needed for practitioners who want to apply the Seq2Tens approach in their models.
% 	\item Our experiments indicate that simple applications of such layers in popular architectures yields improvements on competitive baselines for both discriminative and generative models.
% \end{enumerate*}
As our experiments show, the benefits of this algebraic structure are not just theoretical but can significantly elevate performance of already strong-performing models.

\newpage

\bibliography{references}
\bibliographystyle{iclr2020_conference}

\newpage
\appendix
\section*{How to use this appendix}
% Section~\ref{app:tensor} recalls tensor product and functionals of tensors.
% Section~\ref{app:phi universal} proves that $\Phi$ is universal and Section~\ref{sec:3} that compositions of low-rank functionals of $\Phi$ are very efficient approximators.  
% Section~\ref{sec:4} includes more details on experiments and Section~\ref{app:algs} contains algorithms and implementation details. 

\textbf{For practitioners}, we recommend a look at Section \ref{app:tensor} for a refresher on tensor notation and an introduction to $\T{V}$;
further, the introduction of Section \ref{app:algebra} contains a brief summary of the main theoretical properties of Seq2Tens that make it an attractive feature map for sequence data.
Sections~\ref{app:algorithms_detail} and \ref{app:details} contain details on algorithms and experiments.
%In particular, we draw attention to~\ref{app:full_map} we summarize what we believe to be our main contributions in non-technical terms.
%, namely how 
% \begin{center}
%   \emph{the expensive computation of $\Phi$ can be avoided}.
% \end{center}
%This is all you need to apply Seq2Tens! 

\textbf{For theoreticians}, we recommend Section~\ref{app:algebra} for a proof that $\Phi$ is universal (Theorem~\ref{prop:univ}), how the Seq2Tens map behaves in the high-frequency limit as one goes from discrete to continuous time (Proposition \ref{prop:convergence to signature}), and to Section~\ref{app:iterated} for a quantitative statement of low-rank functionals can be turned into high-rank functionals with sequence-to-sequence transformations. 
We re-emphasize that these more algebra-heavy sections are not needed for practitioners.% but we hope that they convince some readers that tools from non-commutative algebra can lead to insight and new models in ML that capture structured objects such as sequences. 
%We reiterate that the theory of Section~\ref{app:algebra} and \ref{app:iterated} can be completely skipped but we  

\section{Tensors and the Free Algebra}\label{app:tensor}
  % Key to our approach is to construct a map
  % \begin{align}
  %  \Phi:\Seq{V} \rightarrow \T{V} 
  % \end{align}
  % that injects sequences in a vector space $V$ into the linear space $\T{V}$. 
  % The space $\T{V}$ consists of series of tensors of increasing degree and thus has a natural grading in terms of this degree; further it carries a non-commutative multiplication, the so-called convolution product. 
  % All these objects are classical and in a precise mathematical sense, the most general constructions,
  % but they might be less familiar to researchers in machine learning.
  
This section recalls some basics on the tensor product $\otimes$ and the convolution product that turns the linear space $\T{V}$ into an algebra -- the so-called \emph{free algebra} or \emph{free algebra over $V$}. 
  We refer to \cite[Chapter 16]{Lang} for more on tensors, to \citet{Reutenauer93} for free algebras.
  Put briefly, for any linear space $V$ there exists a linear space $\T{V}$ that contains $V$ but that also carries a non-commutative product.
%   The space $\T{V}$ can be realized as series of tensors of increasing degree, hence the name tensor algebra.
	\paragraph{Tensor products on $\R^d$.}
	If $ x = (x_1, \ldots, x_d) \in \R^d $ and $ y = (y_1, \ldots, y_e) \in \R^e $ are two vectors, then their \emph{tensor product} $ x \otimes y $ is defined as the $ (d\times e) $-matrix, or degree $ 2 $ tensor, with entries $ (x \otimes y)_{i,j} = x_iy_j $. This is also commonly called the \emph{outer product} of the two vectors. The space $ \R^d \otimes \R^e $ is defined as the linear span of all degree $ 2 $ tensors $ x \otimes y $ for $ x \in \R^d, y \in \R^e $.
	If $ z \in \R^f $ is another vector, then one may form a degree $ 3 $ tensor $ x \otimes y \otimes z $ with shape $ (d\times e \times f) $ defined to have entries $ (x \otimes y \otimes z)_{i,j,k} = x_iy_jz_k $. The space $ \R^d \otimes \R^e \otimes \R^f $ is analogously defined as the linear span of all degree $ 3 $ tensors $ x \otimes y \otimes z $ for $ x \in \R^d, y \in \R^e, z \in \R^f $.	
	
	The tensor product of two general vector spaces $ V $ and $ W $ can be defined even if they are infinite dimensional, see \cite[Chapter 16]{Lang}, but we invite readers unfamiliar with general tensor spaces to think of $ V $ as $\R^d$ below.
	
 \paragraph{The free algebra $\T{V}$.}
 Ultimately we are not only interested in tensors of some fixed degree $m$ -- that is an element of $V^{\otimes m}$ -- but sequences of tensors of increasing degree. Given some linear space $ V $, the linear space $ \T{V} $ is defined as set of all tensors of any degree over $ V $.
 Formally
 \begin{align}
   \T{V} := \prod_{m\geq 0} V^{\otimes m} =   \{\bt=(\bt_m)_{m\ge 0}\,\vert\, \bt \in V^{\otimes m}\}%= (1, \R^d, \R^d\otimes \R^d, \R^d\otimes \R^d\otimes \R^d, \ldots ).
%(V^{\otimes 0}, V, V^{\otimes 2}, V^{\otimes 3}, \ldots )
 \end{align}
 where we use the notation $V^{\otimes 1}=V$, $V^{\otimes 2} = V \otimes V,\, V^{\otimes 3} = V \otimes V \otimes V $ and so on; by convention we let $ V^{\otimes 0} = \R $.
 We normally write elements of $ \T{V} $ as $ \bt = (\bt_0, \bt_1, \bt_2, \bt_3, \ldots ) $ such that $ \bt_m \in V^{\otimes m} $, that is, $\bt_0$ is a scalar, $ \bt_1 $ is a vector, $ \bt_2 $ is a matrix, $ \bt_3 $ is a $ 3 $-tensor and so on.	
 Note that $\T{V}$ is again a linear space if we define addition and scalar multiplication as
 \begin{align}
 \bs + \bt = (\bs_m+\bt_m)_{m \ge 0} \in \T{V}\text{ and } c \cdot \bt = (c\bt_m)_{m \ge 0} \in \T{V}
 \end{align}
 for $\bs,\bt\in \T{V}$ and $c \in \R$.
	% If $ V $ has a basis\todo{remove?} $ e_1, \ldots, e_d $, then $ V^{\otimes m} $ is spanned by $ \{ e_{i_1}\otimes \cdots \otimes e_{i_m} \, : 1 \leq i_1, \ldots, i_m \leq d \} $. Using this, elements $ (\bt_m)_{m\geq0} $ of $ \T{V} $ can be written as infinite linear combinations of
  % \begin{align}
  % \{ e_{i_1}\otimes \cdots \otimes e_{i_m} \, : 1 \leq i_1, \ldots, i_m \leq d, m\geq 0 \}. 
  % \end{align}
  \begin{example}
    Let $V=\R^d$.
    For $\bv=(\bv_i)_{i=1,\ldots,d} \in \R^d$ consider $\bt=(\bv^{\otimes m})_{m \ge 0} \in \T{\R^d}$ where we denote for brevity
    \[\bt_m:=\bv^{\otimes m}:= \underbrace{\bv \otimes \cdots \otimes \bv}_{m \text{ many tensor products }\otimes} \in (\R^d)^{\otimes m} \text{ and by convention we set }\bv^{\otimes 0 }:=1 \in  (\R^d)^{\otimes 0}.\]
    That is, $\bt_1=\bv^{\otimes 1}=\bv$ is a $d$-dimensional vector, with the $i$ coordinate equal to $\bv_i$; $\bt_2=\bv^{\otimes 2}$ is $d\times d$-matrix with the $(i,j)$-coordinate equal to $\bv_i\bv_j$; $\bt_3=\bv^{\otimes 3}$ is degree $3$-tensor with the $(i,j,k)$-coordinate equal to $\bv_i\bv_j\bv_k$.
    In this special case, the element $\bt\in \T{\R^d}$ consists of entries $\bt_m=\bv^{\otimes m} \in (\R^{d})^{\otimes m}$ that are symmetric tensors, that is the $(i_1,\ldots,i_m)$-th coordinate is the same as the $(i_{\sigma(1)},\ldots,i_{\sigma(d)})$ coordinate if $\sigma$ is a permutation of $\{1,\ldots,d\}$. 
    However, we emphasize that in general an element of $\T{\R^d}$ does not need to be made up of symmetric tensors.
  \end{example}

  \paragraph{A product on $\T{V}$.}
  Key to our approach is that $\T{V}$ is not only a linear space, but what distinguishes it as a feature space for sequences is that it carries a non-commutative product.
  In other words, $\T{V}$ is not just a vector space but a (non-commutative) algebra (an algebra is a vector space where one can multiply elements).
  This is the so-called \emph{tensor convolution product} and defined as follows
	\begin{align}\label{eq:ncp}
    \bs \cdot \bt := \big( \sum_{i=0}^m \bs_i\otimes \bt_{m-i} \big)_{m\geq 0} = \big( 1, \bs_1 + \bt_1, \bs_2 + \bs_1\otimes \bt_1 + \bt_2,\ldots \big).
	\end{align}
%  Another way to motivate this product is to simply extend the tensor product $\otimes m$ by bilinearity.
 % \paragraph{The universal property of $\T{V}$.}
  In a precise mathematical sense, $ \T{V} $ is the most general algebra
  containing $ V $, namely \emph{$\T{V}$ is the ``free algebra'' that contains $V$}; see \cite[Chapter 16]{Lang} for the precise definition of free objects.

\section{A universal feature map for sequences of arbitrary length} \label{app:algebra}
Recall from Section \ref{sec:2}, that given a map defined on a set $\cX$
\[ \phi: \cX \to V \]
we lift $ \phi $ to a map $ \varphi: \cX \to \T{V} $ and define the Seq2Tens feature map for sequences in $\cX$ of arbitrary length as 
\[\Phi:\Seq{\cX} \rightarrow \T{V},\quad \bx \to \prod_{i=1}^T  \varphi(\bx_i).\] 
The remainder of Section~\ref{app:algebra} makes the following statements mathematically rigorous:
\begin{enumerate}[label=(\roman*)]
  \item 
    $\Phi$ is a universal feature map whenever $\phi$ is a universal (Section~\ref{app:phi universal} and \ref{app:proof universality}), 
  \item
    $\Phi$ makes sequences of different length comparable analogous to how $m$-mers make strings of different length comparable (Section~\ref{sec:different length}),  
  \item
    $\Phi$ converges to a well-defined object when we go from discrete to continuous time (sequences converge to paths) (Section~\ref{sec:convergence to sig}).
\end{enumerate}
\subsection{The universality of $\Phi$.}\label{app:phi universal}
% We say that a map $ f : \cX \to W $ from some topological space $ \cX $ to a vector space $ W $ is \emph{universal} if any continuous function on $ \cX $ is approximately a linear function on the image of $ f $. More formally
%Universal feature maps are guarantee that data in $ \cX $ can be analysed by linear maps on $ W $.

\begin{definition}
  Let $\cX$ be a topological space (the ``data space'') and $W$ a linear space (``the feature space'').
	We say that a function $ f: \cX \to W $ is universal (to $C_b(\cX)$) if the the set of functions
		\begin{align}
		\{ x \mapsto \langle \ell, f(x) \rangle \,:\, \ell \in W' \} \subseteq C_b(\cX)
		\end{align}
		is dense in $ C_b(\cX) $.
	\end{definition}
	
\begin{example} Classic examples of this in ML are 
		\begin{itemize}
			\item For $\cX\subset\R^d$ bounded and $W=\T{\R^d}$, the polynomial map $ p : \R^d \to \T{\R^d}, \bx \mapsto (1, \bx, \bx^{\otimes 2}, \bx^{\otimes 3}, \bx^{\otimes 4}, \ldots) $ is universal~\citep{rudin}.
			\item The $ 1 $-layer neural net map $ \bx \mapsto \prod_\theta N_\theta(\bx) $ where $ \theta $ runs over all configurations of parameters is universal under some very mild conditions~\citep{hornik1991approximation}.
		\end{itemize}
	\end{example}

  We now prove the main result of this section
  \begin{thm} \label{prop:univ}
  Let $ \varphi : \cX \to \T{V}, \bx \mapsto (\varphi_m(\bx))_{m \ge 0}, \varphi_m(\bx) \in V^{\otimes m}$ be such that:
  \begin{enumerate}
  	\item For any $ n\geq 1 $ the support of $(\varphi_0,\varphi_{1}, \ldots, \varphi_{m_1})^{\otimes n} $ and $ \varphi_{m_2} $ are disjoint if $ 1\leq m_1 < m_2 $. 
  	\item $ \varphi_0 =1 $ and $\varphi_1:\cX \rightarrow V$ is a bounded universal map with at least one constant term.
  \end{enumerate}
  Then
		\begin{align}\label{eq: phi}
		\Phi:\mathrm{Seq}(\cX) \to \T{V}, \quad (\bx_1,\ldots,\bx_L) \mapsto \prod_{i=1}^L \varphi(\bx_i)
		\end{align}
		is universal.
	\end{thm}
	\begin{remark}\label{rem:seq2tens examples}\,
		\begin{enumerate}[label=(\roman*)]
%      \item By taking $\varphi(\bx)=(1,\bx,0,0,\ldots)$ one generalizes $k$-mers, Idea 2 and Example~\ref{ex:represent}, 
    \item \label{itm:gen}
      Theorem\ref{prop:univ} implies that $\bx \mapsto \prod_{i=1}^L(1,\phi(\bx),0,\ldots)$ is universal whenever $\phi:\cX \rightarrow V$ is universal.
      This is the lift we use throughout the main text, see \eqref{eq:oneplus}.

    \item\label{itm:rp}
        By taking $\varphi : \R^d \to \T{\R^d}, \bx \mapsto (1,\bx,\frac{\bx^{\otimes 2}}{2!}, \frac{\bx^{\otimes{3}}}{3!},\ldots) $ one recovers Chen's signature \citep{Chen54, Chen57, Chen58} as used in rough paths.
      \item\label{itm:drp}
        By taking $\varphi: \R^d \to \T{V}$, $\varphi_1(\bx)$ the polynomial map and $\varphi_m(\bx)=0$ for $m\ge 2$ one recovers the iterated sums of \citet{Tapia19} and \citet{kiraly2016kernels}.   
      \item
        By taking each $ \varphi_m $ to be a trainable Neural Network one gets a trainable universal map $\Phi$ for sequences that includes all of the above,
		\end{enumerate}
	\end{remark}
%   We emphasize that the feature map \eqref{eq: phi} is very general since different choices of $\varphi$ yield very different algebraic structures; e.g.~for signature as in \ref{itm:rp} $\Phi$ takes values in the character group of the shuffle algebra \cite{Reutenauer93}; but even for \ref{itm:drp} the resulting structure is very different and was only recently studied, see \cite{Tapia19}. \ref{itm:gen} allows one to use a NN for $\varphi_1$ to learn the best structure implicitly from the data.
% Despite this generality, Theorem~\ref{prop:univ} guarantees that the resulting map $\Phi$ will be a universal feature map for sequences. 
  
\subsubsection{The algebra of linear functionals on $ \Phi $.}\label{app:quasi} 
The proof of Theorem~\ref{prop:univ} uses that if $ \varphi $ is universal, then the space of linear functionals on $ \Phi(\bx) $ forms a commutative algebra, that is for two linear functionals $\ell_1,\ell_2$ there exists another linear functional $\ell$ such that
\begin{align}
  \langle \ell_1,\Phi(\bx) \rangle \langle \ell_2,\Phi(\bx) \rangle =  \langle \ell,\Phi(\bx) \rangle. 
\end{align}
This new functional $\ell$ is constructed in explicit way from $\ell_1$ and $\ell_2$, with a so-called quasi-shuffle product.
In the remainder of this section \ref{app:phi universal}, we prepare and give the proof of Theorem~\ref{prop:univ}: subsection~\ref{app:quasi} introduces the quasi-shuffle product, and subsection~\ref{app:quasi} uses this to prove Theorem~\ref{prop:univ}.     

We spell out the proof for the case $ \varphi=(1,\phi, 0,0,\dots) \in \T{V}$ since this is the form we use in the main text, Proposition~\ref{thm:univ}, and the other cases follow similarly.  
In fact, without loss of generality we can take $ \phi = \mathrm{id} $ since this does not change the algebraic structure in any way.
That is, we take 
	\begin{align} \label{eq:phi}
	\Phi:\Seq{V}\to \T{V}, \quad \Phi(\bx_1, \ldots, \bx_L) := \prod_{i=1}^L \varphi(\bx_i)
	\end{align}
  with $\varphi(\bx)=(1,\bx,0,0,\ldots)$.
%	Since $ \Phi $ takes values in the linear space $ \T{V} $, we may pair it with linear functional as defined in \eqref{eq: coordinates functional}.
  By using the definition of the product in $\T{V}$ and expanding \eqref{eq:phi} we get  
	\begin{align}
	\Phi(\bx_1, \ldots, \bx_L) = (1, \sum_{i=1}^L \underbrace{\bx_i}_{V}, \sum_{1 \leq i_1 < i_2 \leq L} \underbrace{\bx_{i_1}\otimes \bx_{i_2}}_{V^{\otimes 2}}, \sum_{1 \leq i_1 < i_2 < i_3 \leq L} \underbrace{\bx_{i_1}\otimes \bx_{i_2}\otimes \bx_{i_3}}_{V^{\otimes 3}}, \cdots) 
	\end{align}
	In general, writing $ \Phi_m(\bx) $ for the projection of $ \Phi(\bx) $ onto $ V^{\otimes m} $, we have
	\begin{align} \label{eq:phim}
	\Phi_m(\bx) = \sum_{1 \le i_1 < \cdots < i_m \le L} \bx_{i_1} \otimes \cdots \otimes \bx_{i_m}.
	\end{align}
	So if $ \ell =(0,0,\ldots, \bv_1 \otimes \cdots \otimes \bv_m,0,\ldots) $ with $\bv_1,\ldots,\bv_m \in V$, then 
	\begin{align}
	&\langle \ell, \Phi_m(\bx) \rangle 
	= \langle \bv_1 \otimes \cdots \otimes \bv_m, \sum_{1 \le i_1 < \cdots < i_m \le L} \bx_{i_1} \otimes \cdots \otimes \bx_{i_m} \rangle \\
	&= \sum_{1 \le i_1 < \cdots < i_m \le L} \langle \bv_1 \otimes \cdots \otimes \bv_m, \bx_{i_1} \otimes \cdots \otimes \bx_{i_m} \rangle
	= \sum_{1 \le i_1 < \cdots < i_m \le L} \langle \bv_1 , \bx_{i_1} \rangle \cdots \langle \bv_m, \bx_{i_m} \rangle.
	\end{align}
	Hence $ \langle \ell, \Phi_m(\bx) \rangle $ can be computed efficiently without computing $ \Phi(\bx) $.
  Proposition \ref{prop:rank one} follows by linearity since by definition $\langle \ell, \Phi(\bx) \rangle = \sum_{m  \ge 0} \langle \ell_m, \Phi(\bx) \rangle$ and for each of the terms we can use the above formula when $\ell=(\ell_0,\ell_1,\ell_2,\ldots,\ell_M,0,\ldots)$ is of rank-$1$ and of degree $M$ (Definition~\ref{def:rank}).
%	\begin{example}
%		Assume that $ V = \R^d $, then
%		\begin{align}
%		\langle \ell, \Phi(\bx) \rangle = \sum_{i=1}^d \sum_{k=1}^L \ell_i\bx_k^i, \quad 
%		\langle \ell, \Phi(\bx) \rangle = \sum_{i,j=1}^d \sum_{1 \leq k_1 < k_2 \leq L} \ell_{ij}\bx_{k_1}^i\bx_{k_2}^{j}
%		\end{align}
%	\end{example}
	
	\paragraph{Non-linear functionals acting on $ \Phi $.} We now investigate what happens when one applies non-linear functions to $ \Phi(\bx) $. To do this, we first note that since $ \T{V} $ is a vector space, we may form the free algebra over $ \T{V} $, denoted by $ \T{\T{V}} $, or $ \Tra{2}{V} $. It may be decomposed as 
	\begin{align}
	\Tra{2}{V} = \prod_{n_1, \ldots, n_k \geq 0} V^{\otimes n_1} \big\vert \cdots \big\vert V^{\otimes n_k}
	\end{align}
	where we use the notation $ \otimes $ for the tensor product on $ V $ and the bar $ \vert $ for the tensor product on $ \T{V} $. See~\citet{Ebrahimi15} for more on $ \Tra{2}{V} $ and the bar notation. 
	
	\begin{definition}
		If $ x \in V $ is a vector, we denote by $ x^\star $ its extension 
		\begin{align}
		x^\star := (x^{\otimes m})_{m\geq 0} = (1, x, x^{\otimes 2}, x^{\otimes 3}, \ldots )
		\end{align}
		and if $ \bx = (\bx_1, \ldots, \bx_L) \in \Seq{V} $ is a sequence, then 
		\begin{align}
		\bx^\star := (\bx_1^\star, \ldots, \bx_L^\star) \in \Seq{\T{V}}
		\end{align}
	\end{definition}
%	\begin{definition}
%		If $ V $ has a basis $ e_1, \ldots, e_d $, then we denote by $ V^\star $ the extended $ d+d^2 $-dimensional vector space spanned by
%		\begin{align}
%		\{ e_i : 1\leq i \leq d \} \cup \{ [e_ie_j] : 1\leq i, j \leq d \}
%		\end{align}
%		where $ [e_ie_j] $ is a new basis vector independent of all other basis vectors. If $ x = (x_1, \ldots, x_d) $ is a vector, then we denote by $ x^\star $ the extended vector 
%		\begin{align}
%		x^\star = (x_1, \ldots, x_d, [x_1x_1], \ldots, [x_1x_d], \ldots, [x_dx_1], \ldots, [x_dx_d])
%		\end{align}
%	\end{definition}	
%	\begin{example}
%		$ (\R^2)^\star $ is spanned by the three vectors $ e_1, e_2, [e_1e_2] $.
%	\end{example}

	Since $ \bx^\star $ is a sequence in $ \T{V} $, we may compute $ \Phi(\bx^\star) $ which takes values in $ \T{\T{V}} = \Tra{2}{V} $.

	The reason for the above definition is that when products of linear functions in $ \T{V} $ act on $ \Phi(\bx) $, they may be described as linear functions in $ \Tra{2}{V} $ acting on $ \Phi(\bx^\star) $. That is, $ \T{V} $ is not big enough to capture all non-linear functions acting on $ \Phi(\bx) $, but $ \Tra{2}{V} $ is.
	
	%Note that since $ \bx^\star $ takes values in $ \T{V} $, rank $ 1 $ linear functionals on $ \Phi(\bx^\star) $ can be written as $ e_{i_1} \vert \cdots \vert e_{i_n} $. 
	
	\begin{definition}
		Assume that $ V $ has basis $ e_1, \ldots, e_d $. The \emph{quasi-shuffle} product 
		\begin{align}
		\star : \Tra{2}{V} \times \Tra{2}{V} \to \Tra{2}{V}
		\end{align}
		is defined inductively on rank $ 1 $ elements $ \ell_1 = e_{i_1} \vert \cdots \vert e_{i_m}, \ell_2 = e_{j_1} \vert \cdots \vert e_{j_n} $ by
		\begin{align}
		(\ell_1 \vert e_i)\star(\ell_2 \vert e_j) = (\ell_1 \vert e_i\star \ell_2)\vert e_j + (\ell_1\star \ell_2 \vert e_j)\vert e_i + (\ell_1\star \ell_2)\vert(e_i\otimes e_j).
		\end{align}
		By linearity $ \star $ extends to a product on all of $ \T{V} $.
	\end{definition}

	% \begin{example} \phantom{.}\\[-15pt]
	% 	\begin{align}
	% 	e_i \star e_j &= e_i\vert e_j + e_j\vert e_i + e_i\otimes e_j \\[5pt]
	% 	e_i \star (e_j\vert e_k) &= e_i\vert e_j\vert e_k + e_j\vert e_i\vert e_k + e_j\vert e_k\vert e_i \\
	% 	&+ (e_i\otimes e_j)\vert e_k + e_j\vert(e_i\otimes e_k) \\[5pt]
	% 	(e_i\vert e_j) \star (e_k\vert e_l) &= e_i\vert e_j\vert e_k\vert e_l + e_i\vert e_k\vert e_j\vert e_l + e_i\vert e_k\vert e_l\vert e_j \\ 
	% 	&+ e_k\vert e_i\vert e_j\vert e_l + e_k\vert e_i\vert e_l\vert e_j + e_k\vert e_l\vert e_i\vert e_j \\ 
	% 	&+ (e_i\otimes e_k)\vert e_j\vert e_l +e_i\vert (e_j\otimes e_k)\vert e_l+(e_i\otimes e_k)\vert e_l\vert e_j \\
	% 	&+e_k\vert (e_i\otimes e_l)\vert e_j +(e_i\otimes e_k)\vert (e_j\otimes e_l)
	% 	\end{align}
	% \end{example}
	
	\begin{lemma} \label{lem:shuff}
		The map $ \Phi $ satisfies the following
		\begin{align}
		\langle \ell_1, \Phi(\bx) \rangle \langle \ell_2, \Phi(\bx) \rangle = \langle \ell_1\star \ell_2, \Phi(\bx^\star) \rangle.
		\end{align}
	\end{lemma}
	\begin{proof}
		By writing out \eqref{eq:phim} in coordinates we get
		\begin{align}
		\langle e_{i_1} \vert \cdots \vert e_{i_m}, \Phi(\bx) \rangle = \sum_{1 \leq k_1 < \cdots < k_m \leq L} \langle e_{i_1}, \bx_{k_1} \rangle \cdots \langle e_{i_m}, \bx_{k_m} \rangle.
		\end{align}
	which shows that $ \Phi $ satisfies a recurrence equation.
		% \begin{align} \label{eq:rec}
		% \langle \ell \vert e_i, \Phi(\bx)_L \rangle = \sum_{k=1}^{L-1} \langle \ell,\Phi(\bx)_k\rangle \langle e_i, \bx_{k+1} \rangle
		% \end{align}
		% where we use the notation $ \Phi(\bx)_k = \Phi(\bx_1, \ldots, \bx_k) $.
    The proof follows by induction.
	\end{proof}

	% \begin{example}
	% 	For a simple example of this, note that
	% 	\begin{align}
	% 	\langle e_i, \Phi(\bx) \rangle \langle e_j, \Phi(\bx) \rangle = \big( \sum_{k=1}^L \langle e_i, \bx_k \rangle \big)\big( \sum_{k=1}^L \langle e_j, \bx_k \rangle \big) = \sum_{k_1, k_2=1}^L \langle e_i, \bx_{k_1} \rangle\langle e_j, \bx_{k_2} \rangle
	% 	\end{align}
	% 	One may also write this with the quasi shuffle product since
	% 	\begin{align}
	% 	&\langle e_i, \Phi(\bx) \rangle \langle e_j, \Phi(\bx) \rangle = \langle e_i\vert e_j+e_j\vert e_i+e_i\otimes e_j, \Phi(\bx) \rangle \\
	% 	&= \sum_{1 \leq k_1 < k_2 \leq L} \langle e_i, \bx_{k_1}\rangle\langle e_j, \bx_{k_2}\rangle + \sum_{1 \leq k_1 < k_2 \leq L} \langle e_j, \bx_{k_1}\rangle\langle e_i, \bx_{k_2}\rangle + \sum_{k=1}^L \langle e_i, \bx_{k}\rangle\langle e_j, \bx_{k}\rangle \\
	% 	&= \sum_{k_1, k_2=1}^L \langle e_i, \bx_{k_1} \rangle\langle e_j, \bx_{k_2} \rangle
	% 	%&\langle ij, \Phi_1 \rangle \langle k, \Phi_1 \rangle = \langle ij \sqcup k+[ik]j+i[kj], \Phi_1 \rangle
	% 	\end{align}
	% \end{example}
	
	% We refer to \cite{Tapia19} for more on quasi-shuffle algebras and their use for time warping invariant features of time-series.
	
	The space $ \Tra{2}{V} $ might seem very large and difficult to work with at first. The power of this representation comes from the fact that one may leverage this in proving strong statements about the original map $ \Phi : \Seq{V} \to \T{V} $, and we will use this in the next subsection.
	 
\subsection{Proof of Theorem~\ref{prop:univ}.}\label{app:proof universality}	
We prepare the proof of Theorem \ref{prop:univ} with the following lemma.
	\begin{lemma} \label{lem:inject}
		Let $ \mathrm{Seq}^1(V) $ be the set of all $ \bx = (\bx_1, \ldots, \bx_L) \in \Seq{V} $ with the form $ x_i = (1, x_i^1, \ldots, x_i^d) $. That is, all sequences where one of the terms is constant. Then the map
		\begin{align}
		\mathrm{Seq}^1(V) \to \T{V}, \quad (\bx_1, \ldots, \bx_L) \to \prod_{i=1}^L (1 + \bx_i)
		\end{align}
		is injective.
	\end{lemma}
	\begin{proof}
		Follows from an induction argument over $ L $. For $ L=1 $ it is clear since
		\begin{align}
		\langle e_i, \Phi(\bx)\rangle = x_i.
		\end{align}
		Assume that it is true for $ L $, let $ \bx = (\bx_1, \ldots, \bx_{L+1}), \by = (\by_1, \ldots, \by_{L+1}) $, where we may assume that both have length $ L+1 $ by taking any number of components to be $ 0 $ if necessary. Let $ \ell_1 $ be some linear function that separates $ \Phi(\bx_1, \ldots, \bx_L) $ and $ \Phi(\by_1, \ldots, \by_L) $ and $ \ell_2 $ some linear function that separates $ \Phi(\bx_2, \ldots, \bx_{L+1}) $ and $ \Phi(\by_2, \ldots, \by_{L+1}) $, then by fixing some $ \gamma \in \R $:
		\begin{align}
		&\langle \ell_1\otimes e_0 + \gamma e_0\otimes \ell_2, \Phi(\bx)-\Phi(\by)\rangle \\
		&= \langle \ell_1, \Phi(\bx_1, \ldots, \bx_L)-\Phi(\by_1, \ldots, \by_L)\rangle + \gamma\langle \ell_2, \Phi(\bx_2, \ldots, \bx_{L+1})-\Phi(\by_2, \ldots, \by_{L+1})\rangle.
		\end{align}
		Since neither $ \langle \ell_1, \Phi(\bx_1, \ldots, \bx_L)-\Phi(\by_1, \ldots, \by_L)\rangle $ nor $ \langle \ell_2, \Phi(\bx_2, \ldots, \bx_{L+1})-\Phi(\by_2, \ldots, \by_{L+1})\rangle $ are $ 0 $ by assumption there exists some $ \gamma \in \R $ such that $ \langle \ell_1\otimes e_0 + \gamma e_0\otimes \ell_2, \Phi(\bx)-\Phi(\by)\rangle \not= 0 $. This shows the assertion.
	\end{proof}
  We now have everything to give a proof of Theorem~\ref{prop:univ}.
  %Recall that $ \Phi $ is the map $ \Phi(\bx) = \prod_{i=1}^L \varphi(\bx_i) $ and that we may assume that $ \varphi $ has the form $ \varphi(\bx) = 1 + \varphi_1(\bx) \in (V^{\otimes 0}, V)$, where $ \varphi_1 : \cX \to V $ is a bounded universal map.
	\begin{proof}[Proof of Theorem~\ref{prop:univ}]
		We will show that linear functionals on $ \Phi $ are dense in the \emph{strict topology} \citep{Giles71}. By Theorem \cite[Theorem 3.1]{Giles71} it is enough to show that linear functions on $ \Phi $ form an algebra since by Lemma \ref{lem:inject} they separates the points of $ \Seq{\cX} $. Since they clearly form a vector space it is enough to show that they are closed under point-wise multiplication. Let $ \ell_1,\ell_2 $ be two such, then by Lemma \ref{lem:shuff}
		\begin{align}
		\langle \ell_1, \Phi(\bx) \rangle \langle \ell_2, \Phi(\bx) \rangle = \langle \ell_1\star \ell_2, \prod_{i=1}^L \phi(\bx_i)^\star \rangle
		\end{align}
		so it is enough to show that $ \ell_1\star \ell_2 $ also is a linear function on $\Phi(\bx) $. Note that inductively it is enough to show that if $ e_i,e_j $ are unit vectors, then $ e_i\otimes e_j $ is a linear function on $\Phi(\bx) $. By assumption $ \phi $ is bounded and universal, so the continuous bounded function $ \bx \mapsto \langle e_i, \phi(\bx_k)\rangle \langle e_j,\phi(\bx_k)\rangle $ is approximately linear, and we may write
		\begin{align}
		\langle e_i, \phi(\bx_k)\rangle \langle e_j,\phi(\bx_k)\rangle = \langle h, \phi(\bx_k) \rangle + \varepsilon(\bx_k)
		\end{align}
		where $ \varepsilon(\bx_k) $ can be made arbitrarily small in the strict topology. The assertion now follows since
		\begin{align}
		&\langle e_i\otimes e_j, \prod_{i=1}^L \phi(\bx_i)^\star) \rangle = \sum_{k=1}^L \langle e_i, \phi(\bx_k)\rangle \langle e_j,\phi(\bx_k)\rangle = \sum_{k=1}^L \langle h, \phi(\bx_k)\rangle + \varepsilon(\bx_k) \\
		&= \langle e_{h}, \Phi(\bx) \rangle + n\max_{1\leq k\leq n}\varepsilon(\bx_k).
		\end{align}
	\end{proof}

\subsection{Seq2Tens makes sequences of different length comparable}\label{sec:different length}	
The simplest kind of a sequence is a string, that is a sequence of letters. Strings are determined by
\begin{enumerate}[label=(\roman*)]
\item  
  what letters appear in them,
  \item 
    in what order the letters appear.
\end{enumerate}
A classical way to produce a graded description of strings is by counting their non-contiguous sub-strings.
These are the so-called \emph{$ m $-mers}; % which count all the non-contiguous sub-strings of length $ \leq m $.
for example,
\begin{align}
 \text{The string "aabc" has the 2-mers }  \{  aa, ab, ac, ab, ac, bc\}.
\end{align}
Measuring similarity between strings by counting how many substrings they have in common is a sensible similarity measure, even if the strings have different length; we refer to ~\citet{Leslie04} for applications and to \cite[Chapter 11]{Taylor00} for detailed introduction to the use of substrings in ML. 
% The resulting feature map that maps a string to vector indexed by $k$-mers can be kernelized and has found many applications from text analysis to bioinformatics.  

Our Seq2Tens feature map can be regarded as a vast generalization of such subpattern matching: if $\Phi(\bx) = (\Phi_m(\bx))_{m \ge 0} \in \T{V}$ then the tensor $\Phi_m(\bx) \in V^{\otimes m}$ represents non-contiguous sub-sequences of $\bx=(\bx_1,\ldots,\bx_T)$ of length $m$ and thus comparing $\Phi_m(\bx)$ and $\Phi_m(\by)$ is meaningful even when $\bx$ and $\by$ are of different length. 
It is instructive to spell out in detail how $m$-mers are a special case of Seq2Tens, Example \ref{ex: strings}, and how it generalizes, Example~\ref{ex: signature pattern}.
\begin{example}\label{ex: strings}
 Let $\cX=\{a,b,c\}$ and $\phi:\cX \to V=\R^3$ defined by mapping $a,b,c \in \cX$ to the unit vectors $e_1,e_2,e_3 \in V$, so that $\varphi(a)=(1,e_1,0,0,...) \in \T{V}$, $\varphi(b)=(1,e_2,0,\ldots) \in \T{V}$, and $\varphi(c)=(1,e_3,0,\ldots) \in \T{V}$. 
For the sequence $\bx=(a,a,b,c) \in \Seq{\cX}$ we get 
  \begin{align}
    \Phi(\bx) =& \varphi(a)\varphi(a)\varphi(b)\varphi(c)\\
              =& (1,e_1,0,0,\ldots,) \cdot (1,e_1,0,0,\ldots) \cdot (1,e_2,0,0,\ldots)\cdot (1,e_3,0,0,\ldots)\\
              =& (1, \underbrace{2e_1 + e_2 + e_3}_{\in V},\underbrace{ e_1\otimes e_1 + 2 e_1 \otimes e_2+ 2e_1 \otimes e_3 + e_2 \otimes e_3}_{\in V^{\otimes 2}},\\
                & \underbrace{e_1\otimes e_1 \otimes e_2+e_1\otimes e_1 \otimes e_3, e_1 \otimes e_1 \otimes e_3 \otimes e_4}_{\in V^{\otimes 3}},\underbrace{e_1 \otimes e_1 \otimes e_2 \otimes e_3}_{\in V^{\otimes 4}},0 , \ldots ).
  \end{align}
We see that the tensor $\Phi(\bx) \in V^{\otimes m}$ of degree $m$ contains the
$m$-mers, that is the coordinates of $\Phi_m(\bx)$  count how often a subsequence of length $m$ in $\bx=(a,a,b,c)$ appears; e.g.~the coordinate $e_1 \otimes e_2$ of $\Phi_2(\bx)$ equals $2$ because the substring ``a,b'' appears twice but the $e_1\otimes e_1$ coordinate of $\Phi_2(\bx)$ equals $1$ since the substring ``a,a'' appears only once in $(a,a,b,c)$, etc.  
Similarly, the only non-zero coordinate of $\Phi_4(\bx)$ is $e_1 \otimes e_1 \otimes e_2 \otimes e_3$ since "a,a,b,c" is the only substring of $(a,a,b,c)$ and consequently all coordinate of $\Phi_m(\bx)$ are $0$ for $m \ge 5$. 
\end{example}
An analogous calculation shows that for continuous domain such as $\cX=\R^d$ the coordinates of $\Phi_m(\bx) \in (\R^d)^{\otimes m}$ measure movement patterns in these coordinates. 
Section \ref{sec:convergence to sig} below shows that this interpretation also holds when we go from discrete time (sequences) to continuous time (paths); Example \ref{ex: signature pattern}.
\subsection{Convergence from discrete to continuous time}\label{sec:convergence to sig}	
A common source of of sequence data is to measure a quantity $(x(t))_{t \in [0,T]}$ that evolves in continuous time at fixed times $t_1,\ldots,t_L$ to produce a sequence $\bx=(x(t_1),\ldots,x(t_L))  \in \Seq{\cX}$.
Often the measurements are of high-frequency ($|t_{i+1}-t_i| \to 0$ and $L \to \infty$) and is interesting to understand how our Seq2Tens approach behaves in this limiting case.
As it turns out, when combined with taking finite differences\footnote{This is necessary to counteract the fact the magnitude of $ \Phi $ grows with the sum of the elements in the sequence} our feature map $\Phi(\bx)$ converges to a classical object in analysis, the so-called \emph{signature of the path} $x$ in this limit, \citep{primer2016} .
For brevity, we spell it out here for smooth paths and with the lift $ \varphi(x) = (1, \phi(x), 0, \ldots) $, but readers familiar with rough paths will notice that the result generalizes even to non-smooth paths such as Brownian motion.
\begin{proposition}\label{prop:convergence to signature}
  Let $x \in C^1([0,1],\R^d)$ and for every $L \ge 1$ define 
  \begin{align}
  \bx^L:=(x(t_0^L), x(t_1^L) - x(t_0^L), \ldots, x(t_L^L)-x(t_{L-1}^L))
  \end{align}
  where $t_i:= \frac{i}{L}$ for $i=0,\ldots,L$.
  Then for every $m \ge 0$ we have
  \begin{align}
    \Phi_m(\bx^L) \to \int_{0 \le t_1 \le \cdots \le t_m \le 1} \frac{dx}{dt}(t_1) \otimes \cdots \otimes \frac{dx}{dt}(t_m) dt_1 \cdots d t_m  \text{ as } L \to \infty. 
  \end{align}
\end{proposition}
\begin{proof}
  Using the recurrence relation from the proof of Lemma \ref{lem:shuff} we may write
\begin{align}
\Phi_m(\bx^L) &= \sum_{i=1}^L \Phi_{m-1}(\bx_i^L) \otimes \big(x(t_{i+1}) - x(t_i)\big) 
\end{align}
where $ \bx_i^L $ denotes the sequence $ (x(t_0^L), x(t_1^L) - x(t_0^L), \ldots, x(t_i^L)-x(t_{i-1}^L)) $. By a Taylor expansion this is equal to
\begin{align}
\Phi_m(\bx^L) &= \sum_{i=1}^L \Phi_{m-1}(\bx_i^L) \otimes \frac{dx}{dt}(t_i) (t_{i+1} - t_i) + O(1/L),
\end{align}
so by unravelling the recurrence relation we may write
\begin{align}
\Phi_m(\bx^L) &= \sum_{1 \leq i_1 < \ldots < i_m \leq L} \frac{dx}{dt}(t_{i_1}) \otimes \cdots \otimes \frac{dx}{dt}(t_{i_m}) (t_{{i_1}+1} - t_{i_1})\cdots(t_{{i_m}+1} - t_{i_m}) + O(1/L),
\end{align}
which is a Riemann sum, and thus converges to the asserted limit.
\end{proof}
The interpretation of $\Phi_m(\bx)$ as counting sub-patterns remains even in the continuous time case: 
\begin{example}\label{ex: signature pattern}
Let $x \in C^1([0,1],\R^d)$ and $e_1,\ldots,e_d$ the standard basis of $V=\R^d$.  
For $m=1$ the $e_1$ coordinate equals 
\[\langle e_1, \int_{t = 0}^1 \frac{dx}{dt}(t) dt \rangle =  \langle e_1, x(1)-x(0) \rangle\]
which measures the total movement of the path $ x $ in the direction $ e_1 $.
Analogously, for $m=2$ the $e_1 \otimes e_2$ coordinate equals
\begin{align}
  \langle e_1 \otimes e_2, \int_{0 \le t_1 \le \cdots \le t_m \le 1} \frac{x(t_1)}{dt_1} \otimes \frac{x(t_2)}{dt_2} dt_1 dt_2= \int_{0 \le t_1 \le t_2 \le 1} \langle e_1, \frac{x(t_1)}{dt_1} \rangle \langle e_2, \frac{x(t_2)}{dt_2}\rangle dt_1 dt_2.   
\end{align}
which measure the number of ordered tuples $(t_1,t_2)$, $t_1 < t_2$, such that $x$ moves in direction $e_1$ at time $t_1$ and subsequently in direction $e_2$ at time $t_2$.
\end{example}

\section{Stacking sequence-to-sequence transforms}\label{app:iterated}
% \paragraph{Seq2Tens in a nutshell.}	
% Even computing just the first $M$ entries of $\Phi=(\Phi_m(\bx))_{m \ge 0} \in \T{\R^d}$ is very costly since it involves $O(d^m)$ coordinates. 
% The key of our approach is that \emph{we never compute} $\Phi(\bx)$ (or the first $m$ levels)! 
% Instead, we use formula~\eqref{eq:sum} as derived in Section~\ref{app:quasi}, that shows that LR functionals $\langle \ell, \Phi(\bx)  \rangle$ are extremely fast to evaluate. 
% In the remainder of this Section~\ref{app:iterated} we formalize idea 3, in the sense that we prove that compositions of very simple objects (LR functionals) efficiently approximate complicated objects (full rank functionals of sequences), Proposition~\ref{prop:half}.  
	As $\Phi$ applies to sequences of any length, we may use it to map the original sequence to another sequence in feature space,
	\begin{align} 
	\Seq{V} &\rightarrow \Seq{\T{V}}\\
	(\bx_1,\bx_2,\bx_3,\ldots, \bx_L) &\mapsto \Big(\Phi(\bx_1),\Phi(\bx_1, \bx_2), \Phi(\bx_1, \bx_2, \bx_3), \ldots, \Phi(\bx_1,\ldots,\bx_L) \Big).
	\end{align}
	Since $\T{V}$ is again a linear space, we can repeat this procedure to map $\Seq{\T{V}} $ to $ \Seq{\T{\T{V}}}$.
	By repeating this $D$ times, we have constructed sequence-to-sequence transforms 
	\begin{align}\label{eq:seq2seq}
	\Seq{V}=\Seq{\T{V}} \rightarrow \Seq{\T{\T{V}}} \rightarrow \cdots \rightarrow \Seq{\underbrace{T(\cdots T}_{D \text{ times }}(V)\cdots )}.
	\end{align}
	
	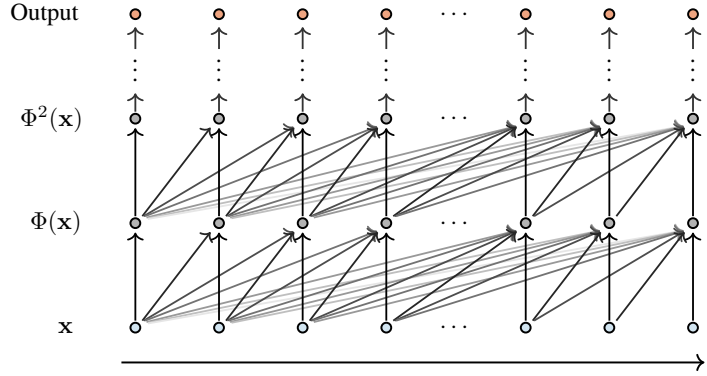
\begin{figure}
		\centering
		\resizebox{.7\textwidth}{!}{
				\begin{tikzpicture}[shorten >=1pt,draw=black!50, node distance=\layersep, x=2cm]
				\node[text width=2em, text centered] at (-.5,0) {$ \mathbf{x} $};
				
				\draw[black, thick,  ->] (-.1,-.5) -- (4.1,-.5);
				
				\draw[black, fill=blue!30, thick] (0,0) circle (2pt);
%				\draw[black, ->] (0.1,0) -- (0.5,0);
				\draw[black, fill=blue!30, thick] (.6,0) circle (2pt);
%				\draw[black, ->] (0.7,0) -- (1.1,0);
				\draw[black, fill=blue!30, thick] (1.2,0) circle (2pt);
%				\draw[black, ->] (1.3,0) -- (1.7,0);
				\draw[black, fill=blue!30, thick] (1.8,0) circle (2pt);
				
				\node[text width=2em, text centered] at (2.3,0) {$ \cdots $};
				
				\draw[black, fill=blue!30, thick] (2.8,0) circle (2pt);
%				\draw[black, ->] (2.9,0) -- (3.3,0);
				\draw[black, fill=blue!30, thick] (3.4,0) circle (2pt);
%				\draw[black, ->] (3.5,0) -- (3.9,0);
				\draw[black, fill=blue!30, thick] (4,0) circle (2pt);

				\node[text width=2em, text centered] at (-.58,1.5) {$ \Phi(\mathbf{x}) $};
				
				\draw[black, fill=black!30, thick] (0,1.5) circle (2pt);
%				\draw[black, ->] (0.1,1.5) -- (0.5,1.5);
				\draw[black, ->] (0,0.1) -- (0,1.4);
				\draw[black!85, ->] (0.05,0.1) -- (0.55,1.4);
				\draw[black!70, ->] (0.06,0.1) -- (1.15,1.4);
				\draw[black!55, ->] (0.07,0.1) -- (1.75,1.4);
				\draw[black!40, ->] (0.08,0.08) -- (2.75,1.4);
				\draw[black!25, ->] (0.09,0.07) -- (3.35,1.4);
				\draw[black!10, ->] (0.09,0.06) -- (3.95,1.4);
				
				\draw[black, fill=black!30, thick] (.6,1.5) circle (2pt);
%				\draw[black, ->] (0.7,1.5) -- (1.1,1.5);
				\draw[black, ->] (.6,0.1) -- (0.6,1.4);
				\draw[black!85, ->] (.65,0.1) -- (1.15,1.4);
				\draw[black!70, ->] (0.66,0.1) -- (1.75,1.4);
				\draw[black!55, ->] (0.67,0.1) -- (2.75,1.4);
				\draw[black!40, ->] (0.68,0.08) -- (3.35,1.4);
				\draw[black!25, ->] (0.69,0.07) -- (3.95,1.4);
				
				\draw[black, fill=black!30, thick] (1.2,1.5) circle (2pt);
%				\draw[black, ->] (1.3,1.5) -- (1.7,1.5);
				\draw[black, ->] (1.2,0.1) -- (1.2,1.4);
				\draw[black!85, ->] (1.25,0.1) -- (1.75,1.4);
				\draw[black!70, ->] (1.26,0.1) -- (2.75,1.4);
				\draw[black!55, ->] (1.27,0.1) -- (3.35,1.4);
				\draw[black!40, ->] (1.28,0.08) -- (3.95,1.4);
				
				\draw[black, fill=black!30, thick] (1.8,1.5) circle (2pt);
				\draw[black, ->] (1.8,0.1) -- (1.8,1.4);
				\draw[black!85, ->] (1.85,0.1) -- (2.75,1.4);
				\draw[black!70, ->] (1.86,0.1) -- (3.35,1.4);
				\draw[black!55, ->] (1.87,0.1) -- (3.95,1.4);
				
				\node[text width=2em, text centered] at (2.3,1.5) {$ \cdots $};
				
				\draw[black, fill=black!30, thick] (2.8,1.5) circle (2pt);
%				\draw[black, ->] (2.9,1.5) -- (3.3,1.5);
				\draw[black, ->] (2.8,0.1) -- (2.8,1.4);
				\draw[black!85, ->] (2.85,0.1) -- (3.35,1.4);
				\draw[black!70, ->] (2.86,0.1) -- (3.95,1.4);
				
				\draw[black, fill=black!30, thick] (3.4,1.5) circle (2pt);
%				\draw[black, ->] (3.5,1.5) -- (3.9,1.5);
				\draw[black, ->] (3.4,0.1) -- (3.4,1.4);
				\draw[black!85, ->] (3.45,0.1) -- (3.95,1.4);
				
				\draw[black, fill=black!30, thick] (4,1.5) circle (2pt);
				\draw[black, ->] (4,0.1) -- (4,1.4);

				\node[text width=2em, text centered] at (-.65,3) {$ \Phi^2(\mathbf{x}) $};
				
				\draw[black, fill=black!30, thick] (0,3) circle (2pt);
%				\draw[black, ->] (0.1,3) -- (0.5,3);
				\draw[black, ->] (0,1.6) -- (0,2.9);
				\draw[black!85, ->] (0.05,1.6) -- (0.55,2.9);
				\draw[black!70, ->] (0.06,1.6) -- (1.15,2.9);
				\draw[black!55, ->] (0.07,1.6) -- (1.75,2.9);
				\draw[black!40, ->] (0.08,1.58) -- (2.75,2.9);
				\draw[black!25, ->] (0.09,1.57) -- (3.35,2.9);
				\draw[black!10, ->] (0.09,1.56) -- (3.95,2.9);
				
				\draw[black, fill=black!30, thick] (.6,3) circle (2pt);
%				\draw[black, ->] (0.7,3) -- (1.1,3);
				\draw[black, ->] (.6,1.6) -- (0.6,2.9);
				\draw[black!85, ->] (.65,1.6) -- (1.15,2.9);
				\draw[black!70, ->] (0.66,1.6) -- (1.75,2.9);
				\draw[black!55, ->] (0.67,1.6) -- (2.75,2.9);
				\draw[black!40, ->] (0.68,1.58) -- (3.35,2.9);
				\draw[black!25, ->] (0.69,1.57) -- (3.95,2.9);
				
				\draw[black, fill=black!30, thick] (1.2,3) circle (2pt);
%				\draw[black, ->] (1.3,3) -- (1.7,3);
				\draw[black, ->] (1.2,1.6) -- (1.2,2.9);
				\draw[black!85, ->] (1.25,1.6) -- (1.75,2.9);
				\draw[black!70, ->] (1.26,1.6) -- (2.75,2.9);
				\draw[black!55, ->] (1.27,1.6) -- (3.35,2.9);
				\draw[black!40, ->] (1.28,1.58) -- (3.95,2.9);
				
				\draw[black, fill=black!30, thick] (1.8,3) circle (2pt);
				\draw[black, ->] (1.8,1.6) -- (1.8,2.9);
				\draw[black!85, ->] (1.85,1.6) -- (2.75,2.9);
				\draw[black!70, ->] (1.86,1.6) -- (3.35,2.9);
				\draw[black!55, ->] (1.87,1.6) -- (3.95,2.9);
				
				\node[text width=2em, text centered] at (2.3,3) {$ \cdots $};
				
				\draw[black, fill=black!30, thick] (2.8,3) circle (2pt);
%				\draw[black, ->] (2.9,3) -- (3.3,3);
				\draw[black, ->] (2.8,1.6) -- (2.8,2.9);
				\draw[black!85, ->] (2.85,1.6) -- (3.35,2.9);
				\draw[black!70, ->] (2.86,1.6) -- (3.95,2.9);
				
				\draw[black, fill=black!30, thick] (3.4,3) circle (2pt);
%				\draw[black, ->] (3.5,3) -- (3.9,3);
				\draw[black, ->] (3.4,1.6) -- (3.4,2.9);
				\draw[black!85, ->] (3.45,1.6) -- (3.95,2.9);
				
				\draw[black, fill=black!30, thick] (4,3) circle (2pt);
				\draw[black, ->] (4,1.6) -- (4,2.9);
				
				\draw[black!80, ->] (4,3.1) -- (4,3.4);
				\node[text width=2em, text centered] at (4,3.8) {$ \vdots $};
				\draw[black!80, ->] (3.4,3.1) -- (3.4,3.4);
				\node[text width=2em, text centered] at (3.4,3.8) {$ \vdots $};
				\draw[black!80, ->] (2.8,3.1) -- (2.8,3.4);
				\node[text width=2em, text centered] at (2.8,3.8) {$ \vdots $};
				\draw[black!80, ->] (1.8,3.1) -- (1.8,3.4);
				\node[text width=2em, text centered] at (1.8,3.8) {$ \vdots $};
				\draw[black!80, ->] (1.2,3.1) -- (1.2,3.4);
				\node[text width=2em, text centered] at (1.2,3.8) {$ \vdots $};
				\draw[black!80, ->] (0.6,3.1) -- (0.6,3.4);
				\node[text width=2em, text centered] at (.6,3.8) {$ \vdots $};
				\draw[black!80, ->] (0,3.1) -- (0,3.4);
				\node[text width=2em, text centered] at (0,3.8) {$ \vdots $};

				\node[text width=2em, text centered] at (-.72,4.5) {Output};
				
				\draw[black, fill=red!40!yellow!60, thick] (0,4.5) circle (2pt);
%				\draw[black, ->] (0.1,4.5) -- (0.5,4.5);
				\draw[black!80, ->] (0,4) -- (0,4.35);
				\draw[black, fill=red!40!yellow!60, thick] (.6,4.5) circle (2pt);
%				\draw[black, ->] (0.7,4.5) -- (1.1,4.5);
				\draw[black!80, ->] (0.6,4) -- (0.6,4.35);
				\draw[black, fill=red!40!yellow!60, thick] (1.2,4.5) circle (2pt);
%				\draw[black, ->] (1.3,4.5) -- (1.7,4.5);
				\draw[black!80, ->] (1.2,4) -- (1.2,4.35);
				\draw[black, fill=red!40!yellow!60, thick] (1.8,4.5) circle (2pt);
				\draw[black!80, ->] (1.8,4) -- (1.8,4.35);
				
				\node[text width=2em, text centered] at (2.3,4.5) {$ \cdots $};
				
				\draw[black, fill=red!40!yellow!60, thick] (2.8,4.5) circle (2pt);
%				\draw[black, ->] (2.9,4.5) -- (3.3,4.5);
				\draw[black!80, ->] (2.8,4) -- (2.8,4.35);
				\draw[black, fill=red!40!yellow!60, thick] (3.4,4.5) circle (2pt);
%				\draw[black, ->] (3.5,4.5) -- (3.9,4.5);
				\draw[black!80, ->] (3.4,4) -- (3.4,4.35);
				\draw[black, fill=red!40!yellow!60, thick] (4,4.5) circle (2pt);
				\draw[black!80, ->] (4,4) -- (4,4.35);
				\end{tikzpicture}
			}
		\caption{The sequence-to-sequence transformation. The bottom row is the original sequence $ \mathbf{x} $ and subsequent ones apply the map $ \Phi $ in the sequence-to-sequence manner.}
		\label{fig:seq2seq}
	\end{figure}
	
	See Figure \ref{fig:seq2seq} for an illustration. We emphasize that in each step of the iteration, the newly created sequence evolves in a much richer space than in the previous step. To make this precise we now introduce the \emph{higher rank free algebras}.
	
\paragraph{Higher rank free algebras.}	
	Just like in Appendix \ref{app:algebra} we need to enlarge the ambient space $ \T{V} $. Recall that we defined $ \Tra{2}{V} := \T{\T{V}} $. This construction can be iterated indefinitely and leads to the higher rank free algebras, recursively defined as follows:
	\begin{definition}
		Define the spaces
		\begin{align}
		\Tra{0}{V} &= V, \quad 
		\Tra{D}{V} = \prod_{m\geq 0} \big( \Tra{D-1}{V}\big)^{\otimes m}. 
		\end{align}
	\end{definition}
	We use the notation $ \otimes_{(D)} $ for the tensor product on $ \Tra{D-1}{V} $ which makes $ \big(\Tra{D}{V}, +, \otimes_{(D)}\big) $ into a multi-graded algebra over $ \vs $. See \citet{BCO20} for more on this iterated construction.
\paragraph{Half-shuffles.}	
	By iterating the sequence-to-sequence $ D $ times, one gets a map 
	\begin{align}
	\Phi^D : \Seq{V} \to \Tra{D}{V}.
	\end{align} 
	These are very large spaces, but as we will see in Proposition \ref{prop:half} below, linear functionals on the full map $ \Seq{V} \to \Tra{D}{V} $ can be de-constructed into so called \emph{half-quasi-shuffle} on the original map $ \Phi : \Seq{V} \to \T{V} $.
	
	Just like in Appendix \ref{app:algebra} we consider the sequence $ \bx $ as its extension $ \bx^\star $ taking values in $ \T{V} $. Hence linear functionals can be written as linear combinations of elements of the form $ e_{i_1} \otimes_{(2)} \cdots \otimes_{(2)} e_{i_n} $.
	
	\begin{definition}
		The half-quasi-shuffle product is defined on rank $ 1 $ tensors by
		\begin{align}
		\ell_1\prec (\ell_2\otimes_{(2)} e_i) = (\ell_1\star \ell_2)\otimes_{(2)} e_i 
		\end{align}
		and extends by bi-linearity to a map $ \Tra{2}{V}\times \Tra{2}{V} \to \Tra{2}{V} $.
	\end{definition}
	
	Proposition \ref{prop:half} shows that by composing $ \Phi $ with itself, low degree tensors on the second level can be rewritten as higher degree tensors on the first level. This indicates that iterated compositions of $ \Phi $ can be much more efficient than computing everything on the first level. We show this for the first level, but by iterating the statement it can be applied for any number $ D \geq 2 $. 
	
%	For the next statement, let $ \ell $ be a linear functional on $ \Tra{2}{V} $, then the unit vector $ e_\ell $ is a linear functional on $ \Tra{3}{V} $.
%  We again assume without loss of generality that $ \varphi(\bx) = (1, \bx,0,\ldots) $.
	
	\begin{proposition} \label{prop:half}
		Let $ \bphi $ be the sequence-to-sequence transformation:
		\begin{align}
		\Seq{V} \to \Seq{\T{V}}, \quad (\bx_1, \ldots, \bx_L) \mapsto (\Phi(\bx_1), \ldots, \Phi(\bx, \ldots, \bx_L))
		\end{align}
		and let $ \Delta (\bx_0, \ldots, \bx_L) = (\bx_1-\bx_0, \ldots, \bx_L-\bx_{L-1}) $. Then
		\begin{align}
		\langle e_{\ell_1}\otimes_{(3)} e_{\ell_2}, \Phi(\Delta\bphi(\bx)) \rangle = \langle \ell_1 \prec \ell_2, \Phi(\bx)\rangle
		\end{align}
	\end{proposition}
	\begin{proof}
    We use the notation $ \Phi(\bx)_k = \Phi(\bx_1, \ldots, \bx_k) $.
		By induction:
		\begin{align}
		&\langle e_{\ell_1}\otimes_{(3)} e_{\ell_2\otimes_{(2)} e_i}, \Phi(\Delta\bphi(\bx)) \rangle \\
		&= \sum_{1 \leq k_1 < k_2 \leq L} \big(\langle \ell_1, \Phi(\bx)_{k_1}-\langle \ell_1, \Phi(\bx)_{k_1-1}\big) \rangle \big(\langle \ell_2\otimes_{(2)} e_i, \Phi(\bx)_{k_2}\rangle-\langle \ell_2\otimes_{(2)} e_i, \Phi(\bx)_{k_2-1}\rangle\big) \\
		&= \sum_{k=1}^{L-1} \langle \ell_1, \Phi(\bx)_{k} \rangle \big(\langle \ell_2\otimes_{(2)} e_i, \Phi(x)_{k+1}\rangle-\langle \ell_2\otimes_{(2)} e_i, \Phi(\bx)_{k}\rangle\big) \\
		&= \sum_{k=1}^{L-1} \langle \ell_1, \Phi(\bx)_{k} \rangle \big( \sum_{1\leq l\leq k} \langle \ell_2, \Phi(\bx)_{l}\rangle \bx^i_{l+1}-\langle \ell_2, \Phi(\bx)_{l-1}\rangle \bx^i_{l}\big) \\
		&= \sum_{k=1}^{L-1} \langle \ell_1, \Phi(\bx)_{k} \rangle \langle \ell_2, \Phi(\bx)_{k}\rangle \bx^i_{k+1} 
		= \sum_{k=1}^{L-1} \langle \ell_1\star \ell_2, \Phi(\bx)_{k} \rangle \bx^i_{k+1} \\
		&= \langle (\ell_1\star \ell_2)\otimes_{(2)} e_i, \Phi(\bx)_{L} \rangle.
		\end{align}
	\end{proof}

\section{Details on computations} \label{app:algorithms_detail}
Here we give further information on the implementation of LS2T layers detailed in the main text. For simplicity, we fix the state-space of sequences to be $V  = \R^d$ from here onwards. We also remark that although some of the following considerations and techniques might look unusual for the standard ML audience, they are well-known in the signatures community \citep{morrill2020generalised}

\subsection{Variations} \label{app:variations}
\paragraph{Truncation degree.}
To reiterate from Section \ref{sec:2}, for a given static feature map $\phi:\R^d \to V$  the Seq2Tens feature map $\Phi: \Seq{\R^d} \rightarrow \T{V}$ represents a sequence $\bx = (\bx_1, \dots, \bx_L) \in \Seq{\R^d}$ as a tensor in $T(V)$,  
\begin{align}
    \Phi(\bx) = (\Phi_m(\bx))_{m \geq 0}, \quad \Phi_m(\bx) = \sum_{1 \leq i_1 < \dots < i_m \leq L} \phi(\bx_{i_1}) \otimes \cdots \otimes \phi(\bx_{i_m}),
\end{align}
where $\Phi_m: \Seq{\R^d} \rightarrow V^{\otimes m}$ is given by a summation over all noncontiguous length-$m$ subsequences of $\bx$ with non-repeating indices. Therefore, for a sequence of length $L \in \NN$, $\Phi_m$ can have potentially non-zero terms for $m \leq L$.
An empirical observation is that for most datasets computing everything up to the $L$th level is redundant in the sense that usually the first $M \in \NN$ levels already contain most of the information a discriminative or a generative model picks up on where $M \ll L$. It is thus better treated as a hyperparameter, which we call ``order'' in the main text.
%Note that this is similar to choosing a truncation degree in signature methods, which has been studied in \cite{fermanian2020linear}.

Below we take for brevity $\phi(\bx_i)=\bx_i \in V=\R^d$ since with other maps $\phi$ simply amounts to replacing $\bx_i \in \R^d$ by $\phi(\bx_i) \in \R^e$. 
\paragraph{Distinguishing functionals across levels.} Let us consider the LS2T map $\tilde\Phi_{\tilde\theta}$, each output coordinate of which is given by a linear functional of $\Phi$, i.e. $\tilde\Phi_{\tilde\theta}(\bx) = \left(\langle \ell^1, \Phi(\bx) \rangle, \dots, \langle \ell^n, \Phi(\bx), \rangle\right)$ for a sequence $\bx \in \Seq{\R^d}$ and a collection of rank-$1$ elements $\theta = (\ell^k)_{k=1}^n \subset T(\R^d)$. Then, a single output coordinate of $\tilde\Phi_{\theta}$ may be written for $1 \leq k \leq n$ as
\begin{align}
\langle \ell^k, \Phi(\bx) \rangle = \sum_{m=0} \langle \ell^k_m, \Phi_m(\bx) \rangle,
\end{align}
for $\ell^k = (\ell^k_m)_{m \geq 0}$, i.e. we take inner products of tensors that are of the same degree, and then sum these up. We found that rather than taking the summation across tensor levels, it is beneficial to treat the linear functional on each level of the free algebra as an independent output to have
\begin{align}
    \tilde\Phi_{m, \tilde\theta}(\bx) = (\langle \ell^1_m, \Phi_m(\bx), \dots, \langle \ell^n_m, \Phi_m(\bx) \rangle) \quad\text{and}\quad \tilde\Phi_{\tilde\theta}(\bx) = (\tilde\Phi_{m, \tilde\theta}(\bx))_{m \geq 0}, 
\end{align}
where now $\tilde\Phi_{\tilde\theta}$ has output dimensionality $(M \times n)$ with $M \in \NN$ the truncation degree of $\Phi$ as detailed in the previous paragraph. Hence this modification scales the output dimension by $M$, but it will be important for the next step we discuss. It is for this modification that in Figure \ref{fig:ls2t_architectures}, each output of a LS2T layer has dimensionality $n \times m$ for a width-$n$ and order-$m$ LS2T map, while in Figure \ref{fig:gpvae_ls2t_encoder} the B-LS2T layer has output dimensionality $h \times 4$, since we set $n=h$ and $m=4$.

\paragraph{The need for normalization.}
Here we motivate the need to follow each LS2T layer by some form of normalization. Let $\bx = (\bx_1, \dots, \bx_L) \in \Seq{\R^d}$ be a sequence. Let $\alpha \in \R$ be a scalar and define $\by = \alpha \bx \in \Seq{\R^d}$ a scaled version of $\bx$. Let us investigate how the features change:
\begin{align}
    \Phi_m(\by) &= \sum_{1 \leq i_1 < \cdots < i_M \leq L} \by_{i_1} \otimes \cdots \otimes \by_{i_m} 
    = \sum_{1 \leq i_1 < \cdots < i_M \leq L} (\alpha \bx_{i_1}) \otimes \cdots \otimes (\alpha \bx_{i_m}) \\
    &=  \alpha^m \sum_{1 \leq i_1 < \cdots < i_M \leq L} \bx_{i_1} \otimes \cdots \otimes \bx_{i_m},
\end{align}
and therefore we have $\Phi(\by) = (\Phi_m(\by))_{m \geq 0} = (\alpha^m \Phi_m(\bx))_{m \geq 0}$, which analogously translates into the low-rank Seq2Tens map since

\begin{align}
    \tilde\Phi_{m, \tilde\theta} (\by) &= (\langle \ell^1_m, \Phi_m(\by) \rangle, \dots, \langle \ell^n_m, \Phi_m(\by) \rangle) 
    = (\langle \ell^1_m, \alpha^m \Phi_m(\bx) \rangle, \dots, \langle \ell^n_m, \alpha^m \Phi_m(\bx) \rangle) \\
    &= \alpha^m (\langle \ell^1_m, \Phi_m(\bx) \rangle, \dots, \langle \ell^n_m, \Phi_m(\bx) \rangle).
\end{align}

From this point alone, it is easy to see that $\Phi_m$, and thus $\tilde\Phi_{m, \tilde\theta}$ will move across wildly different scales for different values of $m$, which is inconvenient for the training of neural networks. To counterbalance this, we used a batch normalization layer after each LS2T layer in Sections \ref{subseq:4_tsc}, \ref{subseq:4_mortality} that computes mean and variance statistics across time and the batch itself, while for the GP-VAE in Section \ref{subseq:4_gpvae} we used a layer normalization that computes the statistics only across time. 

\paragraph{Sequence differencing.} In both Section \ref{subseq:4_tsc} and Section \ref{subseq:4_gpvae}, we precede each LS2T layer by a differencing layer and a time-embedding layer.

Let $\Delta: \Seq{\R^d} \rightarrow \Seq{\R^d}$ be the discrete difference operator defined for a sequence $\bx = (\bx_1, \dots, \bx_L) \in \Seq{\R^d}$ as
\begin{align}
    \Delta \bx := (\bx_1, \bx_2 - \bx_1, \dots, \bx_L - \bx_{L-1}) \in \Seq{\R^d},
\end{align}
where we made the simple identification that $\bx_0 \equiv 0$, i.e. for all sequences we first concatenate a $\mathbf{0}$ observation along the time axis, in the signature learning community this is called basepoint augmentation \cite{morrill2020generalised}, which is beneficial for two reasons: \begin{enumerate*}[label=(\roman*)] \item now $\Delta$ preserves the length $L$ of a sequence, \item now $\Delta$ is one-to-one, since otherwise $\Delta$ would be translation invariant, i.e. it would map all sequences which are translations of each other to the same output \end{enumerate*}.

To motivate differencing, first let us consider $\Phi_m(\bx)$ for $m=1$, and for brevity denote $\Delta \bx_i = \bx_i - \bx_{i-1}$ for $i=1, \dots, L$ and the convention $\bx_0 = 0$. Then, we may write
\begin{align}
    \Phi_1(\bx) = \sum_{i=1}^L \Delta \bx_i = \bx_L,
\end{align}
which means that now the first level of the Seq2Tens map is simply point-wise evaluation at the last observation time, and when used as a sequence-to-sequence transformation over expanding windows (i.e.~\eqref{eq:seq2seq_lls2t}), it is simply the identity map of the sequence.

Analogously, for the low-rank map we have
\begin{align}
    \tilde\Phi_{1,\theta}(\bx) = \sum_{i=1}^L (\langle \ell^1_1, \Delta \bx_i \rangle, \dots, \langle \ell^n_1, \Delta \bx_i \rangle) = (\langle \ell^1_1, \bx_L \rangle, \dots, \langle \ell^n_1, \bx_L \rangle),    
\end{align}
which is simply a linear map applied to $\bx$ in an observation-wise manner. The higher order terms, $\Phi_m(\bx)$ and $\tilde\Phi_{m, \tilde\theta}(\bx)$ can generally be written as
\begin{align} \label{eq:diff_phi}
    \Phi_m(\bx) &= \sum_{1 \leq i_1 < \dots i_m \leq L} \Delta \bx_{i_1} \otimes \cdots \otimes \Delta\bx_{i_m},
\end{align}
and
\begin{align} \label{eq:diff_tildephi}
    \tilde\Phi_{m, \tilde\theta} &= \sum_{1 \leq i_1 < \cdots < i_m \leq L} (\langle \bz^1_{m, 1}, \Delta \bx_{i_1} \rangle \cdots \langle \bz^1_{m, m}, \Delta \bx_{i_m} \rangle, \dots, \langle \bz^n_{m, 1}, \Delta \bx_{i_1} \rangle \cdots \langle \bz^n_{m, m}, \Delta \bx_{i_m} \rangle)
\end{align}
for some rank-$1$ degree-$m$ tensors $\ell^k_m = \bz_{m, 1}^k \otimes \cdots \otimes \bz_{m,m}^k$ for $k=1, \dots, n$. We observed that this way the higher order terms are relatively stable across time as the length of a sequence increases, while without differencing they can become unstable, exhibit high oscillations, or simply blow-up.

An additional benefit of taking differences is that the maps $\Phi$ and $\tilde\Phi_{\tilde\theta}$ become warping invariant, that is, invariant to time warpings. It is easy to see this by checking that if $\bx_i = \bx_{i-1}$ then $\Delta \bx_i = \mathbf{0}$ and all the corresponding terms in the summations \eqref{eq:diff_phi} and \eqref{eq:diff_tildephi} are zeros.

\paragraph{Time-embeddings.} By time-embedding, we mean adding as an extra coordinate to an input sequence the observation times $(t_i, \bx_i)_{i=1, \dots, L} \in \Seq{\R^{d+1}}$. Some datasets already come with a pre-specified observation-grid, in which case we can use that as a time-coordinate at every use of a time-embedding layer. If there is no pre-specified observation grid, we can simply add a normalized and equispaced coordinate, i.e. $t_i = i / L$.

Time-embeddings can be beneficial preceding both convolutional layers and LS2T layers. For convolutions, it allows to learn features that are not translation invariant \citep{Liu2018coordConv}. For the LS2T layer, the interpretation is slightly different. Note that in both Sections \ref{subseq:4_tsc} and \ref{subseq:4_gpvae}, we employ the time-embedding before the differencing block. This can be equivalently reformulated as after differencing adding an additional constant coordinate to the sequence, i.e. $(t_i - t_{i-1}, \bx_i - \bx_{i-1})_{i=1,\dots,L} \in \Seq{\R^{d+1}}$, where $t_i - t_{i-1} = 1 / L$ is simply a constant. This is motivated by Lemma \ref{lem:inject}, which states that the map $\Phi: \Seq{\R^d} \rightarrow T(V)$ is injective for sequences with a constant coordinate. Thus, the time-embedding before the differencing block is equivalent to adding a constant coordinate after the differencing block, and its purpose is to guarantee injectivity of $\Phi$.

\paragraph{Delay embeddings and convolutions.} A useful preprocessing technique for time series are delay embeddings, which simply amount to augmenting the state-space of sequences with a certain number of previous observations, motivated by Takens' theorem \citep{takens1981detecting,sauer1991embedology}, which ensures that a high-dimensional dynamical system can be reconstructed from low-dimensional observations. Theorem \ref{thm:univ} guarantees  that  if $\phi$ is  a  universal  feature  map  on  the  state-space  then $\Phi$ is universal. The most straightforward nonlinearity to take as $\phi$ is a multilayer perceptron (MLP), that is, $\phi = \phi_D \circ \cdots \circ \phi_1$ with $\phi_{j}(\bx) = \sigma(\mathbf{W}_j \bx + \mathbf{b}_j)$. By combining such dense layers with delay embeddings (lags), one recovers a temporal convolution layer, i.e. $\phi_j(\bx_i^l) = \sigma\left(\sum_{k=0}^l \mathbf{W}_{j,k} \mathbf{x}_{i-k} + \mathbf{b}_j\right)$, which motivates the use of convolutions in the preprocessing layers.

\subsection{Recursive computations} \label{app:recursive}
Next, we show how the computation of the maps $\Phi_m$ and $\Phi_{m, \theta}$ can be formulated as a joint recursion over the tensor levels and the sequence itself. 

Since $\Phi_m$ is given by a summation over all noncontiguous length-$m$ subsequences with non-repetitions of a sequence $\bx \in \Seq{\R^d}$, simple reasoning shows that $\Phi_m$ obeys the recursion across $m$ and time for $2 \leq l \leq L$ and $1 \leq m$
\begin{align} \label{eq:recursive_Seq2Tens}
    \Phi_m(\bx_1, \dots \bx_l) = \Phi_m(\bx_1, \dots \bx_{l-1}) + \Phi_{m-1}(\bx_1, \dots, \bx_{l-1}) \otimes \bx_{l},
\end{align}
with the initial conditions $\Phi_0 \equiv 1$, $\Phi_1(\bx_1) = \bx_1$ and $\Phi_m(\bx_1) = \mathbf{0}$ for $m \geq 2$.

Let $\ell = (\ell_m)_{m \geq 0} \in T(\R^d)$ be a sequence of rank-$1$ tensors with $\ell_m = \bz_{m,1} \otimes \cdots \otimes \bz_{m,m} \in (\R^d)^{\otimes m}$ a rank-$1$ tensor of degree-$m$. Then, $\langle \ell_m, \Phi_m(\bx) \rangle$ may be computed analogously to \eqref{eq:recursive_Seq2Tens} using the recursion for $2 \leq l \leq L$, $1 \leq m$
\begin{align} \label{eq:ls2t_independent}
    \langle \ell_m, \Phi_m(\bx_1, \dots, \bx_l) \rangle =& \langle \bz_{m, 1} \otimes \cdots \otimes \bz_{m, m}, \Phi_m(\bx_1, \dots, \bx_l) \rangle \\ =& \langle \bz_{m, 1} \otimes \cdots \otimes \bz_{m, m}, \Phi_m(\bx_1, \dots, \bx_{l-1}) \rangle \\
    &+ \langle \bz_{m, 1} \otimes \cdots \otimes \bz_{m, m-1}, \Phi_{m-1}(\bx_1, \dots, \bx_{l-1}) \rangle \langle \mathbf{z}_{m,m}, \bx_{l} \rangle \\
    =& \langle \ell_m, \Phi_m(\bx_1, \dots \bx_{l-1}) \rangle \label{eqline:lm_prev} \\
    &+ \langle \bz_{m, 1} \otimes \cdots \otimes \bz_{m, m-1}, \Phi_{m-1}(\bx_1, \dots, \bx_{l-1}) \rangle \langle \mathbf{z}_{m,m}, \bx_{l} \label{eqline:not_lm_1_prev} \rangle
\end{align}
and the initial conditions can be rewritten as the identities $\langle z_{0, 0}, \Phi_0 \rangle = 1$, $\langle \mathbf{z}_{m,1}, \Phi_1(\bx) \rangle = \langle \mathbf{z}_{m,1}, \bx \rangle$ and $\langle \bz_{m,1} \otimes \cdots \otimes \bz_{m, m}, \Phi_m(\bx_1) \rangle = \mathbf{0}$ for $2 \leq m$.

A slight inefficiency of the previous recursion given in \eqref{eqline:lm_prev}, \eqref{eqline:not_lm_1_prev} is that one generally cannot substitute $\langle \ell_{m-1}, \Phi_m(\bx_1, \dots, \bx_{l-1}) \rangle$ for the $\langle \bz_{m,1} \otimes \dots \otimes \bz_{m, m-1}, \Phi_m(\bx_1, \dots, \bx_{l-1}) \rangle \rangle$ term in \eqref{eqline:not_lm_1_prev}, since $\ell_{m-1} \neq \bz_{m,1} \otimes \cdots \otimes \bz_{m, m-1}$ generally. This means that to construct the degree-$m$ linear functional $\langle \ell_m, \Phi_m(\bx_1, \dots, \bx_l) \rangle$, one has to start from scratch by first constructing the degree-$1$ term $\langle \bz_{m,1}, \Phi_1 \rangle$ first, then the degree-$2$ term $\langle \bz_{m,1} \otimes \bz_{m,2}, \Phi_2 \rangle$, and so forth. This further means in terms of complexities that while \eqref{eq:recursive_Seq2Tens} has linear complexity in the largest value of $m$, henceforth denoted by $M \in \NN$, \eqref{eqline:lm_prev}, \eqref{eqline:not_lm_1_prev} has a quadratic complexity in $M$ due to the non-recursiveness of the rank-$1$ tensors $(\ell_m)_{m} = (\bz_{m,1} \otimes \cdots \otimes \bz_{m,m})_m$.

The previous observation indicates that an even more memory and time efficient recursion can be devised by parametrizing the rank-$1$ tensors $(\ell_m)_{m}$ in a recursive way as follows: let $\ell_1 = \bz_1 \in \R^d$ and define $\ell_{m} = \ell_{m-1} \otimes \bz_m \in (\R^d)^{\otimes m}$ for $2 \leq m$, i.e. $\ell_m = \bz_1 \otimes \cdots \otimes \bz_m$ for $\{\bz_1, \dots \bz_m \} \subset \R^d$. This parametrization indeed allows to substitute $\ell_{m-1}$ in \eqref{eqline:not_lm_1_prev}, which now becomes
\begin{align} \label{eq:ls2t_recursive}
    \langle \ell_m, \Phi_m(\bx_1, \dots, \bx_l) \rangle = \langle \ell_m, \Phi_m(\bx_1, \dots, \bx_{l-1} \rangle + \langle \ell_{m-1}, \Phi_{m-1}(\bx_1, \dots, \bx_{l-1})\rangle \langle \bz_m, \bx_{l} \rangle,
\end{align}
and hence, due to the recursion across $m$ for both $\ell_m$ and $\Phi_m$, it is now linear in the maximal value of $m$, denoted by $M \in \NN$. This results in a less flexible, but more efficient LS2T, due to the additional added recursivity constraint on the rank-$1$ elements. We refer to this version as the \textit{recursive variant}, while to the non-recursive construction as the \textit{independent variant}.

Next, we show how the previous computations can be rewritten as a simple RNN-like discrete dynamical system. For simplicity, we consider the recursive formulation, but the independent variant can also be formulated as such with a larger latent state size. Let $(\ell^j)_{j = 1, \dots, n}$ be $n \in \NN$ different rank-$1$ recursive elements, i.e. $\ell^j = (\ell^j_m)_{m \geq 0}$, $\ell^j_m = \bz^j_1 \otimes \dots \bz^j_m$ for $\bz^j_m \in \R^d$, $m \geq 0$ and $j = 1, \dots, n$. Also, denote $h_{m,i}^j := \langle \ell_m^j, \Phi_m(\bx_1, \dots, \bx_i) \rangle \in \R$, a scalar corresponding to the output of the $j$th linear functional on the $m$th tensor level for the sequence $(\bx_1, \dots, \bx_i)$. We collect all such functionals for given $m$ and $i$ into $\bh_{m, i} := (h_{m, i}^1, \dots h_{m, i}^n) \in \R^n$, i.e. $\bh_{m, i} = \tilde\Phi_{m, \tilde\theta}(\bx_1, \dots, \bx_i)$. 

Additionally, we collect all weight vectors $\bz_m^j \in \R^d$ for a given $m \in \NN$ into the matrix $\bZ_m := (\bz_m^1, \dots, \bz_m^n)^\top \in \R^{n \times d}$. Then, we may write the following vectorized version of \eqref{eq:ls2t_recursive}:
\begin{align}
    \bh_{1, i} &=     \bh_{1, i-1} + \bZ_1 \bx_i, \label{eq:ls2t_recursive_vectorized_lv1} \\
    \bh_{m, i} &= \bh_{m, i-1} + \bh_{m-1, i-1} \odot \bZ_m \bx_i \quad \text{for } m \geq 2,
\end{align}
with the initial conditions $\bh_{m, 0} = \mathbf{0} \in \R^n$ for all $m \geq 1$, and $\odot$ denoting the Hadamard product.

  \begin{algorithm}[t]
	\caption{Computing the LS2T layer with independent tensors across levels}
	\label{alg:ls2t_independent}
	\begin{algorithmic}[1]
		\STATE {\bfseries Input:} Sequences $(\bx^j)_{j=1,\dots,n_\bx} = (\bx^j_1, \dots, \bx^j_{L})_{j=1,\dots,n_\bx} \subset \Seq{\R^d}$, \\
		rank-$1$ tensors $(\ell^k)_{k=1,\dots,n_\ell} = (\bz^k_{m,1} \otimes \cdots \otimes \bz^k_{m,m})^{k=1,\dots,n_\ell}_{m=1,\dots,M} \subset T(\R^d)$, LS2T order $M \in \NN$ 
		\STATE Compute $A[m, i, j, l, k] \gets \langle \bz^j_{m, k}, \bx^i_{l} \rangle$ for $m \in \{1,\dots,M\}$, $i \in \{1,\dots,n_\bx\}$, $j \in \{1,\dots,n_\ell\}$, $l \in \{1, \dots, L\}$ and $k \in \{1, \dots, m\}$ \\
		\FOR{$m=1$ {\bfseries to} $M$}
		\STATE Assign $R \gets A[m, :, :, :, 1]$
		\FOR{$k=2$ {\bfseries to} $m$}
		\STATE Iterate $R \gets A[m, :, :, :, k] \odot R[:, :, \boxplus+1]$ \label{algline:alg1_cumsum1}
		\ENDFOR
		\STATE Save $Y_m \gets \cdot R[:, :, \boxplus]$ \label{algline:alg1_cumsum2}
		\ENDFOR
		\STATE {\bfseries Output:} Sequences $(Y_1, \dots, Y_M)$ each of shape $(n_\bx \times L \times n_\ell)$
	\end{algorithmic}
\end{algorithm}

\begin{algorithm}[t]
	\caption{Computing the LS2T layer with recursive tensors across levels}
	\label{alg:ls2t_recursive}
	\begin{algorithmic}[1]
		\STATE {\bfseries Input:} Sequences $(\bx^j)_{j=1,\dots,n_\bx} = (\bx^j_1, \dots, \bx^j_{L})_{j=1,\dots,n_\bx} \subset \Seq{\R^d}$, \\
		rank-$1$ tensors $(\ell^k)_{k=1,\dots,n_\ell} = (\bz^k_{1} \otimes \cdots \otimes \bz^k_{m})^{k=1,\dots,n_\ell}_{m=1,\dots,M} \subset T(\R^d)$, LS2T order $M \in \NN$ 
		\STATE Compute $A[m, i, j, l] \gets \langle \bz^j_{m}, \bx^i_{l} \rangle$ for $m \in \{1,\dots,M\}$, $i \in \{1,\dots,n_\bx\}$, $j \in \{1,\dots,n_\ell\}$ and $l \in \{1, \dots, L\}$ \\
		\STATE Assign $R \gets A[1, :, :, :]$
		\STATE Save $Y_1 \gets R[:, :, \boxplus]$ \label{algline:alg2_cumsum1}
		\FOR{$m=2$ {\bfseries to} $M$}
		\STATE Update $R \gets A[m, :, :, :] \odot R[:, :, \boxplus+1]$ \label{algline:alg2_cumsum2}
		\STATE Save $Y_m \gets R[:, :, \boxplus]$
		\ENDFOR
		\STATE {\bfseries Output:} Sequences $(Y_1, \dots, Y_M)$ each of shape $(n_\bx \times L \times n_\ell)$
	\end{algorithmic}
\end{algorithm}

\subsection{Algorithms} \label{app:algs}

We have shown previously that one may compute $\Phi_\theta(\bx_1, \dots, \bx_i) = (\Phi_{m, \theta}(\bx_1, \dots, \bx_i))_{m \geq 0}$ recursively in a vectorized way for a given sequence $(\bx_1, \dots, \bx_i) \in \Seq{\R^d}$. Now, in Algorithms \ref{alg:ls2t_independent} and \ref{alg:ls2t_recursive}, we additionally show how to further vectorize the previous computations across time and the batch. For this purpose, let $(\bx^j)_{j=1, \dots, n_\bX} \subset \Seq{\R^d}$ be $n_\bX \in \NN$ sequences in $\R^d$ and $(\ell^k)_{k=1,\dots, n_\ell} \subset T(\R^d)$ be $n_\ell$ be rank-$1$ tensors in $T(\R^d)$.

Additionally, we adopted the notation for describing algorithms from \citet{kiraly2016kernels}. For arrays, $1$-based indexing is used. Let $A$ and $B$ be $k$-dimensional arrays with size $(n_1 \times \dots \times n_k)$, and let $i_j \in \{1, \dots, n_j\}$ for $j \in \{1, \dots, k\}$. Then, the following operations are defined:
  \begin{enumerate}[label=(\roman*)]
  	\item  The cumulative sum along axis $j$:
  	\begin{align*}
  		A[\dots, :, \boxplus, :, \dots][\dots, i_{j-1}, i_j,, i_{j+1} \dots] := \sum_{\kappa=1}^{i_j} A[\dots, i_{j-1}, \kappa, i_{j+1}, \dots].
  	\end{align*}
  	\item The slice-wise sum along axis $j$:
  	\begin{align*}
	  	A[\dots, :, \Sigma, :, \dots][\dots, i_{j-1}, i_{j+1}, \dots] := \sum_{\kappa=1}^{n_j} A[\dots, i_{j-1}, \kappa, i_{j+1}, \dots].
  	\end{align*}
  	\item The shift along axis $j$ by $+m$ for $m \in \NN$:
  	\begin{align*}
	  	A[\dots, :, +m, :, \dots][\dots, i_{j-1}, i_j, i_{j+1}, \dots] := \left\lbrace\begin{array}{ll} A[\dots, i_{j-1}, i_j-m, i_{j+1}, \dots], & \text{ if } i_j > m, \\ 0, & \text{ if } i_j \leq m.  \end{array}\right.
  	\end{align*}
  	\item The Hadamard product of arrays $A$ and $B$:
  	\[(A \odot B) [i_1, \dots, i_k] := A[i_1, \dots, i_k] \cdot B[i_1, \dots, i_k]. \] 
  \end{enumerate}

 \begin{table}[t]
    \centering
    \begin{small}
    \caption{Forward pass computation time in seconds on a Gefore 2080Ti GPU for varying sequence length $L$, fixed batch size $N=32$, state-space dimension $d=64$ and output dimension $h=64$.}
    \label{table:fwd_pass}
    \vspace{10pt}
    \resizebox{\textwidth}{!}{
    \begin{tabular}{ccccccccc}
        \toprule
        \multirow{2}{*}{$L$} & \multirow{2}{*}{Conv1D} & \multirow{2}{*}{LSTM} & \multicolumn{3}{c}{LS2T} & \multicolumn{3}{c}{LS2T-R}  \\
        \cmidrule{4-9}
        & & & $M=2$ & $M=6$ & $M=10$ & $M=2$ & $M=6$ & $M=10$ \\
        \midrule
        $32$ & $8.1 \times 10^{-4}$ & $1.2 \times 10^{-1}$ & $1.7 \times 10^{-3}$ & $4.5\times 10^{-3}$ & $9.9\times 10^{-3}$ & $1.7 \times 10^{-3}$ & $2.5 \times 10^{-3}$ & $3.4 \times 10^{-3}$ \\
        $64$ & $8.5 \times 10^{-4}$ & $2.3 \times 10^{-1}$ & $1.8 \times 10^{-3}$ & $4.5 \times 10^{-3}$ & $9.9 \times 10^{-3}$ & $1.8 \times 10^{-3}$ & $2.6 \times 10^{-3}$ & $3.4 \times 10^{-3}$ \\
        $128$ & $9.7 \times 10^{-4}$ & $4.6 \times 10^{-1}$ & $2.1 \times 10^{-3}$ & $4.9 \times 10^{-3}$ & $1.0 \times 10^{-2}$ & $2.1 \times 10^{-3}$ & $2.9 \times 10^{-3}$ & $3.8 \times 10^{-3}$ \\
        $256$ & $1.1 \times 10^{-3}$ & $9.3 \times 10^{-1}$ & $2.4 \times 10^{-3}$ & $5.2 \times 10^{-3}$ & $1.1 \times 10^{-2}$ & $2.4 \times 10^{-3}$ & $3.2 \times 10^{-3}$ & $4.0 \times 10^{-3}$ \\
        $512$ & $1.3 \times 10^{-3}$ & $1.8 \times 10^0$ & $3.2 \times 10^{-3}$ & $6.0 \times 10^{-3}$ & $1.1 \times 10^{-2}$ & $3.0 \times 10^{-3}$ & $4.0 \times 10^{-3}$ & $4.8 \times 10^{-3}$ \\
        $1024$ & $1.9 \times 10^{-3}$ & $3.7 \times 10^0$ & $4.4 \times 10^{-3}$ & $7.1 \times 10^{-3}$ & $1.2 \times 10^{-2}$ & $4.4 \times 10^{-3}$ & $5.1 \times 10^{-3}$ & $6.0 \times 10^{-3}$ \\
        \midrule
    \end{tabular}}
    \end{small}
\end{table}

\subsection{Complexity analysis} \label{app:complexity}
We give a complexity analysis of Algorithms \ref{alg:ls2t_independent} and \ref{alg:ls2t_recursive}. Inspection of Algorithm \ref{alg:ls2t_independent} says that it has $O(M^2 \cdot n_\bx \cdot L \cdot n_\ell \cdot d)$ complexity in both time and memory with an additional memory cost of storing the $O(M^2 \cdot n_\ell \cdot d)$ number of parameters, the rank-$1$ elements $(\ell^k_m)_m$, which are stored in terms of their components $\bz_{m, j}^k \in \R^d$. In contrast, Algorithm \ref{alg:ls2t_independent} has  a time and memory cost of $O(M \cdot n_\bx \cdot L \cdot n_\ell \cdot d)$, thus linear in $M$, and the recursive rank-$1$ elements are now only an additional $O(M \cdot n_\ell \cdot d)$ number of parameters.

Additionally to the big-O bounds on complexities, another important question is how well the computations can be parallelized, which can have a larger impact on computations when e.g. running on GPUs. Observing the algorithms, we can see that they are not \emph{completely} parallelizable due to the cumsum ($\boxplus$) operations in Lines \ref{algline:alg1_cumsum1}, \ref{algline:alg1_cumsum2} (Algorithm \ref{alg:ls2t_independent}) and Lines \ref{algline:alg2_cumsum1}, \ref{algline:alg2_cumsum2} (Algorithm \ref{alg:ls2t_recursive}). The cumulative sum operates recursively on the whole time axis, therefore it is not parallelizable, but can be computed very efficiently on modern architectures.

To gain further intuition about what kind of performance one can expect for our LS2T layers, we benchmarked the computation time of a forward pass for varying sequence lengths and varying hyperparameters of the model. For comparison, we ran the same experiment with an LSTM layer and a Conv1D layer with a filter size of $32$. The input is a batch of sequences of shape $(n_\bX \times L \times d)$, while the output has shape $(n_\bX \times L \times h)$, where $d \in \NN$ is the state-space dimension of the input sequences, while $h \in \NN$ is simply the number of channels or hidden units in the layer. For our layers, we used our own implementation in Tensorflow, while for LSTM and Conv1D, we used the Keras implementation using the Tensorflow backend. 

In Table \ref{table:fwd_pass}, we report the average computation time of a forward pass over $100$ trials, for fixed batch size $n_\bX = 32$, state-space dimension $d = 64$, output dimension $h = 64$ and varying sequence lengths $L \in \{32, 64, 128, 256, 512, 1024\}$. LS2T and LS2T-R respectively refer to the independent and recursive variants, and $M \in \NN$ denotes the truncation degree. We can observe that while the LSTM practically scales linearly in $L$, the scaling of LS2T is sublinear for all practical purposes, exhibiting a growth rate that is more close to that of the Conv1D layer, that is fully parallelizable. Specifically, while the LSTM takes $3.7$ seconds to make a forward pass for $L=1024$, all variants of the LS2T layer take less time than that \emph{by a factor of at least a 100}. This suggests that its computations are \emph{highly parallelizable} across time. Additionally, we observe that LS2T exhibits a more aggressive growth rate with respect to the parameter $M$ due to the quadratic complexity in $M$ (although the numbers show only linear growth), while LS2T-R scales very favourably in $M$ as well due to the linear complexity (the results indicate a sublinear growth rate).

\subsection{Initialization}\label{app:init}
Below we detail the initialization procedure used by our models for the parameters $\theta$ of the LS2T layer, where $\theta = (\ell^k)_{k=1}^{n_\ell} \subset T(\R^d)$. As before, each $\ell^k = (\ell^k_{m})_{m \geq 0} \in T(\R^d)$ is given as a sequence of rank-$1$ tensors, such that $\ell^k_m = \bz^k_{m, 1} \otimes \cdots \otimes \bz^k_{m,m}$ with $\bz^k_{m, j} \in \R^d$ for the independent variant, while $\ell^k_m = \bz_1^k \otimes \cdots \otimes \bz_m^k$ with $\bz_m \in \R^d$ for the recursive variant. Hence, by initialization, we mean the initialization of the components $\bz^k_{m, j}$ or $\bz^k_m$.

To find a satisfactory initialization scheme, we took as starting point the Glorot \citep{glorot2010understanding} initialization, which specifies that for a layer with input dimension $n_{in}$ and output dimension $n_{out}$, the weights should be independently drawn from a centered distribution with variance $2 / (n_{in} + n_{out})$, where the distribution that is used is usually a uniform or a Gaussian.

\paragraph{Independent variant.} We first consider the independent variant of the algorithm. The weights are given as the rank-$1$ tensors 
\begin{align}
    \ell^k_m = \bz_{m, 1}^k \otimes \cdots \otimes \bz_{m, m}^k \in (\R^d)^{\otimes m} \quad\text{for}\quad k = 1, \dots, n_\ell\quad\text{and}\quad m \geq 0.
\end{align}

Denote $\bz^k_{m, j} = (z^k_{m, j, 1}, \dots, z^k_{m, j, d}) \in \R^d$, and assume that for a given $m \in \NN$ that each of $z^k_{m, j, p}$ are drawn independently from some distribution with
\begin{align}
    \EE[z^k_{m, j, p}] = 0 \quad\text{and}\quad \EE[z^k_{m, j, p}]^2 = \sigma_m^2 \quad\text{for}\quad j=1, \dots, m \quad\text{and}\quad p=1, \dots, d.
\end{align}
Then, for a given multi-index $\mathbf{i} = (i_1, \dots, i_m) \in \{1, \dots, d\}^m$, the $\mathbf{i}\text{th}$ coordinate of $\ell^k_m$ is given as
\begin{align}
    \ell^k_{m, \mathbf{i}} = z^k_{m, 1, i_1} \cdots z^k_{m, m, i_m}    
\end{align}
and has as its first two moments
\begin{align}
    \EE[\ell^k_{m, \mathbf{i}}] = 0\quad\text{and}\quad\EE[\ell^k_{m, \mathbf{i}}]^2 = \sigma_m^{2m}
\end{align}
due to the independence of the corresponding terms in the product. Therefore, to have $\EE[\ell^k_{m, \mathbf{i}}]^2 = 2 / (d^m + n_\ell)$, where we made the substitutions $n_{in} = d^m$ and $n_{out} = n_\ell$, we can simply set
\begin{align}
    \sigma_m^2 = \sqrt[m]{\frac{2}{d^m + n_\ell}}.    
\end{align}

\paragraph{Recursive variant.} In the recursive variant, the weights themselves are constructed recursively as
\begin{align}
    \ell^k_m = \bz^k_1 \otimes \cdots \otimes \bz^k_m \quad\text{for}\quad k=1,\dots,n_\ell \quad\text{and}\quad m \geq 0.
\end{align}
Thus, for $\mathbf{i} = (i_1, \dots, i_m) \in \{1, \dots, d\}^m$, the $\mathbf{i}$th component of $\ell^k_m$ is given as
\begin{align}
    \ell^k_{m, \mathbf{i}} = z^k_{1, i_1} \cdots z^k_{m, i_m},
\end{align}
and if we assume that for a given $m \in \NN$, $z^k_{m, i_m}$ is drawn from a centered distribution with variance $\sigma_m^2$, then we have
\begin{align}
    \EE[\ell^k_{m, \mathbf{i}}] = 0 \quad\text{and}\quad \EE[\ell^k_{m, \mathbf{i}}]^2 = \sigma_1^2 \cdots \sigma_m^2,
\end{align}
which means that now our goal is to have $\sigma_1^2 \cdots \sigma_m^2 = 2 / (d^m + n_\ell)$ for all $m \geq 1$, which is achievable inductively by
\begin{align}
    \sigma_1^2 = \frac{2}{d + n_\ell} \quad\text{and}\quad \sigma_{m+1}^2 = \frac{d^m + n_\ell}{d^{m+1} + n_\ell} \quad\text{for}\quad {m \geq 1}. 
\end{align}
For both variants of the algorithms, we used the above described initialization schemes, where the tensor component $\bz$'s were drawn from a centered uniform or a Gaussian distribution with the specified variances. Although the resulting weight tensors $\ell^k_m$ were of neither distribution, they had the pre-specified first two moments, that seemed sufficient to successfully train models with LS2T layers when combined with succeeding normalization layers as described in Appendix \ref{app:variations}.

However, we remark that when \citet{glorot2010understanding} derived their weight initialization scheme, they considered a fully linear regime across layers, while for $\bx = (\bx_1, \dots, \bx_L) \in \Seq{\R^d}$, the features $\Phi_m(\bx_1, \dots \bx_L)$ on which the weight tensors $\ell^k_m$ act will be highly nonlinear for any given $m \geq 2$ with increasing levels of nonlinearity for larger values of $m$. Therefore, it is highly likely that better initialization schemes can be derived by studying the distribution of $\Phi_m(\bx)$, that might also make it possible to abandon the normalization layers succeeding the LS2T layers and still retain the layers' ability to learn relevant features of the data.
Alternatively, data dependent initializations could also prove useful here, such as the LSUV initialization \citep{mishkin2015all}.

\section{Details on experiments} \label{app:details}
In the following, we give details on the time series classification (Appendix \ref{app:tsc}), mortality prediction (Appendix \ref{app:mortality}) and sequential data imputation (Appendix \ref{app:imputation}) experiments. For running all experiments, we used GPU-based computations on a set of computing nodes, that were equipped with 11 NVIDIA GPUs in total: 4 Tesla K40Ms, 5 Geforce 2080 TIs and 2 Quadro GP100 GPUs. The benchmarks and code used to run the experiments using Tensorflow as backend are available at \url{https://github.com/tgcsaba/seq2tens}.

\subsection{Time series classification} \label{app:tsc}
\paragraph{Problem formulation.} Classification is a traditional task in discriminative supervised machine learning: let $\cX$ be the data space and $\cY = \{1, \dots, c\}$ the discrete output space that consists of only categorical values with $c \in \NN$ the total number of classes. The problem is then to predict the corresponding labels of a set of unlabelled examples $\bX^\star = (\bx^\star_i)_{i=1}^{n_{\bX^\star}}$ given a set of labelled examples $(\bX, \by) = (\bx_i, y_i)_{i=1}^{n_\bX} \subseteq \cX \times \cY$. In the context of time series classification (TSC), the data space $\cX = \Seq{\R^d}$ is the space of multivariate sequences, i.e. $\bx_i = (\bx_{i,j})_{j=1}^{l_{\bx_i}}$ where $l_{\bx_i} \in \NN$ is the length of the sequence $\bx_i$ that can change from instance from instance.

\begin{wraptable}{r}{8.5cm}
    \vspace{-20pt}
	\caption{Specification of datasets used for benchmarking}
	\label{table:dataset_spec}
	\vspace{-10pt}
    % \vskip 0.15in
    \begin{center}
    \begin{small}
	\begin{sc}
    \begin{tabular}{lrrrrrr}
    \toprule
    Dataset  & $n_c$ & $d$ & $L_\bx$ & $n_\bX$ & $n_{\bX_\star}$ \\ %[-5pt]
    \midrule
        Arabic Digits & 10 & 13 & \numrange[range-phrase = --]{4}{93} & 6600 & 2200\\
        AUSLAN & 95 & 22 & \numrange[range-phrase = --]{45}{136} & 1140 & 1425\\
        Char.~Traj.& 20 & 3 & \numrange[range-phrase = --]{109}{205} & 300 & 2558\\
        CMUsubject16 & 2 & 62 & \numrange[range-phrase = --]{127}{580} & 29 & 29\\
        DigitShapes & 4 & 2 & \numrange[range-phrase = --]{30}{98} & 24 & 16\\
        ECG & 2 & 2 & \numrange[range-phrase = --]{39}{152} & 100 & 100\\
        Jap.~Vowels & 9 & 12 & \numrange[range-phrase = --]{7}{29} & 270 & 370\\
        Kick vs Punch & 2 & 62 & \numrange[range-phrase = --]{274}{841} & 16 & 10\\
        LIBRAS & 15 & 2 & 45 & 180 & 180\\
        NetFlow & 2 & 4 & \numrange[range-phrase = --]{50}{997} & 803 & 534\\
        PEMS & 7 & 963 & 144 & 267 & 173\\
        PenDigits & 10 & 2 & 8 & 300 & 10692\\
        Shapes & 3 & 2 & \numrange[range-phrase = --]{52}{98} & 18 & 12\\
        UWave & 8 & 3 & 315 & 896 & 3582\\
        Wafer & 2 & 6 & \numrange[range-phrase = --]{104}{198} & 298 & 896\\
        Walk vs Run & 2 & 62 & \numrange[range-phrase = --]{128}{1918} & 28 & 16\\ %[-5pt]
    \bottomrule
    \end{tabular}
	\end{sc}
    \end{small}
    \end{center}
    % \vskip -1.0in
    \vspace{-15pt}
\end{wraptable}

\paragraph{Datasets.}
Table \ref{table:dataset_spec} details the datasets from \citet{baydogan2015multivarate} that were used for the TSC experiment. The columns are defined as follows: $n_c$ denotes the number of classes, $d$ the dimension of the state space, $L_\bx$ the range of sequence lengths, $n_\bX$ and $n_{\bX_\star}$ respectively denote the number of examples in the prespecified training and testing sets. As preprocessing, the state space dimensions were normalized to zero mean and unit variance. From the experiment, we excluded the datasets {\sc LP1, LP2, \dots LP5}, because all of these contain a very low number of training examples ($n_\bX < 50$ for 4 out of 5), and also a low signal-to-noise ratio (around 60\%-80\% accuracy achieved by non-DL, classic time series classifiers), which is arguably not the setting when deep learning becomes particularly relevant.

\paragraph{Baselines.} The benefit of the multivariate TSC archive \citep{baydogan2015multivarate} is that there exist several publications which report the test set results of their respective TSC models, which makes it possible to directly compare against them. We included all results among the comparison that we are aware of: DTW\textsubscript{i} \citep{Sakoe1978DTW}, ARKernel \citep{Cuturi2011AR}, SMTS \citep{baydogan2015learning}, LPS \citep{Baydogan2015TimeSR}, gRSF \citep{karlsson2016generalized}, mvARF \citep{tuncel2018autoregressive}, MUSE \citep{Schfer2017MUSE}, MLSTMFCN \citep{Karim2019LSTMFCN}. Additionally, a recent survey paper \citep{Fawaz2019DLforTSC} benchmarked a range of DL models for TSC, however, they only considered a subset of these multivariate datasets. Therefore, so as to have results across the whole archive, we borrow the two strongest models as baselines, FCN and ResNet, and train them across the whole archive. In fact, we also chose the FCN as the base model to upgrade with LS2T layers, specifically for its strong performance and simplicity, since FCN is a vanilla CNN model consisting of three temporal convolution layers with kernel sizes $(8, 5, 3)$ and filters $(128, 256, 128)$, where each layer is succeeded by batch normalization and \texttt{relu} activation. We denote this as \FCN{128}, while \FCN{h} refers to an FCN with filters $(h, 2h, h)$.
 The ResNet is a more complicated, residual network \citep{he2016deep} consisting of three FCN blocks of widths $(64, 128, 128)$ with skip-connections in-between, where the width of the convolutional layers in each FCN block are now uniform (hence, the middle convolutional layer also has $h$ filters rather than $2h$). For more details, refer to \citet{Fawaz2019DLforTSC,Wang2017TSCfromScratch}. Also note that for MLSTMFCN the same results are reported as the ones in \citet{Schfer2017MUSE}.

\begin{figure}[t]
	% \centering
	% \begin{center}
	\vspace{-20pt}
	\begin{center}
	\beginpgfgraphicnamed{fcn-ls2t-plot}
	\begin{tikzpicture}[scale=0.8, every node/.style={scale=0.8}, shorten >=1pt,draw=black!50, node distance=\layersep]
		
	\node[seq, draw=purple, fill=lightpurple] at (0.0, 0.0) (inp) {\sc Input};
	% \draw[<->, color=black, thin] (-0.5, -1.75) -- (-0.5, 1.75);
	\node at (-0.5, 0) (l) {$\scriptstyle \ell$};
	% \draw[<->, color=black, thin] (-0.35, -1.9) -- (0.35, -1.9);
	\node at (0.0, -2.0) (l) {$\scriptstyle d$};
	
	\node[seq, draw=blue, fill=lightblue] at (1.0, 0.0) (conv1) {\sc Time + Conv (8)};
	\draw[->, color=black, thin] (inp) -- (conv1);
	\node at (1.0, -2.0) (l) {$\scriptstyle h$};
	
	\node[seq, draw=blue, fill=lightblue] at (2.0, 0.0) (bn1) {\sc BN + Act (ReLU)};
	\draw[->, color=black, thin] (conv1) -- (bn1);
	\node at (2.0, -2.0) (l) {$\scriptstyle h$};
	
	\node[wideseq, draw=blue, fill=lightblue] at (3.125, 0.0) (conv2) {\sc Time + Conv (5)};
	\draw[->, color=black, thin] (bn1) -- (conv2);
	\node at (3.125, -2.0) (l) {$\scriptstyle 2h$};
	
	\node[wideseq, draw=blue, fill=lightblue] at (4.375, 0.0) (bn2) {\sc BN + Act (ReLU)};
	\draw[->, color=black, thin] (conv2) -- (bn2);
	\node at (4.375, -2.0) (l) {$\scriptstyle 2h$};
	
	\node[seq, draw=blue, fill=lightblue] at (5.5, 0.0) (conv3) {\sc Time + Conv (3)};
	\draw[->, color=black, thin] (bn2) -- (conv3);
	\node at (5.5, -2.0) (l) {$\scriptstyle h$};
	
	\node[seq, draw=blue, fill=lightblue] at (6.5, 0.0) (bn3) {\sc BN + Act (ReLU)};
	\draw[->, color=black, thin] (conv3) -- (bn3);
	\node at (6.5, -2.0) (l) {$\scriptstyle h$};
	
	% \node[cross=2pt, draw=darkgrey] at (6.0, 2.5) (cross1) {}; 
	
	\node[seq, draw=green, fill=lightgreen] at (7.5, 0.0) (time1) {\sc Time + Diff};
	\draw[->, color=black, thin] (bn3) -- (time1);
	\node at (7.5, -2.0) (l) {$\scriptstyle h+1$};
	
	\node[circ2, draw=darkgrey, fill=grey] at (7.5, 2.10) (circ1) {$+$};
	\draw[->, color=black, thin] (inp.east) |- ++(0.0, 1.25) -| (circ1.north);
	\node at (3.75, 3.4) (l) {\sc Input shortcut};
	
	\node[seq, draw=green, fill=lightgreen] at (8.4, 0.1) (ls2t1_behind) {};
	\node[seq, draw=green, fill=lightgreen] at (8.5, 0.0) (ls2t1) {\sc LS2T};
	\draw[->, color=black, thin] (time1) -- (ls2t1);
	\node at (8.5, -2.0) (l) {$\scriptstyle n \times m$};
	
	\node[wideseq, draw=green, fill=lightgreen] at (9.625, 0.0) (wide1) {\sc BN + Reshape};
	\draw[->, color=black, thin] (ls2t1) -- (wide1);
	\node at (9.625, -2.0) (l) {$\scriptstyle n m$};
	
	% \node[wideseq, draw=green, fill=lightgreen] at (10.875, 0.0) (wide2) {\sc Time + Diff};
	% \draw[->, color=black, thin] (wide1) -- (wide2);
	% \node at (10.875, -2.0) (l) {$\scriptstyle n m + 1$};
	
	% \node[seq, draw=green, fill=lightgreen] at (11.9, 0.1) (ls2t2_behind) {};
	% \node[seq, draw=green, fill=lightgreen] at (12.0, 0.0) (ls2t2) {\sc LS2T};
	% \draw[->, color=black, thin] (wide2) -- (ls2t2);
	% \node at (12.0, -2.0) (l) {$\scriptstyle n \times m$};
	
	% \node[wideseq, draw=green, fill=lightgreen] at (13.125, 0.0) (wide3) {\sc BN + Reshape};
	% \draw[->, color=black, thin] (ls2t2) -- (wide3);
	% \node at (13.125, -2.0) (l) {$\scriptstyle n m$};
	
	\node[outer sep=2pt, inner sep=0pt] at (10.75, 0.0) (dots) {\dots};
	\draw[->, color=black, thin] (wide1) -- (dots);
	
	\node[wideseq, draw=green, fill=lightgreen] at (11.875, 0.0) (wide4) {\sc Time + Diff};
	\draw[->, color=black, thin] (dots) -- (wide4);
	\node at (11.875, -2.0) (l) {$\scriptstyle n m + 1$};
	
	\node[rectangle, draw=green, fill=lightgreen, text centered, minimum height=0.5cm, minimum width=1.25cm, outer sep=2pt] at (13.275, 0.1) (ls2t3_behind) {};
	\node[rectangle, draw=green, fill=lightgreen, text centered, minimum height=0.5cm, minimum width=1.25cm, outer sep=2pt] at (13.375, 0.0) (ls2t3) {\sc LS2T};
	\draw[->, color=black, thin] (wide4) -- (ls2t3);
	
	\node at (13.375, -0.5) (l) {$\scriptstyle n \times m$};
	
	\node[rectangle, draw=green, fill=lightgreen, text centered, minimum height=0.5cm, minimum width=2.0cm, outer sep=2pt] at (15.5, 0.0) (wide4) {\sc BN + Rshp};
	\draw[->, color=black, thin] (ls2t3) -- (wide4);
	\node at (15.5, -0.5) (l) {$\scriptstyle n m$};
	\node at (16.75, 0.0) (l) {$\scriptstyle 1$};
	
	\node at (11.125, 3.15) (l) {\sc FCN shortcut};
	
	\node[circ2, draw=darkgrey, fill=grey] at (15.5, 0.6) (circ2) {$+$};
	\draw[->, color=black, thin] (bn3.east) |- ++(0.0, 1.0) -| (circ2.north);
	
	\draw [black, thin] (0.75,-2.2) to [square right brace ] (6.75,-2.2);
	\node at (3.75, -2.7) (l) {\sc FCN block};

	\draw [black, thin] (7.25, -2.2) to [square right brace ] (16.5,-2.2);
	\node at (12.0, -2.7) (l) {\sc Deep LS2T block};
	
	\end{tikzpicture}
	\endpgfgraphicnamed
	\vspace{-5pt}
	\caption{Depiction of the models used for time series classification. LS2T\textsuperscript{3} only consists of a deep LS2T block (yellow), while $\text{FCN-LS2T}^3$ also precedes it with an FCN block (green).}
	\label{fig:ls2t_architectures}
	\end{center}
	\vspace{-15pt}
\end{figure}
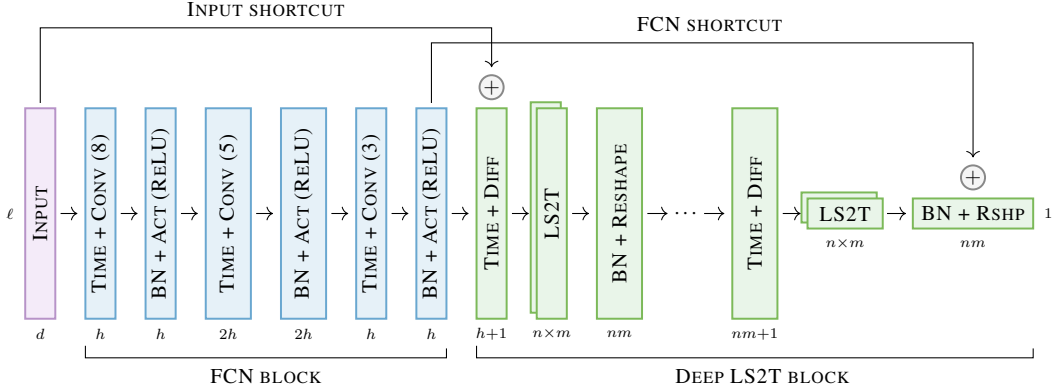

\paragraph{Architecture.} The structure of a generic \FCNLStwoTwidth{h}{n}{d} model of FCN width-$h$, LS2T width-$n$, LS2T-depth $d$ and LS2T-order $m$ is visualized on Figure \ref{fig:ls2t_architectures}, where the FCN block is additionally augmented with a time-embedding preceding each convolution compared to the vanilla FCN, while the deep LS2T block uses both time-embedding and difference layers before each LS2T layer. The usefulness of these is discussed in Appendix \ref{app:variations}. There is a shortcut connection coming from the {\sc Input} layer, that is added to the input of the first {\sc Time + Diff} layer in the LS2T block, allowing the first LS2T layer to access the input data additionally to the FCN output. Also, there is another shortcut connection coming from the output of the FCN block and added to the output of the LS2T block, which allows the FCN to directly affect the classification performance, hence, allowing the LS2T block to focus on learning global temporal interactions with the localized interactions coming from the FCN block. Both skip-connections use a time-distributed linear projection layer to match the dimension of the residual branch, and the shortcut from the FCN output also uses a GAP layer to pool over the time axis before being added to the final output that the classification layer receives.

% \paragraph{Implementation.}
% Our two models, LS2T\textsuperscript{3} and $\text{FCN-LS2T}^3$, are structured as depicted on Figure \ref{fig:ls2t_architectures}, where LS2T\textsuperscript{3} only consists of the stacked Seq2Tens block, while $\text{FCN-LS2T}^3$ precedes it with an FCN block. In the FCN block, we also precede each convolutional layer by a time-embedding layer and succeed it by a batch normalization layer and a \texttt{relu} activation. In the LS2T\textsuperscript{3} block, each LS2T layer is preceded by time embedding and differencing and succeeded by a flattening\footnote{By flattening we mean merging all axes except the batch and time axes.} and a batch normalization layer, as motivated in Appendix \ref{app:variations}. The LS2T layers in both LS2T\textsuperscript{3} and $\text{FCN-LS2T}^3$ use the recursive formulation of the low-rank LS2T layers, as detailed in Appendix \ref{app:recursive}. We implemented all models in Keras using the Tensorflow blackend, while for the LS2T layers we used our own Tensorflow implementation. % Code is available at \texttt{GITHUBAUTHOR}.

\begin{wraptable}{r}{7.0cm}
    \vspace{-10pt}
    \centering
    \caption{Number of trainable parameters}
    \label{table:params}
    \vspace{-5pt}
    \begin{small}
	\begin{sc}
    \begin{tabular}{lcc}
        \toprule
        \multirow{2}{*}{Model} & \multicolumn{2}{c}{Trainable parameters} \\
        \cmidrule{2-3}
        &  Median & Med.~abs.~dev. \\
        \midrule
            \LStwoTwidth{64}{3} & $3.5 \times 10^{4}$ & $1.5 \times 10^3$ \\
            \FCNLStwoTwidth{64}{64}{3} & $1.2 \times 10^{5}$ & $2.4 \times 10^3$ \\
            \FCN{128} & $2.7 \times 10^5$ & $3.5 \times 10^3$ \\
            \FCNLStwoTwidth{128}{64}{3} & $3.4 \times 10^{5}$ & $3.7 \times 10^3$ \\
        	ResNet & $5.2 \times 10^5$ & $2.4 \times 10^3$ \\
            \bottomrule
    \end{tabular}
	\end{sc}
    \end{small}
    \vspace{-15pt}
\end{wraptable}

\paragraph{Parameter comparison.} Table \ref{table:params} depicts the median number of trainable parameters for the models considered by us in this experiment and their median absolute deviation. While the smallest model, \LStwoTwidth{64}{3}, has about the third of the parameters as \FCNLStwoTwidth{64}{64}{3} due to the added \FCN{64} block in the latter, \FCNLStwoTwidth{64}{64}{3} still has about half as many parameters as \FCN{128} as it is a much thinner network. On the other hand, \FCNLStwoTwidth{128}{64}{3} uses an \FCN{128} block with a \LStwoTwidth{64}{3} block on top and additional skip-connections, so it is not surprising that its number of parameters are between \FCN{128} and ResNet, with ResNet being the largest model due to it being a residual network of three FCN blocks of various sizes. At the same time, parameter counting might not be a good proxy for measuring the size of deep learning models generally \citep{maddox2020rethinking}, and even more so when the layer types constituting the different models also vary additionally to the number of parameters.

\paragraph{Training details.} For the training of all models, an {\sc Adam} optimizer \citep{Kingma2015AdamAM} was used with an initial learning rate of $\alpha = 1 \times 10^{-3}$, and we employed a learning rate decay of $\beta = 1 / 2$ after $100$ epochs of no improvement in the training loss, and stopping early after no improvement over $500$ epochs in the loss, after which the lowest loss parameter set was used. The maximum number of epochs for all were set to $2000$, except for ResNet it was set to $1500$, since that is what \citet{Fawaz2019DLforTSC} used. The batch size was set to $b = \max(\min(0.1\cdot n_\bX, b_{max}), b_{min})$, where for the models \LStwoTwidth{64}{3}, \FCN{128}, \FCNLStwoTwidth{64}{64}{3}, \FCNLStwoTwidth{128}{64}{3} the values were $b_{max}=16$ and $b_{min}=4$, for ResNet $b_{max}=64$ and $b_{min}=4$ were used. This is also the same setting as how \FCN{128} and ResNet were trained in \citet{Fawaz2019DLforTSC} with the exception that they did not cap the batch size at $b_{min}=4$ that for small datasets ($n_\bX < 40$) made their training unstable, which is why on some of these datasets our baselines, \FCN{128} and ResNet are stronger. We also remark that before we introduced the skip-connections in the FCN-LS2T architecture (Figure \ref{fig:ls2t_architectures}), training was considerably more unstable for this model, and at that point, using the {\sc SWATS} optimizer \citep{keskar2017improving} in place of {\sc Adam} could provide some improvements on the results; while after upgrading the architecture, changing the optimizer did not seem to make a difference.

\paragraph{Results.} The full table of results is in in Table \ref{table:classification_accuracies}, where for the models that we trained ourselves, i.e.~\FCN{128}, ResNet, \LStwoTwidth{64}{3}, \FCNLStwoTwidth{64}{64}{3}, \FCNLStwoTwidth{128}{64}{3}, we report the mean and standard deviation of test set accuracies over 5 model trains. Figure \ref{fig:box_and_cd} depicts the box-plot distributions of classification accuracies and the corresponding critical difference (CD) diagram. The CD diagram depicts the mean ranks of each method averaged over datasets with a calculated CD region using the Nemenyi test \citep{nemenyi1963distribution} for an alpha value of $\alpha=0.1$. For the Bayesian signed-rank test \citep{benavoli2014bayesian}, we used the implementation from \texttt{https://github.com/janezd/baycomp}, and the resulting posterior probabilities are compared in Table \ref{table:bayes_signed_rank_probs}, Section \ref{subseq:4_tsc}. The posterior distributions themselves are visualized on Figure \ref{fig:baycomp1} The region of practical equivalence (rope) was set to $\texttt{rope} = 1 \times 10^{-3}$, that is, two accuracies were practically equivalent if they are at most the given distance from each other. For the visualizations and computation of probabilities, the posteriors were evaluated using $n=10^5$ Monte Carlo samples.

\begin{table}[t]
	\caption{Classifier accuracies on the multivariate TSC datasets with the best and second best highlighted for each row in bold and italic, respectively.}
	\label{table:classification_accuracies}
% 	\vskip 0.15in
	\begin{center}
		\resizebox{\textwidth}{!}{
		\begin{sc}
		\begin{tabular}{lccccccccccccc}%{lllll}
			\toprule
            Dataset & ARKernel & DTW & LPS & SMTS & gRSF & mvARF & MUSE & MLSTMFCN & \FCN{128} & ResNet & \LStwoTwidth{64}{3} & \FCNLStwoTwidth{64}{64}{3} & \FCNLStwoTwidth{128}{64}{3} \\ 
			\midrule
			ArabicDigits & $0.988$ & $0.908$ & $0.971$ & $0.964$ & $0.975$ & $0.952$ & $0.992$ & $0.990$ & $0.995(0.001)$ & $0.995(0.002)$ & $0.979(0.002)$ & $\mathit{0.996}(0.001)$ & $\mathbf{0.997}(0.001)$ \\
			AUSLAN & $0.918$ & $0.727$ & $0.754$ & $0.947$ & $0.955$ & $0.934$ & $0.970$ & $0.950$ & $0.979(0.003)$ & $0.971(0.003)$ & $0.987(0.002)$ & $\mathbf{0.996}(0.001)$ & $\mathit{0.995}(0.001)$ \\
			Char.~Traj. & $0.900$ & $0.948$ & $0.965$ & $0.992$ & $\mathit{0.994}$ & $0.928$ & $0.937$ & $0.990$ & $0.992(0.001)$ & $0.985(0.002)$ & $0.980(0.003)$ & $0.993(0.001)$ & $\mathbf{0.995}(0.000)$ \\
			CMUsubject16 & $\mathbf{1.000}$ & $0.930$ & $\mathbf{1.000}$ & $\mathit{0.997}$ & $\mathbf{1.000}$ & $\mathbf{1.000}$ & $\mathbf{1.000}$ & $\mathbf{1.000}$ & $\mathbf{1.000}(0.000)$ & $\mathbf{1.000}(0.000)$ & $\mathbf{1.000}(0.000)$ & $\mathbf{1.000}(0.000)$ & $\mathbf{1.000}(0.000)$ \\
			DigitShapes & $\mathbf{1.000}$ & $\mathbf{1.000}$ & $\mathbf{1.000}$ & $\mathbf{1.000}$ & $\mathbf{1.000}$ & $\mathbf{1.000}$ & $\mathbf{1.000}$ & $\mathbf{1.000}$ & $\mathbf{1.000}(0.000)$ & $\mathbf{1.000}(0.000)$ & $\mathbf{1.000}(0.000)$ & $\mathbf{1.000}(0.000)$ & $\mathbf{1.000}(0.000)$ \\
			ECG & $0.820$ & $0.790$ & $0.820$ & $0.818$ & $0.880$ & $0.785$ & $0.880$ & $0.870$ & $0.860(0.018)$ & $0.856(0.010)$ & $0.824(0.016)$ & $\mathbf{0.892}(0.015)$ & $\mathit{0.886}(0.014)$ \\
			Jap.~Vowels & $0.984$ & $0.962$ & $0.951$ & $0.969$ & $0.800$ & $0.959$ & $0.976$ & $\mathbf{1.000}$ & $0.990(0.003)$ & $0.989(0.003)$ & $0.984(0.005)$ & $0.991(0.003)$ & $\mathit{0.994}(0.003)$ \\
			Kick vs Punch & $0.927$ & $0.600$ & $0.900$ & $0.820$ & $\mathbf{1.000}$ & $\mathit{0.976}$ & $\mathbf{1.000}$ & $0.900$ & $\mathbf{1.000}(0.000)$ & $\mathbf{1.000}(0.000)$ & $\mathbf{1.000}(0.000)$ & $\mathbf{1.000}(0.000)$ & $\mathbf{1.000}(0.000)$ \\
			LIBRAS & $0.952$ & $0.888$ & $0.903$ & $0.909$ & $0.911$ & $0.945$ & $0.894$ & $\mathbf{0.970}$ & $\mathit{0.966}(0.002)$ & $\mathit{0.966}(0.008)$ & $0.859(0.008)$ & $0.946(0.005)$ & $0.956(0.008)$ \\
			NetFlow & nan & $0.976$ & $0.968$ & $0.977$ & $0.914$ & nan & $0.961$ & $0.950$ & $0.970(0.003)$ & $0.953(0.006)$ & $0.921(0.014)$ & $0.962(0.006)$ & $0.962(0.005)$ \\
			PEMS & $0.750$ & $0.832$ & $0.844$ & $0.896$ & $1.000$ & nan & nan & nan & $0.775(0.019)$ & $0.787(0.008)$ & $0.725(0.013)$ & $0.788(0.025)$ & $0.802(0.017)$ \\
			PenDigits & $0.952$ & $0.927$ & $0.908$ & $0.917$ & $0.932$ & $0.923$ & $0.912$ & $\mathbf{0.970}$ & $\mathit{0.967}(0.002)$ & $0.963(0.001)$ & $0.956(0.002)$ & $0.963(0.003)$ & $0.962(0.002)$ \\
			Shapes & $\mathbf{1.000}$ & $\mathbf{1.000}$ & $\mathbf{1.000}$ & $\mathbf{1.000}$ & $\mathbf{1.000}$ & $\mathbf{1.000}$ & $\mathbf{1.000}$ & $\mathbf{1.000}$ & $\mathbf{1.000}(0.000)$ & $\mathbf{1.000}(0.000)$ & $\mathbf{1.000}(0.000)$ & $\mathbf{1.000}(0.000)$ & $\mathbf{1.000}(0.000)$ \\
			UWave & $0.904$ & $0.916$ & $\mathbf{0.980}$ & $0.941$ & $0.929$ & $0.952$ & $0.916$ & $0.970$ & $\mathit{0.979}(0.001)$ & $0.978(0.001)$ & $0.958(0.001)$ & $0.975(0.002)$ & $0.976(0.001)$ \\
			Wafer & $0.968$ & $0.974$ & $0.962$ & $0.965$ & $\mathit{0.992}$ & $0.931$ & $\mathbf{0.997}$ & $0.990$ & $0.987(0.005)$ & $0.989(0.002)$ & $0.983(0.003)$ & $0.988(0.001)$ & $0.990(0.001)$ \\
			Walk vs Run & $\mathbf{1.000}$ & $\mathbf{1.000}$ & $\mathbf{1.000}$ & $\mathbf{1.000}$ & $\mathbf{1.000}$ & $\mathbf{1.000}$ & $\mathbf{1.000}$ & $\mathbf{1.000}$ & $\mathbf{1.000}(0.000)$ & $\mathbf{1.000}(0.000)$ & $\mathbf{1.000}(0.000)$ & $\mathbf{1.000}(0.000)$ & $\mathbf{1.000}(0.000)$ \\
			\midrule 
			Avg.~acc. & $0.938$ & $0.899$ & $0.933$ & $0.945$ & $0.955$ & $0.949$ & $0.962$ & $\mathbf{0.970}$ & $0.966$ & $0.964$ & $0.947$ & $0.\mathit{968}$ & $\mathbf{0.970}$ \\
			Med.~acc. & $0.952$ & $0.929$ & $0.964$ & $0.964$ & $0.984$ & $0.952$ & $0.976$ & $0.990$ & $0.988$ & $0.987$ & $0.982$ & $\mathit{0.992}$ & $\mathbf{0.994}$ \\
			Sd.~acc. & $0.073$ & $0.111$ & $0.073$ & $0.059$ & $0.058$ & $0.055$ & $0.043$ & $0.039$ & $0.061$ & $0.059$ & $0.079$ & $0.056$ & $0.054$ \\
			Avg.~rank & $6.000$ & $6.812$ & $6.000$ & $5.625$ & $4.625$ & $6.714$ & $4.933$ & $3.333$ & $3.000$ & $3.500$ & $5.312$ & $\mathit{2.750}$ & $\mathbf{2.312}$ \\
			Med.~rank & $6.000$ & $8.000$ & $6.000$ & $6.500$ & $2.500$ & $8.500$ & $4.000$ & $3.000$ & $\mathit{2.500}$ & $3.000$ & $6.500$ & $\mathit{2.500}$ & $\mathbf{2.000}$ \\
			Sd.~rank & $4.071$ & $4.246$ & $4.397$ & $3.810$ & $3.964$ & $4.514$ & $4.044$ & $2.610$ & $2.066$ & $2.251$ & $3.591$ & $1.880$ & $1.493$ \\
			\bottomrule
		\end{tabular}
	\end{sc}
	}
	\end{center}
%	 \vskip -1.0in
\end{table}

\begin{figure}[]
    % \vspace{-10pt}
	\centering
	\begin{minipage}{0.49\textwidth}
		\centering
		% \hspace{10pt}
		\includegraphics[width=2.75in]{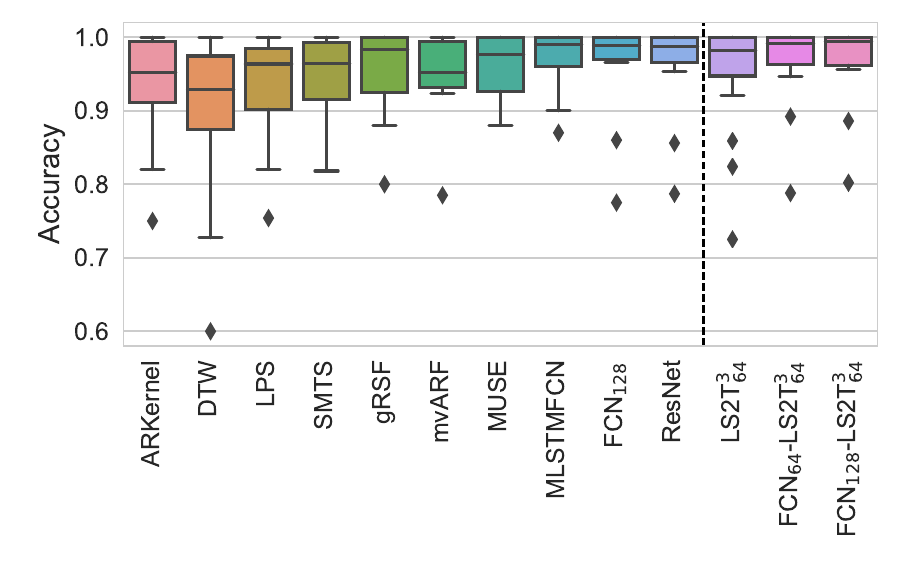}
	\end{minipage}
	\begin{minipage}{0.49\textwidth}
		\centering
		\vspace{-30pt}
		% \hspace{-10pt}
		\includegraphics[width=2.75in]{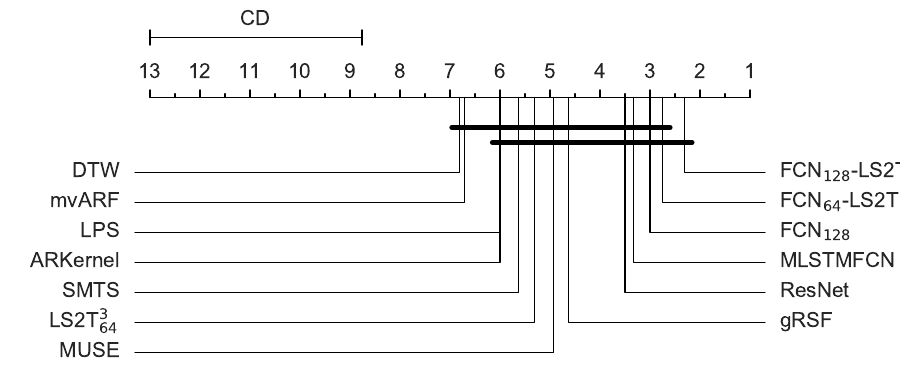}
	\end{minipage}
	\vspace{-10pt}
	\caption{Box-plot of classification accuracies (left) and critical difference diagram (right).}
	\label{fig:box_and_cd}
	% \vspace{-20pt}
\end{figure}

\newpage

\begin{figure}
	\centering
	\begin{minipage}{0.245\textwidth}
		\centering
% 		\hspace{-30pt}
		\includegraphics[width=1.1in]{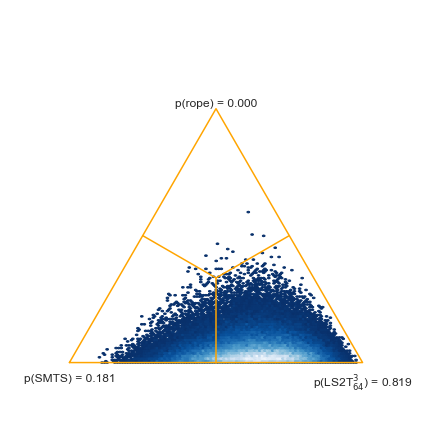}
		% \captionof{figure}{Caption for image}
		% \label{fig:sample_figure}
	\end{minipage}
	\begin{minipage}{0.245\textwidth}
		\centering
% 		\hspace{-30pt}
		\includegraphics[width=1.1in]{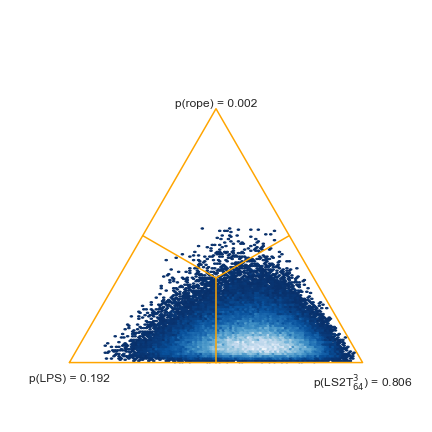}
		% \captionof{figure}{Caption for image}
		% \label{fig:sample_figure}
	\end{minipage}
	\begin{minipage}{0.245\textwidth}
		\centering
% 		\hspace{-30pt}
		\includegraphics[width=1.1in]{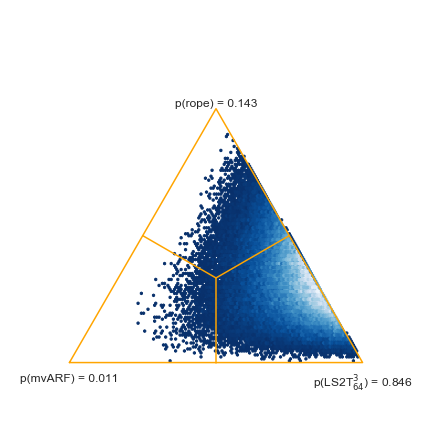}
		% \captionof{figure}{Caption for image}
		% \label{fig:sample_figure}
	\end{minipage}
	\begin{minipage}{0.245\textwidth}
		\centering
% 		\hspace{-30pt}
		\includegraphics[width=1.1in]{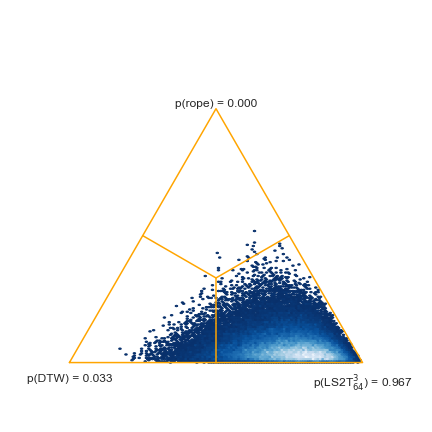}
		% \captionof{figure}{Caption for image}
		% \label{fig:sample_figure}
	\end{minipage}
	% \vspace{-10pt}
	\begin{minipage}{0.245\textwidth}
		\centering
% 		\hspace{-30pt}
		\includegraphics[width=1.1in]{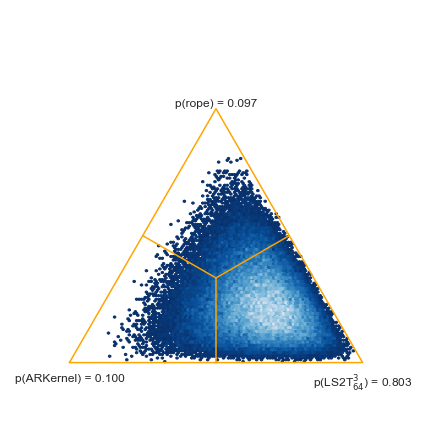}
		% \captionof{figure}{Caption for image}
		% \label{fig:sample_figure}
	\end{minipage}
	\begin{minipage}{0.245\textwidth}
		\centering
% 		\hspace{-30pt}
		\includegraphics[width=1.1in]{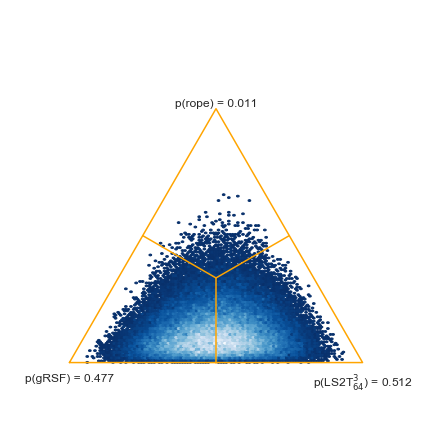}
		% \captionof{figure}{Caption for image}
		% \label{fig:sample_figure}
	\end{minipage}
	% \vspace{-10pt}
	\begin{minipage}{0.245\textwidth}
		\centering
		% 		\hspace{-30pt}
		\includegraphics[width=1.1in]{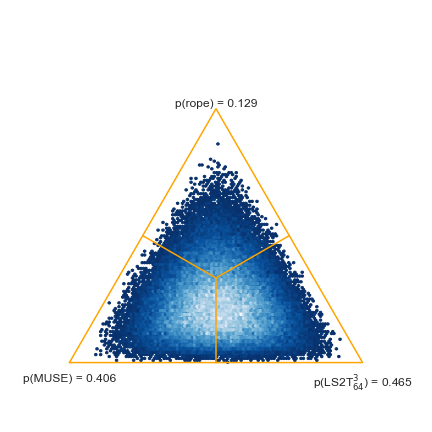}
		% \captionof{figure}{Caption for image}
		% \label{fig:sample_figure}
	\end{minipage}
	% \vspace{-10pt}
	\begin{minipage}{0.245\textwidth}
		\centering
% 		\hspace{-30pt}
		\includegraphics[width=1.1in]{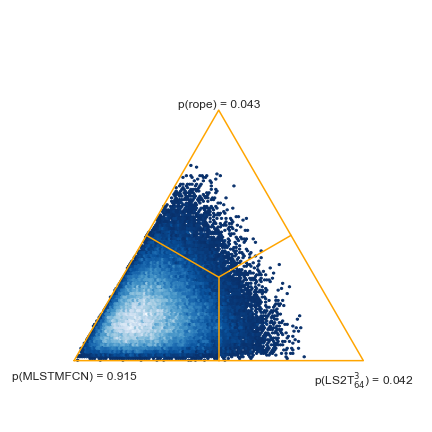}
		% \captionof{figure}{Caption for image}
		% \label{fig:sample_figure}
	\end{minipage}
	\begin{minipage}{0.245\textwidth}
		\centering
% 		\hspace{-30pt}
		\includegraphics[width=1.1in]{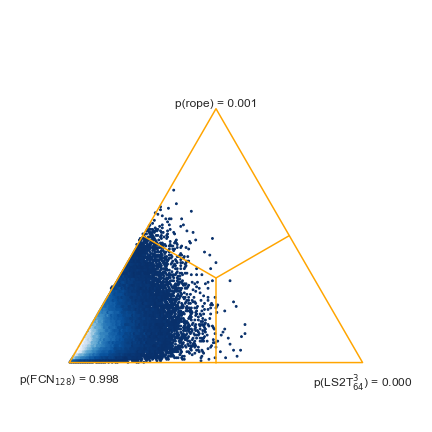}
		% \captionof{figure}{Caption for image}
		% \label{fig:sample_figure}
	\end{minipage}
	\begin{minipage}{0.245\textwidth}
		\centering
% 		\hspace{-30pt}
		\includegraphics[width=1.1in]{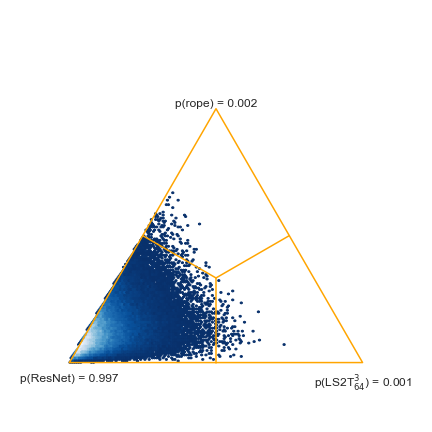}
		% \captionof{figure}{Caption for image}
		% \label{fig:sample_figure}
	\end{minipage}
	\begin{minipage}{0.245\textwidth}
		\centering
% 		\hspace{-30pt}
		\includegraphics[width=1.1in]{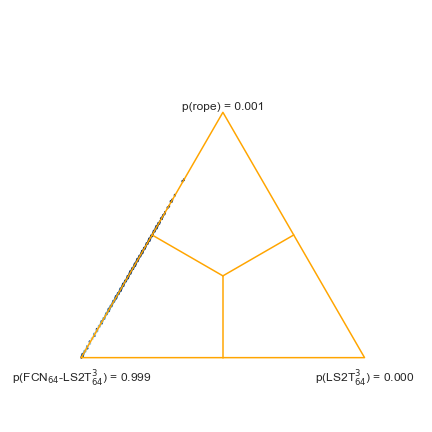}
		% \captionof{figure}{Caption for image}
		% \label{fig:sample_figure}
	\end{minipage}
	\begin{minipage}{0.245\textwidth}
		\centering
% 		\hspace{-30pt}
		\includegraphics[width=1.1in]{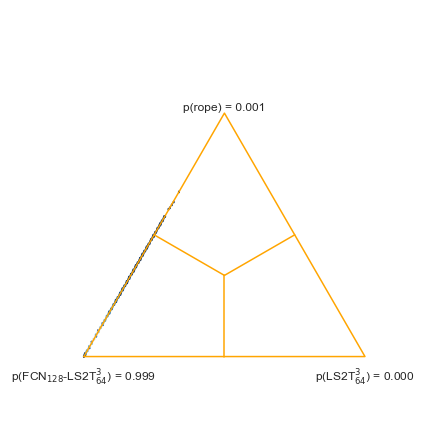}
		% \captionof{figure}{Caption for image}
		% \label{fig:sample_figure}
	\end{minipage}
	\begin{minipage}{0.245\textwidth}
		\centering
% 		\hspace{-30pt}
		\includegraphics[width=1.1in]{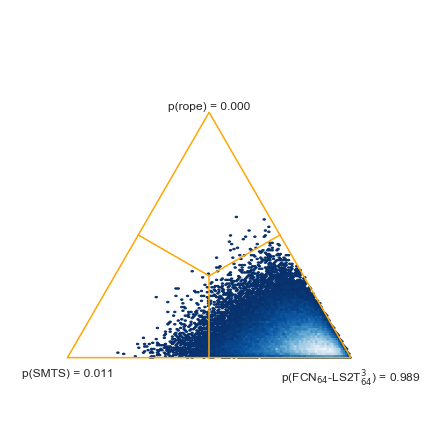}
		% \captionof{figure}{Caption for image}
		% \label{fig:sample_figure}
	\end{minipage}
	\begin{minipage}{0.245\textwidth}
		\centering
% 		\hspace{-30pt}
		\includegraphics[width=1.1in]{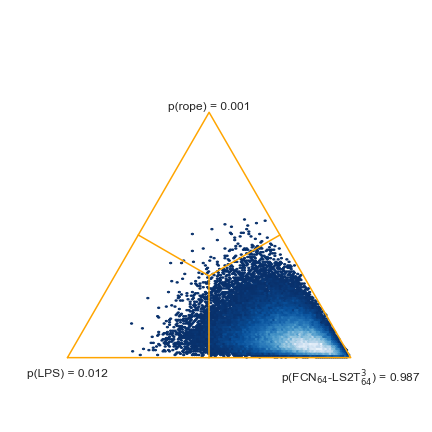}
		% \captionof{figure}{Caption for image}
		% \label{fig:sample_figure}
	\end{minipage}
	\begin{minipage}{0.245\textwidth}
		\centering
% 		\hspace{-30pt}
		\includegraphics[width=1.1in]{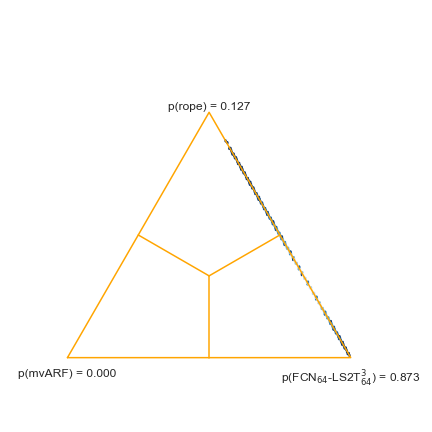}
		% \captionof{figure}{Caption for image}
		% \label{fig:sample_figure}
	\end{minipage}
	\begin{minipage}{0.245\textwidth}
		\centering
% 		\hspace{-30pt}
		\includegraphics[width=1.1in]{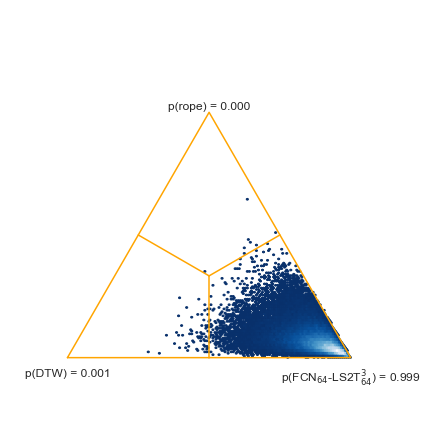}
		% \captionof{figure}{Caption for image}
		% \label{fig:sample_figure}
	\end{minipage}
	% \vspace{-10pt}
	\begin{minipage}{0.245\textwidth}
		\centering
% 		\hspace{-30pt}
		\includegraphics[width=1.1in]{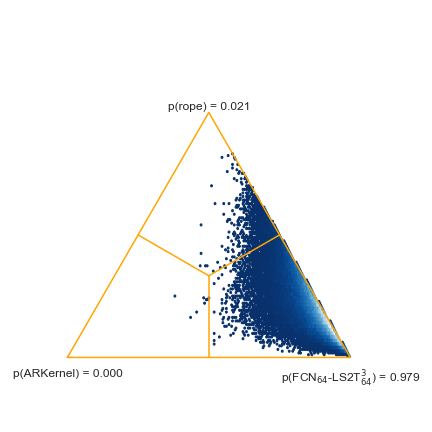}
		% \captionof{figure}{Caption for image}
		% \label{fig:sample_figure}
	\end{minipage}
	\begin{minipage}{0.245\textwidth}
		\centering
% 		\hspace{-30pt}
		\includegraphics[width=1.1in]{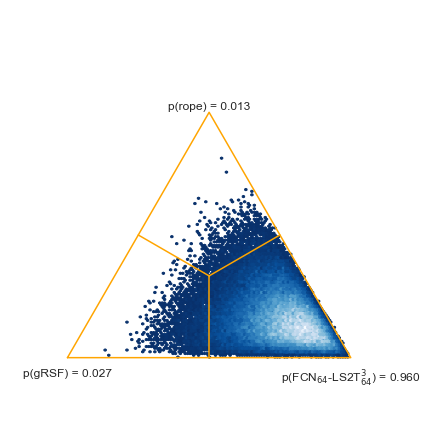}
		% \captionof{figure}{Caption for image}
		% \label{fig:sample_figure}
	\end{minipage}
	% \vspace{-10pt}
	\begin{minipage}{0.245\textwidth}
		\centering
		% 		\hspace{-30pt}
		\includegraphics[width=1.1in]{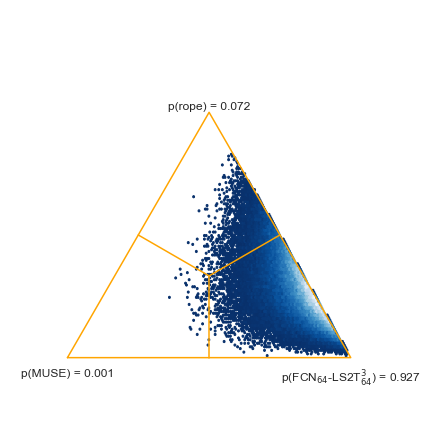}
		% \captionof{figure}{Caption for image}
		% \label{fig:sample_figure}
	\end{minipage}
	% \vspace{-10pt}
	\begin{minipage}{0.245\textwidth}
		\centering
% 		\hspace{-30pt}
		\includegraphics[width=1.1in]{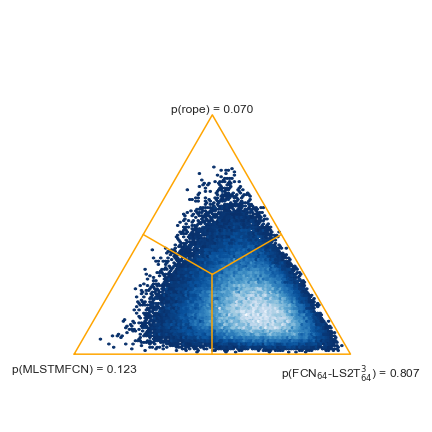}
		% \captionof{figure}{Caption for image}
		% \label{fig:sample_figure}
	\end{minipage}
	\begin{minipage}{0.245\textwidth}
		\centering
% 		\hspace{-30pt}
		\includegraphics[width=1.1in]{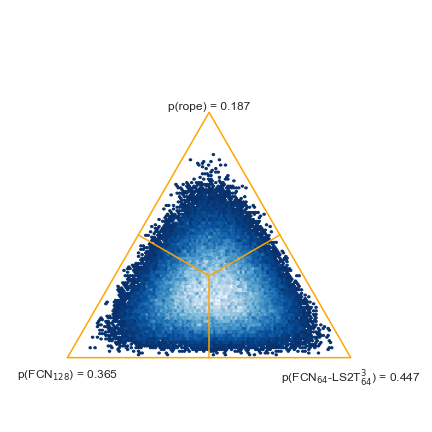}
		% \captionof{figure}{Caption for image}
		% \label{fig:sample_figure}
	\end{minipage}
	\begin{minipage}{0.245\textwidth}
		\centering
% 		\hspace{-30pt}
		\includegraphics[width=1.1in]{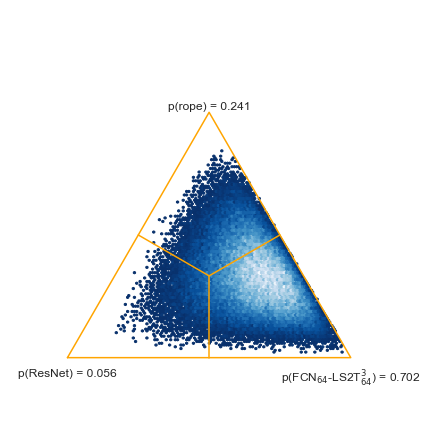}
		% \captionof{figure}{Caption for image}
		% \label{fig:sample_figure}
	\end{minipage}
	\begin{minipage}{0.245\textwidth}
		\centering
% 		\hspace{-30pt}
		\includegraphics[width=1.1in]{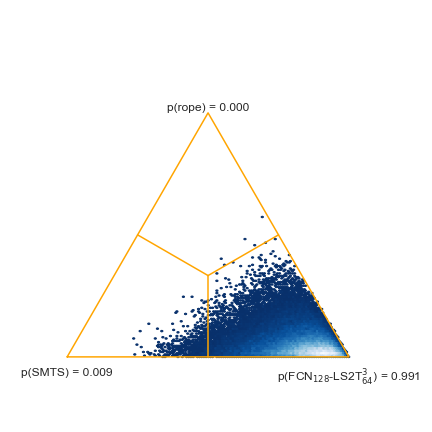}
		% \captionof{figure}{Caption for image}
		% \label{fig:sample_figure}
	\end{minipage}
	\begin{minipage}{0.245\textwidth}
		\centering
% 		\hspace{-30pt}
		\includegraphics[width=1.1in]{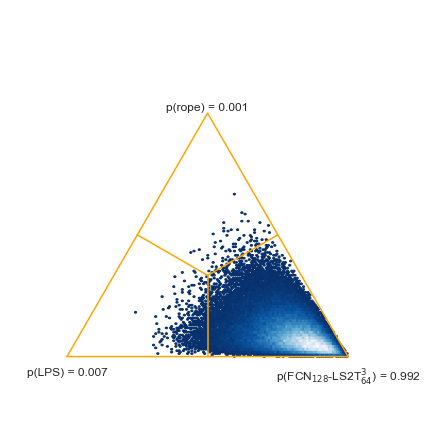}
		% \captionof{figure}{Caption for image}
		% \label{fig:sample_figure}
	\end{minipage}
	\begin{minipage}{0.245\textwidth}
		\centering
% 		\hspace{-30pt}
		\includegraphics[width=1.1in]{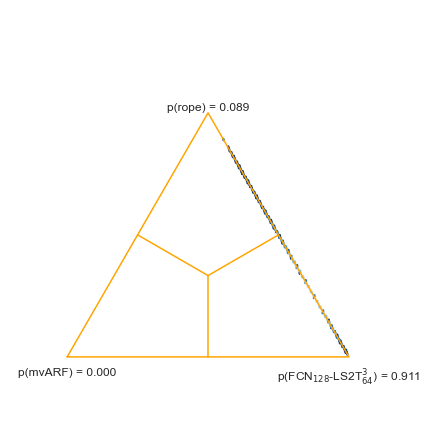}
		% \captionof{figure}{Caption for image}
		% \label{fig:sample_figure}
	\end{minipage}
	\begin{minipage}{0.245\textwidth}
		\centering
% 		\hspace{-30pt}
		\includegraphics[width=1.1in]{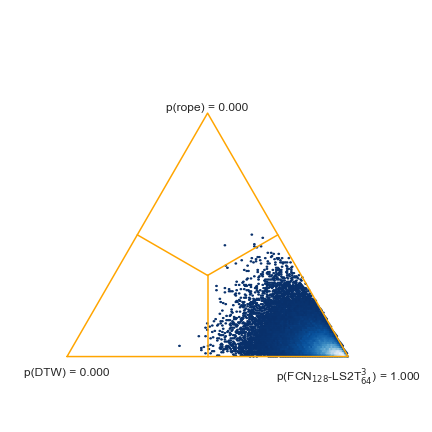}
		% \captionof{figure}{Caption for image}
		% \label{fig:sample_figure}
	\end{minipage}
	% \vspace{-10pt}
	\begin{minipage}{0.245\textwidth}
		\centering
% 		\hspace{-30pt}
		\includegraphics[width=1.1in]{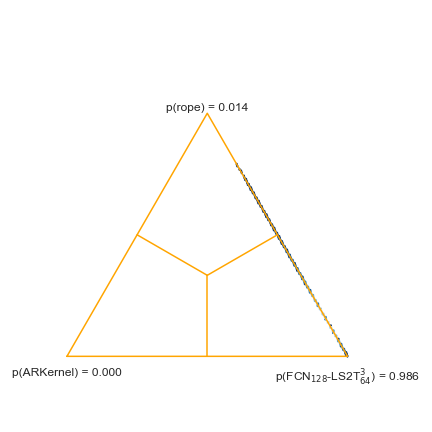}
		% \captionof{figure}{Caption for image}
		% \label{fig:sample_figure}
	\end{minipage}
	\begin{minipage}{0.245\textwidth}
		\centering
% 		\hspace{-30pt}
		\includegraphics[width=1.1in]{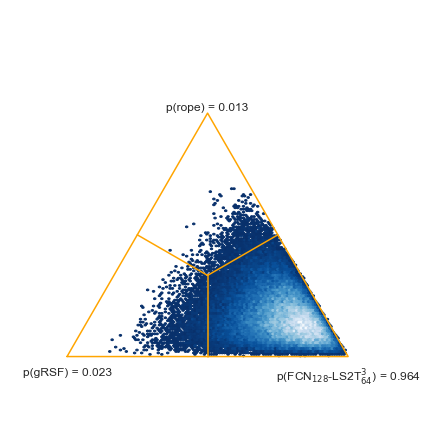}
		% \captionof{figure}{Caption for image}
		% \label{fig:sample_figure}
	\end{minipage}
	% \vspace{-10pt}
	\begin{minipage}{0.245\textwidth}
		\centering
		% 		\hspace{-30pt}
		\includegraphics[width=1.1in]{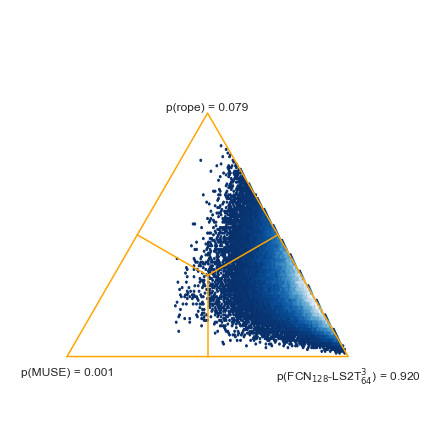}
		% \captionof{figure}{Caption for image}
		% \label{fig:sample_figure}
	\end{minipage}
	% \vspace{-10pt}
	\begin{minipage}{0.245\textwidth}
		\centering
% 		\hspace{-30pt}
		\includegraphics[width=1.1in]{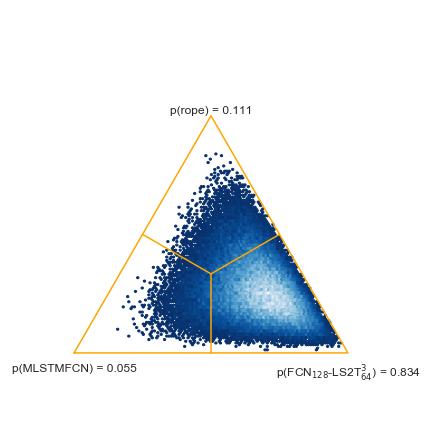}
		% \captionof{figure}{Caption for image}
		% \label{fig:sample_figure}
	\end{minipage}
	\begin{minipage}{0.245\textwidth}
		\centering
% 		\hspace{-30pt}
		\includegraphics[width=1.1in]{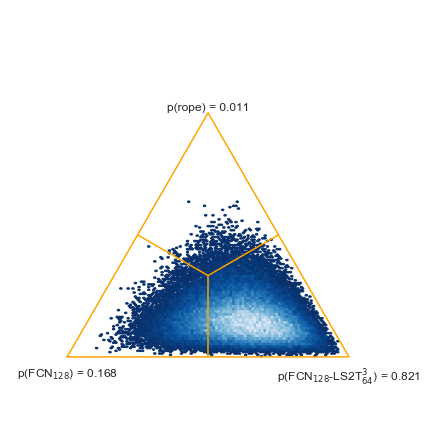}
		% \captionof{figure}{Caption for image}
		% \label{fig:sample_figure}
	\end{minipage}
	\begin{minipage}{0.245\textwidth}
		\centering
% 		\hspace{-30pt}
		\includegraphics[width=1.1in]{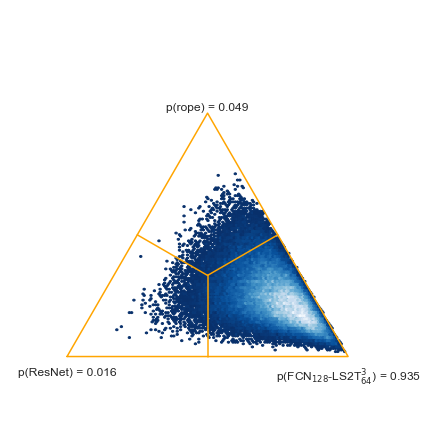}
		% \captionof{figure}{Caption for image}
		% \label{fig:sample_figure}
	\end{minipage}
	\caption{Posterior distribution plots of pairwise Bayesian signed-rank test comparisons}
	\label{fig:baycomp1}
% 	\vspace{-10pt}
\end{figure}

\clearpage

\subsection{Mortality prediction} \label{app:mortality}

\paragraph{Problem formulation.} Mortality prediction in healthcare is a form of binary classification using medical datasets.  This kind of data is very heterogenous with the input space being a combination of dynamic and static features, i.e.~ $\cX = \Seq{\R^d} \bigoplus \R^e$, and the class distributions often also being highly imbalanced. Also among the dynamic features there can often be missing values, in fact they are usually only observed very sparsely. Hence, it is not guaranteed that for a time series $\bx \in \Seq{\R^d}$, all coordinates of an observation $\bx_{i, t_j}$ are observed for a given time-point $t_j$. In other words, we are given for every $\bx_i \in \Seq{\R^d}$ an additional observation mask $\m_i = (\m_{i, t_j})_{j=1}^{L_i} \in \Seq{\{0, 1\}^d}$, that specifies whether a given coordinate of $\bx_i$ was observed at time $t_j$ or not.

\paragraph{Dataset.} We consider the {\sc Physionet2012} dataset for this task, that consists of medical time series of 12,000 ICU stays over at least 48 hours. Overall, 6 static features were recorded and potentially up to 37 TS values were measured both irregularly and asynchronously for each stay with certain dimensions completely missing in some cases. The task is to predict the in-hospital mortality of patients during their hospital stay. From an ML point of view, this is a difficult dataset due missing values, low signal-to-noise ratio, and imbalanced class distributions with a prevalence ratio of around $14 \%$. For comparability of results, we use the same train-val-test split as in \citet{horn2020set} available at \url{https://github.com/ExpectationMax/medical_ts_datasets}, where the 12,000 examples were split in a ratio of $64$-$16$-$20$. Additionally, examples not containing any TS information were excluded from the dataset, for a list of these see \citet[App.~A.1]{horn2020set}. 

\paragraph{Preprocessing.} For missing TS values, we use a three-step imputation method: \begin{enumerate*}[label=(\arabic*)] \item compute the mean value of each dynamic feature across the training dataset, \item if a TS value is missing at time $t=0$, impute it with the mean, \item for missing TS values at time $t > 0$, use forward imputation on a roll-forward basis \end{enumerate*}. We remark that it would have also been possible to use the GP-VAE imputation from the following experiment, but our aim was to keep the two experiments separate, in particular, to allow fair comparability with the results in \citet{horn2020set}. Furthermore, we make the information about missing values available to the model using the augmentation defined in \citet[eq.~9]{che2018recurrent}, which consists of adding as extra coordinates the observation mask and the number of time steps elapsed since an observation was made, both for each dynamic feature. The static features are handled by tiling along the time axis and adding them as extra coordintaes. Finally, all static and dynamic features are normalized to zero mean and unit variance using the training set statistics. 

\paragraph{Baselines.} As baselines, we use the experiments conducted in \citet{horn2020set}, that includes SOTA architectures for irregularly sampled data such as their {\sc SeFT-Attn}, {\sc GRU-D} \citep{che2018recurrent}, {\sc IP-Nets} \citep{shukla2019interpolation}, {\sc Phased-LSTM} \citep{neil2016phased}, {\sc Transformer} \citep{vaswani2017attention} and {\sc Latent-ODE} \citep{rubanova2016latent}. Together these methods make up a very strong baseline to compare against. However, the main question for us still is whether we can improve on the vanilla FCN model with the FCN-LS2T architecture (Figure \ref{fig:ls2t_architectures}), since our aim is simply to demonstrate that LS2T layers can serve as useful building blocks in deep learning models via their ability to capture non-local interactions in heterogeneous time series and sequences.

\paragraph{Hyperparameter selection.} We train two models on this task, FCN and FCN-LS2T. To keep the experiment fair, we align with the experimental setting in \citet{horn2020set} and follow their hyperparameter selection procedure using randomized search. For both models, we uniformly sample $20$ hyperparameter settings from the hyperparameter grid specified as follows: \begin{enumerate*}[label=(\arabic*)] \item for FCN-LS2T, the FCN width from $h \in \{64, 128, 256\}$. the LS2T width from $n \in \{64, 128, 256\}$, the LS2T order from $m \in \{2, 3, 4\}$, LS2T depth from $d \in \{1, 2, 3\}$ and whether to use the recursive or independent LS2T formulation (see Appendix \ref{app:recursive}); \item for FCN, the width from $h \in \{64, 128, 256\}$; \item for both models, we sample the dropout preceding the classification layer from $r_1 \in \{0.0, 0.1, 0.2, 0.3, 0.4\}$, the spatial dropout that follows all convolutional and LS2T layers from $r_2 \in \{0.0, 0.1, 0.2, 0.3, 0.4\}$. For training, we also sample for both models the batch size used from $b \in \{4, 8, 16, 32\}$ and the initial learning rate from $\alpha \in \{1 \times 10^{-3}, 5 \times 10^{-4}, 2.5 \times 10^{-4}, 1 \times 10^{-4}\}$. \end{enumerate*} For both architectures, we train a model for each of the 20 hyperparameter settings, and then select the setting which provides the best performance on the validation set. The best model is selected by computing a composite z-score on the validation set, which consists of computing a z-score across the realizations for each metric, that is, {\sc Accuracy}, {\sc AUPRC}, {\sc AUROC}, and then taking a sum of these z-scores.

\paragraph{Training details.} Last but not least, we specify the training methodology. Similarly to \citet{horn2020set}, rather than utilizing class weights during training to deal with unbalanced class distributions, we use a generator which feeds balanced batches to the model during each training iteration. This approach is more beneficial for small batch training on such heavily unbalanced datasets, such as the current one. Then, we define an epoch as the number of training iterations required to see all examples from the class with the lowest prevalence ratio. The maximum number of epochs is set to $200$, and we stop early after $50$ epochs of no improvement over the validation {\sc AUPRC}, after which the model is restored to the best parameter set according to this metric. We also employ a learning rate decay of $\beta = 1/2$ after 10 epochs of no improvement and only until the learning rate reaches $\alpha_{min} = 1 \times 10^{-4}$. Clearly, using the same validation set for early stopping and selecting the hyperparameters introduces a bias in the model selection, however, this is partially remedied by using for hyperparameter selection a composite of three metrics, rather than just the {\sc AUPRC}.

\paragraph{Evaluation.} After selecting the best hyperparameters, we independently train and evaluate on the test set each model $5$ times. The means and standard deviations of the resulting performance metrics over these $5$ model trains are what displayed in Table \ref{table:mort_pred}. The best found hyperparameter settings are the following: \begin{enumerate*}[label=(\arabic*)] \item for FCN-LS2T, FCN width $h = 64$, LS2T width $n=256$, LS2T order $m=3$, LS2T depth $d=3$, recursive formulation, dropout $r_1 = 0.3$, spatial dropout $r_2 = 0.4$, batch size $b=32$, inital learning rate $\alpha = 1 \times 10^{-4}$; \item for FCN, FCN width $h=256$, dropout $r_1 = 0.4$, spatial dropout $r_2 = 0.3$, batch size $b=4$, initial learning rate $\alpha = 1 \times 10^{-4}$. \end{enumerate*}

\subsection{Generative sequential data imputation} \label{app:imputation}

\paragraph{Problem formulation.} Imputation of sequential data can be formulated as a problem of generative unsupervised learning. The input space is given as $\cX = \Seq{\R^d}$ and we are given a number of examples $\bX = (\bx_i)_{i=1}^{n_\bX} \subset \Seq{\R^d}$ with $\bx_i = (\bx_{i, t_j})_{j=1}^{L_i}$. Similarly to before, there are missing values in the input sequences, that is, we are given for every $\bx_i \in \Seq{\R^d}$ an additional observation mask $\m_i = (\m_{i, t_j})_{j=1}^{L_i} \in \Seq{\{0, 1\}^d}$, that specifies whether a given coordinate of $\bx_i$ was observed at time $t_j$ or not. The task in this case is specifically to model the distribution of the unobserved coordinates given the observed coordinates potentially at different time-points.

\paragraph{Model details.} We expand on the GP-VAE \citep{fortuin2019gpvae} model in details. Let $\bx = (\bx_i)_{i=1,\dots,L} \in \Seq{\R^d}$ be a sequence of length $L \in \NN$. The model assumes that $\bx$ is noisily generated time-point-wise conditioned on discrete-time realizations of a latent process denoted by $\bz = (\bz_{i})_{i=1,\dots,L} \in \Seq{\R^{d^\prime}}$,
\begin{align}
p_\theta(\bx_i \vert \bz_i) = \cN(\bx_i \given g_\theta(\bz_i), \sigma^2 \mathbf{I}_d),
\end{align}
where $g_\theta: \R^{d^\prime} \rightarrow \R^d$ is the time-point-wise decoder network parametrized by $\theta$, while $\sigma^2 \in \R$ is the observation noise variance. The temporal interdependencies are modelled in the latent space by assigning independent Gaussian process (GP) priors \citep{williams2006gaussian} to the coordinate processes of $\bz$, i.e. denoting $\bz_i = (z_i^j)_{j=1,\dots,d^\prime} \in \R^{d^\prime}$, it is assumed that $z^j \sim \mathcal{GP}(m(\cdot), k(\cdot, \cdot))$, where $m: \R \rightarrow \R$ and $k: \R \times \R \rightarrow \R$ are the mean and covariance functions. The authors propose to use the Cauchy kernel as covariance function, defined as 
\begin{align}
    k(\tau, \tau^\prime) = \tilde\sigma^2 \left(1 + \frac{(\tau - \tau^\prime)^2}{l^2}\right)^{-1},
\end{align}
which can be seen as an infinite mixture of RBF kernels, allowing one to model temporal dynamics on multiple length scales. For the variational approximation \citep{blei2017variational, zhang2018advances}, an amortized Gaussian \citep{Gershman2014AmortizedII} is used that factorizes across the latent space dimensions, but not across the observation times:
\begin{align}
    q_\psi(\bz_1, \dots, \bz_L \given \bx_1, \dots \bx_L) &= q_\psi(z_1^1, \dots z_L^1 \given \bx_1, \dots, \bx_L) \cdots, q_\psi(z_1^{d^\prime}, \dots z_L^{d^\prime} \given \bx_1, \dots \bx_L) \\
    &= \cN(z_1^1, \dots, z_L^1 \given \mathbf{m}_1, \mathbf{A}_1) \cdots \cN(z_1^{d^\prime}, \dots, z_L^{d^\prime} \given \mathbf{m}_{d^\prime}, \mathbf{A}_{d^\prime}), 
\end{align}
where $\mathbf{m}_j \in \R^L$ are the posterior means and $\mathbf{A}_j \in \R^{L \times L}$ are the posterior covariance matrices for $j=1, \dots, d^\prime$. In general, estimating the full covariance matrices $\mathbf{A}_j \in \R^{L \times L}$ from a single data example $\bx = (\bx_1, \dots, \bx_L) \in \Seq{\R^d}$ is an ill-posed problem. To circumvent the curse of dimensionality in the matrix estimation while allowing for long-range correlations in time, a structured precision matrix representation is used, such that $\mathbf{A}_j^{-1} = \mathbf{B}_j \mathbf{B}_j^\top$ with $\mathbf{B}_j \in \R^{L \times L}$ a lower bidiagonal matrix, such as in \citet{dorta2018structured, blei2006dynamic, bamler2017dynamic}, which results in a tridiagonal precision matrix and a potentially dense covariance matrix.

Training across the whole dataset $\bX = (\bx_i)_{i=1}^{n_\bX} \subset \Seq{\R^d}$ is coupled through the decoder and encoder parameters $\theta$ and $\psi$, and the ELBO is computed as
\begin{align}
    \frac{1}{n_\bX} \sum_{i=1}^{n_\bX} \log p(\bx_i) \geq \frac{1}{n_\bX}\sum_{i=1}^{n_\bX} \Big( \sum_{j=1}^{L_i} &\EE_{q_\psi(\bz_{i, j} \given \bx_i)}[\log p_\theta(\bx_{i, j} \given \bz_{i, j})] \\
    &- \beta \KL{q_\psi(\bz_i \given \bx_i)}{p(\bz_{i})}) \Big), 
\end{align}
where the log-likelihood term is only computed across observed features as was done in \citet{nazabal2018handling}. Similarly to $\beta$-VAEs \citep{higgins2017beta}, $\beta$ is used to rebalance the KL term, now to account for the missingness rate. 

\paragraph{Baselines.} Additionally to the baseline GP-VAE, the reported baseline results are the same ones as in \citet{fortuin2019gpvae}, which are mean imputation, forward imputation, VAE \citep{kingma2013auto}, HI-VAE \citep{nazabal2018handling} and BRITS \citep{Cao2018brits}. Among these, the VAE based models are Bayesian and provide a probability measure on possible imputations, while the mean/forward imputation methods and the RNN based BRITS only provide a single imputation.

\paragraph{Datasets.}

\begin{wraptable}{r}{9.25cm}
    \vspace{-15pt}
	\caption{Specification of datasets used for imputation}
	\label{table:dataset_spec2}
	\vspace{-10pt}
    % \vskip 0.15in
    \begin{center}
    \begin{small}
	\begin{sc}
    \begin{tabular}{lcccccc}
    \toprule
    Dataset  & $n_c$ & $m$ & $d$ & $L_\bx$ & $n_\bX$ & $n_{\bX_\star}$ \\
    \midrule
        HMNIST & $10$ & $0.45$ & $28 \times 28$ & $10$ & $60000$ & $10000$ \\
        Sprites & - & $0.6$ & $64 \times 64 \times 3$ & $8$ & $9000$ & $2664$ \\
        Physionet & $2$ & $0.82$ & $35$ & $48$ & $3997$ & - \\
    \bottomrule
    \end{tabular}
	\end{sc}
    \end{small}
    \end{center}
    % \vskip -1.0in
    \vspace{-15pt}
\end{wraptable}

Table \ref{table:dataset_spec2} details the datasets used, which are the same ones as considered in \citet{fortuin2019gpvae}. The columns are defined as: $n_c \in \NN$ denotes the number of classes if the dataset is labelled, $m \in (0, 1)$ denotes ratio of missing data, $d \in \NN$ denotes the state space dimension of sequences, $L_\bx \in \NN$ denotes the sequence length, $n_\bX, n_{\bX_\star} \in \NN$ denote the number of examples in the respective training and testing sets. For Sprites no labels are available, while for Physionet all examples are in the training set and no ground truth values are available. For HMNIST, the MNAR version was used, the most difficult missingness mechanism \citep{fortuin2019gpvae}.

\begin{wrapfigure}{r}{0.45\textwidth}
	\centering
	\vspace{-5pt}
	%     \begin{subfigure}{0.74\textwidth}
	% 		\centering
	% % 		\vspace{-10pt}
	% 		\includegraphics[width=4.1in]{imgs/sprites_reconstruction.pdf}
	% 		% \captionof{figure}{Caption for image}
	% 		% \label{fig:sample_figure}
	% % 		\captionof{figure}{Reconstructions from the Sprites dataset with the missingnes  (top), reconstructed (middle) and original (bottom).}
	%     % \caption{Sprites}
	%     \end{subfigure}
		% \begin{subfigure}{0.25\textwidth}
		\vspace{-10pt}
		\hspace{-10pt}
		\centering
		% \vspace{-5pt}
		% \begin{center}
		% \hspace{-20pt}
		% \begin{turn}{270}
		\beginpgfgraphicnamed{gpvae-ls2t-encoder-plot}
		\begin{tikzpicture}[scale=0.6, every node/.style={scale=0.6}, shorten >=1pt,draw=black!50, node distance=\layersep]
			
		\node[seq, draw=purple, fill=lightpurple] at (0.0, 0.0) (inp) {\sc Input};
		% \draw[<->, color=black, thin] (-0.5, -1.75) -- (-0.5, 1.75);
		\node at (-0.5, 0.0) (l) {$\scriptstyle \ell$};
		% \draw[<->, color=black, thin] (-0.35, -1.9) -- (0.35, -1.9);
		\node at (0.0, 2.0) (l) {$\scriptstyle d$};
		
		\node[seq, draw=blue, fill=lightblue] at (1.0, 0.0) (preprocess) {\sc Preprocessor};
		\draw[->, color=black] (inp) -- (preprocess);
		\node at (1.0, 2.0) (l) {$\scriptstyle d$};
		
		\node[seq, draw=blue, fill=lightblue] at (2.0, 0.0) (conv) {\sc Convolution};
		\draw[->, color=black] (preprocess) -- (conv);
		\node at (2.0, 2.0) (l) {$\scriptstyle h$};
		
		\node[seq, draw=green, fill=lightgreen] at (3.0, 0.0) (time) {\sc Time + Diff};
		\draw[->, color=black] (conv) -- (time);
		\node at (3.0, 2.0) (l) {$\scriptstyle h+1$};
		
		\node[seq, draw=green, fill=lightgreen] at (3.7, 0.0) (ls2t_behind3) {};
		\node[seq, draw=green, fill=lightgreen] at (3.8, 0.0) (ls2t_behind3) {};
		\node[seq, draw=green, fill=lightgreen] at (3.9, 0.0) (ls2t_behind3) {};
		\node[seq, draw=green, fill=lightgreen] at (4.0, 0.0) (bls2t) {\sc B-LS2T};
		\draw[->, color=black] (time) -- ($(bls2t.north) + (0.0, 0.)$);
		\node at (4.0, 2.0) (l) {$\scriptstyle h \times 4$};
		
		\node[wideseq, draw=green, fill=lightgreen] at (5.125, 0.0) (wide) {\sc LN + Reshape};
		\draw[->, color=black] (bls2t) -- (wide);
		\node at (5.125, 2.0) (l) {$\scriptstyle 4h$};
		
		\node[seq, draw=blue, fill=lightblue] at (6.25, 0.0) (dense1) {\sc Dense};
		\draw[->, color=black] (wide) -- (dense1);
		\node at (6.125, 2.0) (l) {$\scriptstyle h^\prime$};
		
		\node[outer sep=0.5pt, inner sep=0pt] at (7.25, 0.0) (dots) {$\dots$};
		\draw[->, color=black] (dense1) -- (dots);
		
		\node[seq, draw=blue, fill=lightblue] at (8.25, 0.0) (dense2) {\sc Dense};
		% \draw[->, color=black] (dots) -- (dense2);
		\node at (8.25, 2.0) (l) {$\scriptstyle h^\prime$};
		\draw[->, color=black] (dots) -- (dense2);
		
		\node[seq, draw=yellow, fill=lightyellow] at (9.25, 0.0) (latent) {\sc Latent};
		\draw[->, color=black] (dense2) -- (latent);
		\node at (9.25, 2.0) (l) {$\scriptstyle 3d^\prime$};
		
		\end{tikzpicture}
		\endpgfgraphicnamed
		% \end{turn}
		% \hspace{20pt}
		% \end{center}
		% \vspace{-5pt}
		% \caption{Encoder}
		% \end{subfigure}
	\vspace{-5pt}
	\caption{Encoder in GP-VAE (B-LS2T).}
	\label{fig:gpvae_ls2t_encoder}
	\vspace{-10pt}
	\end{wrapfigure}
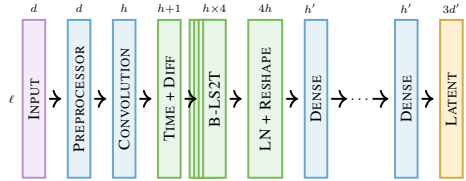

\paragraph{Experiment details.} As depicted in Figure \ref{fig:gpvae_ls2t_encoder}, the difference between the original GP-VAE model and ours is that is that we additionally employ a single bidirectional Seq2Tens block (B-LS2T) in the encoder network following the convolutional layer, but preceding the time-distributed dense layers. The motivation for this is that the original encoder only takes local sequential structure into account using the convolutional layer. Hence, it does not exploit global sequential information, which might limit the expressiveness of the encoder network. This limitation can lead to suboptimal inference, due to the fact that the encoder is not able to represent a rich enough subset of the variational family of distributions. This is called the amortization gap in the literature \citep{Cremer2018inference}.

We have thus hypothesized that by incorporating a bidirectional LS2T layer into the model that takes sequential structure into account not only locally, but globally, we can improve the expressiveness of the encoder network, that can in turn improve on the variational approximation. However, it should be noted that according to the findings of \citet{Cremer2018inference}, a larger encoder network can potentially result in the variational parameters being overfitted to the training data, and can degrade the generalization on unseen data examples. Therefore, the main question is whether increasing the expressiveness of the encoder will improve the quality of the variational approximation on both seen and unseen examples, or will it lead to overfitting to the seen examples?  

Another interesting question that we have not considered experimentally, but could lead to improvements is the following. The time-point-wise decoder function assumes that $d^\prime \in \NN$ is large enough, so that $\bz \in \Seq{\R^{d^\prime}}$ is able to fully represent $\bx \in \Seq{\R^d}$ in a time-point-wise manner including its dynamics. Although in theory the GP prior should be able to learn the temporal dynamics in the latent space, this might again only be possible for a large enough latent state size $d^\prime$. In practice, it could turn out to be more efficient to use some of the contextual information in the decoder network as well, either locally, using e.g. a CNN, or globally, using e.g. LS2T layers or RNNs/LSTMs.

\paragraph{Implementation.}
For the implementation of the GP-VAE, we used the same one as in \citet{fortuin2019gpvae}, which implements it using Keras and Tensorflow. The bidirectional LS2T layer used our own implementation based on the same frameworks. The hyperparameters of the models, which are depicted in Appendix A in \citet{fortuin2019gpvae}, were left unchanged. The only change we concocted is the B-LS2T layer in the encoder network as depicted in Figure \ref{fig:gpvae_ls2t_encoder}. The width of the B-LS2T layer was set to be the same as the convolutional layer, and $M=4$ tensor levels were used. The parametrization of the low-rank LS2T layer used the independent formulation as detailed in Appendix \ref{app:recursive}.

We also made a simple change to how the data is fed into the encoder. In the original model, the missing values were imputed with $0$, while we instead used the forward imputed values. This was necessary due to the difference operation preceding the B-LS2T layer in Figure \ref{fig:gpvae_ls2t_encoder}. With the zero imputation, the coordinates with missing values exhibited higher oscillations after differencing, while with forward imputation the missing values were more well-behaved. A simple way to see this is that, if there were no preprocessing and convolutional layers in Figure \ref{fig:gpvae_ls2t_encoder} preceding the difference block, then this step would be equivalent to imputing the missing values with zero after differencing.

% \newpage
\paragraph{Result details.} Table \ref{table:gp_vae_comparison} shows the achieved performance on the datasets with our upgraded model, GP-VAE (B-LS2T), compared against the original GP-VAE \citep{fortuin2019gpvae} and the baselines. The reported results are negative log-likelihood (NLL), mean squared error (MSE) and AUROC on HMNIST, while on Sprites the MSE is reported and on Physionet the AUROC score. As Sprites is unlabeled, downstream classification performance (AUROC) is undefined on this dataset, while Physionet does not have ground truth values for the missing entries, and reconstruction error (MSE,NLL) is not defined. The only missing entry is Sprites NLL, which was omitted to preserve space. We observe that increasing the expressiveness of the encoder did manage to improve the results on HMNIST and Physionet. The only case where no improvement was observable is Sprites, where the GP-VAE already achieved a very low MSE score of $MSE = 2 \times 10^{-3}$.

To gain some intuition whether the lack of improvement on Sprites was due to the GP-VAE's performance already being maxed out, or there was some other pathology in the model, we further investigated the Sprites dataset and found a bottleneck in both the original and enhanced GP-VAE models. Due to the high dimensionality of the state space of input sequences, $d = 12288$, the width of the first convolutional layer in the encoder network was set to $h = 32$ in order to keep the number of parameters in the layer manageable and be able to train the model with a batch size of $n = 64$, while all subsequent layers had a width of $h^\prime = 256$. Thus, to see if this was indeed an information bottleneck, we increased the convolution width to $h = 256$ and decreased the batch size to $n = 16$, with all other hyperparameters unchanged. Then, we trained using this modification both the baseline GP-VAE and our GP-VAE (B-LS2T) five times on the Sprites datasets. The achieved MSE scores were \begin{enumerate*}[label=(\roman*)] \item GP-VAE (base): $MSE = 1.4 \times 10^{-3} \pm  4.1 \times 10^{-5}$, \item GP-VAE (B-LS2T): $MSE = 1.3 \times 10^{-3} \pm 4.9 \times 10^{-5}$\end{enumerate*}. Therefore, the smaller convolutional layer was indeed causing an information bottleneck, and by increasing its width to be on par with the other layers in the encoder, we managed to improve the performance of both models. The improvement on the GP-VAE (B-LS2T) was larger, which can be explained by the observation that lifting the information bottleneck additionally allowed the benefits of the B-LS2T layer to kick in, as detailed previously.

To sum up, we have empirically validated the hypothesis that capturing global information in the encoder was indeed beneficial, and managed to improve on the results even on unseen examples in all cases. The experiment as a whole supports that our introduced LS2T layers can serve as useful building blocks in a wide range of models, not only discriminative, but also generative ones.

\begin{figure}[t]
	\centering
	% 		\vspace{-10pt}
	\includegraphics[width=5in]{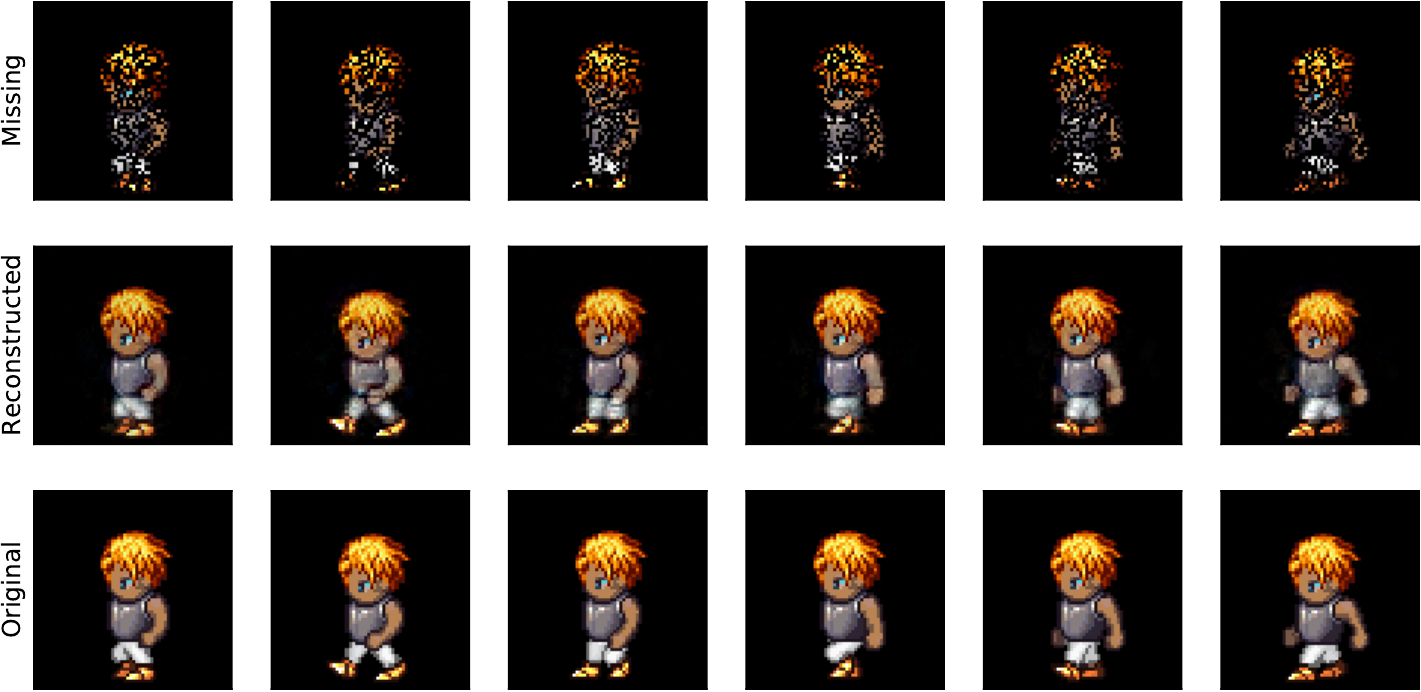}
	% \captionof{figure}{Caption for image}
	% \label{fig:sample_figure}
	\caption{Reconstructions from the Sprites dataset with the images with missingness (top), reconstructed (middle) and original (bottom).}
	% \caption{Sprites}
\end{figure}

\newpage

\subsubsection*{Acknowledgements}
CT is supported by the ``Mathematical Institute Award'' from the Mathematical Institute at the University of Oxford. PB is supported by the Engineering and Physical Sciences Research Council [EP/R513295/1].
HO is supported by the EPSRC grant ``Datasig'' [EP/S026347/1], the Alan Turing Institute, and the Oxford-Man Institute.

\end{document}